%% file: paper.tex
\documentclass[11pt]{article}

\usepackage{mathrsfs}  
\usepackage{verbatim}
\usepackage{etoolbox}
\usepackage{breakcites}

\newtoggle{icml}
\togglefalse{icml}

\input{arxiv_style}

\input{dylan}
\input{macros}

\usepackage{xspace}
\usepackage{setspace}

\title{Logarithmic Regret for Adversarial Online Control}

\author{Dylan J. Foster\\MIT\\{\small\texttt{dylanf@mit.edu}}\and
Max Simchowitz\\UC Berkeley\\{\small\texttt{msimchow@berkeley.edu}}}
\date{}

\begin{document}

\maketitle

\begin{abstract}
\input{abstract}
\end{abstract}%

\section{Introduction}
\label{sec:intro}
\input{section_introduction}

\section{Riccatitron: Logarithmic regret for online linear control}
\label{sec:algorithms}
\input{section_algorithms}
\section{Proving Theorem \ref*{thm:main_reg_decomp}: Advantages
  without states}
\label{sec:thm_main_reg_decomp}
\label{sec:analysis}
\input{section_analysis}

\section{General policy classes}
\label{sec:beyond_linear}
\input{section_general_class}

\section{Conclusion}
\label{sec:conclusion}
\input{conclusion}


\bibliography{refs}

\newpage

\appendix

\renewcommand{\contentsname}{Contents of Appendix}
\tableofcontents
\addtocontents{toc}{\protect\setcounter{tocdepth}{3}}
\clearpage

\section*{Organization and notation}
\addcontentsline{toc}{section}{Organization and notation}
\input{app_notation}

\section{Extensions \label{app:generalization}}
\input{section_generalizations}
\section{Limitations of online learning with stationary costs \label{app:olws_limits}}
\input{app_stationary}

\section{Basic technical results}
\label{app:technical}
\input{appendix_technical}

\section{Proofs from Section \ref*{sec:algorithms}}
\label{app:algorithm}
\input{appendix_algorithms}

\section{Proofs from Section \ref*{sec:analysis}}
\label{app:analysis}
\input{appendix_analysis}

\end{document}

%% file: arxiv_style.tex
\usepackage{mathrsfs}  
\usepackage{verbatim}
\usepackage{setspace}
\usepackage{showlabels}
\usepackage{longtable}
\usepackage{makecell} 
\iftoggle{icml}
{
	
}
{

	\usepackage[letterpaper, left=1in, right=1in, top=1in, bottom=1in]{geometry}
	\usepackage[scaled=.88]{helvet}
	\usepackage{xpatch}
	\xpatchcmd{\proof}{\itshape}{\normalfont\proofnameformat}{}{}
	\newcommand{\proofnameformat}{\bfseries}

}

\usepackage[dvipsnames]{xcolor}
\usepackage[colorlinks=true, linkcolor=blue!70!black, citecolor=blue!70!black,urlcolor=black,breaklinks=true]{hyperref}
\usepackage{microtype}

\usepackage{algorithm}
\usepackage[noend]{algpseudocode}

\usepackage{natbib}
\bibpunct{(}{)}{;}{a}{,}{,}

\usepackage{amsthm}
\usepackage{thmtools}

\usepackage{mathtools}
\usepackage{amsmath}
\usepackage{bbm}
\usepackage{amsfonts}
\usepackage{amssymb}

\usepackage{MnSymbol} 

\newenvironment{dproof}[1][Proof]%
{%
  \par\noindent{\bfseries\upshape {#1.}\ }%
 }%
 {\qed\ignorespacesafterend~\\}

\theoremstyle{definition}  

\newtheorem{lemma}{Lemma}[section]

\newtheorem{corollary}{Corollary}

\theoremstyle{plain}

\newtheorem{theorem}{Theorem}
\newtheorem{definition}{Definition}

\usepackage{prettyref}
\newcommand{\pref}[1]{\prettyref{#1}}
\newcommand{\pfref}[1]{Proof of \prettyref{#1}}
\newcommand{\savehyperref}[2]{\texorpdfstring{\hyperref[#1]{#2}}{#2}}
\newrefformat{eq}{\savehyperref{#1}{\textup{(\ref*{#1})}}}
\newrefformat{eqn}{\savehyperref{#1}{Eq.~\ref*{#1}}}
\newrefformat{lem}{\savehyperref{#1}{Lemma~\ref*{#1}}}
\newrefformat{def}{\savehyperref{#1}{Definition~\ref*{#1}}}
\newrefformat{line}{\savehyperref{#1}{Line~\ref*{#1}}}
\newrefformat{thm}{\savehyperref{#1}{Theorem~\ref*{#1}}}
\newrefformat{infthm}{\savehyperref{#1}{Informal Theorem~\ref*{#1}}}
\newrefformat{corr}{\savehyperref{#1}{Corollary~\ref*{#1}}}
\newrefformat{sec}{\savehyperref{#1}{Section~\ref*{#1}}}
\newrefformat{ssec}{\savehyperref{#1}{Section~\ref*{#1}}}
\newrefformat{sssec}{\savehyperref{#1}{Section~\ref*{#1}}}
\newrefformat{app}{\savehyperref{#1}{Appendix~\ref*{#1}}}
\newrefformat{ass}{\savehyperref{#1}{Assumption~\ref*{#1}}}
\newrefformat{ex}{\savehyperref{#1}{Example~\ref*{#1}}}
\newrefformat{fig}{\savehyperref{#1}{Figure~\ref*{#1}}}
\newrefformat{alg}{\savehyperref{#1}{Algorithm~\ref*{#1}}}
\newrefformat{rem}{\savehyperref{#1}{Remark~\ref*{#1}}}
\newrefformat{conj}{\savehyperref{#1}{Conjecture~\ref*{#1}}}
\newrefformat{prop}{\savehyperref{#1}{Proposition~\ref*{#1}}}
\newrefformat{proto}{\savehyperref{#1}{Protocol~\ref*{#1}}}
\newrefformat{prob}{\savehyperref{#1}{Problem~\ref*{#1}}}
\newrefformat{claim}{\savehyperref{#1}{Claim~\ref*{#1}}}
\newrefformat{table}{\savehyperref{#1}{Table~\ref*{#1}}}


%% file: dylan.tex
\DeclarePairedDelimiter{\abs}{\lvert}{\rvert} %
\DeclarePairedDelimiter{\brk}{[}{]}
\DeclarePairedDelimiter{\crl}{\{}{\}}
\DeclarePairedDelimiter{\prn}{(}{)}
\DeclarePairedDelimiter{\nrm}{\|}{\|}
\DeclarePairedDelimiter{\tri}{\langle}{\rangle}

\DeclarePairedDelimiter{\floor}{\lfloor}{\rfloor}

\DeclareMathOperator{\En}{\mathbb{E}}

\DeclareMathOperator*{\argmin}{arg\,min} 

\newcommand{\wt}[1]{\widetilde{#1}}
\newcommand{\wh}[1]{\widehat{#1}}

\def\ddefloop#1{\ifx\ddefloop#1\else\ddef{#1}\expandafter\ddefloop\fi}
\def\ddef#1{\expandafter\def\csname bb#1\endcsname{\ensuremath{\mathbb{#1}}}}
\ddefloop ABCDEFGHIJKLMNOPQRSTUVWXYZ\ddefloop
\def\ddefloop#1{\ifx\ddefloop#1\else\ddef{#1}\expandafter\ddefloop\fi}
\def\ddef#1{\expandafter\def\csname b#1\endcsname{\ensuremath{\mathbf{#1}}}}
\ddefloop ABCDEFGHIJKLMNOPQRSTUVWXYZ\ddefloop
\def\ddef#1{\expandafter\def\csname c#1\endcsname{\ensuremath{\mathcal{#1}}}}
\ddefloop ABCDEFGHIJKLMNOPQRSTUVWXYZ\ddefloop
\def\ddef#1{\expandafter\def\csname h#1\endcsname{\ensuremath{\widehat{#1}}}}
\ddefloop ABCDEFGHIJKLMNOPQRSTUVWXYZ\ddefloop
\def\ddef#1{\expandafter\def\csname hc#1\endcsname{\ensuremath{\widehat{\mathcal{#1}}}}}
\ddefloop ABCDEFGHIJKLMNOPQRSTUVWXYZ\ddefloop
\def\ddef#1{\expandafter\def\csname tc#1\endcsname{\ensuremath{\widetilde{\mathcal{#1}}}}}
\ddefloop ABCDEFGHIJKLMNOPQRSTUVWXYZ\ddefloop

\newcommand{\ls}{\ell}

\newcommand{\veps}{\varepsilon}

\newcommand{\ldef}{\vcentcolon=}
\newcommand{\rdef}{=\vcentcolon}

%% file: macros.tex
\newcommand{\vaw}{\textsf{VAW}\xspace}
\newcommand{\vvaw}{\sfterm{VAW}\xspace}
\newcommand{\sfterm}[1]{{\normalfont \textsf{#1}}}

\newcommand{\alg}{\mathsf{BL}}

\newcommand{\ind}[1]{^{(#1)}}

\newcommand{\Knot}{K_0}
\newcommand{\Mst}{\cM_{0}}
\newcommand{\Rst}{R_{\star}}
\newcommand{\Dsig}{D_{\Sigma}}

\newcommand{\Hclinf}{H_{\mathrm{cl},\infty}}
\newcommand{\Dqinf}{D_{\qstar_{\infty}}}
\newcommand{\nust}{\nu_{\star}}

\newcommand{\dapdef}{\pim_t(x;\matw) = -\Kinf x - q^M(\wpast),\quad\text{ where } q^M(\wpast) = \sum_{i=1}^m M^{[i]}w_{t-i} }

\newcommand{\algcomment}[1]{\textcolor{Blue}{\footnotesize{\texttt{\textbf{//
          #1}}}}}

\newcommand{\mainalg}{\sfterm{Riccatitron}\xspace}

\newcommand{\bigoh}{\cO}
\newcommand{\bigoht}{\wt{\cO}}

\newcommand{\psdleq}{\preceq}
\newcommand{\psdgeq}{\succeq}

\newcommand{\psdgt}{\succ}

\newcommand{\op}{\mathrm{op}}

\newcommand{\eigmin}{\lambda_{\mathrm{min}}}
\newcommand{\eigmax}{\lambda_{\mathrm{max}}}

\newcommand{\matx}{\mathbf{x}}
\newcommand{\matu}{\mathbf{u}}
\newcommand{\matw}{\mb{w}}

\newcommand{\dimx}{d_{\matx}}
\newcommand{\dimu}{d_{\matu}}

\newcommand{\Pinf}{P_{\infty}}
\newcommand{\Kinf}{K_{\infty}}

\newcommand{\dare}{\sfterm{DARE}\xspace}

\renewcommand{\wr}[1][t]{w_{#1:T}}

\newcommand{\Vf}{\mathbf{V}}
\newcommand{\Qf}{\mathbf{Q}}

\newcommand{\Qstar}{\mathbf{Q}^{\star}}
\newcommand{\Vstar}{\mathbf{V}^{\star}}

\newcommand{\Qh}{\wh{\Qf}}

\newcommand{\Qhat}{\Qh}

\newcommand{\pihat}{\wh{\pi}}

\newcommand{\pistar}{\pi^{\star}}
\newcommand{\qstar}{q^{\star}}

\newcommand{\wpast}{\matw_{t-1}}

\newcommand{\ustar}{u^{\star}}
\newcommand{\Reg}{\mathrm{Reg}_T}

\newcommand{\Ttrunc}{T_{\mathrm{trunc}}}

\newcommand{\Mi}[1][i]{M^{[#1]}}

\newcommand{\Lt}{L_t}
\newcommand{\Linf}{L_{\infty}}

\newcommand{\kappapi}{\kappa_{0}}
\newcommand{\gammapi}{\gamma_{0}}
\newcommand{\bigohs}{\cO_{\star}}

\newcommand{\Dm}{D_{\cM}}
\newcommand{\picheck}{\check{\pi}}

\newcommand{\Ctrunc}{C_{\mathrm{trunc}}}
\newcommand{\Capx}{C_{\mathrm{apx}}}
\newcommand{\Creg}{C_{\mathrm{reg}}}

\newcommand{\Doco}{D_{\mathrm{oco}}}
\newcommand{\Goco}{G_{\mathrm{oco}}}
\newcommand{\alphaoco}{\alpha_{\mathrm{oco}}}

\newcommand{\Pt}{P_t}
\newcommand{\Ptpl}{P_{t+1}}
\newcommand{\Sigt}{\Sigma_t}
\newcommand{\Kt}{K_t}
\newcommand{\bias}{c}
\newcommand{\ct}{\bias_t}
\newcommand{\ctpl}{\bias_{t+1}}
\newcommand{\Aclinf}{A_{\mathrm{cl},\infty}}
\newcommand{\Aclt}[1][t]{A_{\mathrm{cl},#1}}
\newcommand{\Lclinf}{L_{\mathrm{cl},\infty}}
\newcommand{\Lclt}[1][t]{L_{\mathrm{cl},#1}}
\newcommand{\Deltast}{\Delta_{\mathrm{stab}}}
\newcommand{\Tstab}{T_{\mathrm{stab}}}
\newcommand{\kappast}{\kappa_{\infty}}
\newcommand{\gammabar}{\bar{\gamma}_{\infty}}
\newcommand{\gammab}{\gammabar}

\newcommand{\Dqst}{D_{q^{\star}}}
\newcommand{\Dq}{D_{q}}

\newcommand{\Sigmainf}{\Sigma_{\infty}}
\newcommand{\Siginf}{\Sigma_{\infty}}

\newcommand{\adv}{\mathbf{A}}
\newcommand{\advstar}{\mathbf{A}^{\star}}
\newcommand{\advhat}{\widehat{\mathbf{A}}}

\newcommand{\dapReg}{\cM_0\text{-}\Reg}
\newcommand{\pim}{\pi^{(M)}}
 \newcommand{\dap}{\sfterm{DAP}}
  \newcommand{\dapx}{\sfterm{DAP}\xspace}

\newcommand{\Psist}{\Psi_{\star}}
\newcommand{\Gammast}{\Gamma_{\star}}
\newcommand{\gammast}{\gamma_{\infty}}
\newcommand{\betast}{\beta_{\star}}

\newcommand{\cost}{J_{T}}
\newcommand{\system}{p}
\newcommand{\simiid}{\overset{\textrm{i.i.d.}}{\sim}}

\newcommand{\pilearn}{\wh{\pi}}

\newcommand{\mb}[1]{\boldsymbol{#1}}

\newcommand{\grad}{\nabla}

\newcommand{\I}{\mathbb{I}}

\newcommand{\trn}{\intercal}


%% file: abstract.tex

We introduce a new algorithm for online linear-quadratic control in a known system subject to adversarial disturbances. Existing regret bounds for this setting scale as $\sqrt{T}$ unless strong stochastic assumptions are imposed on the disturbance process. We give the first algorithm with logarithmic regret for arbitrary adversarial disturbance sequences, provided the state and control costs are given by known quadratic functions. Our algorithm and analysis use a characterization for the optimal offline control law to reduce the online control problem to (delayed) online learning with approximate \emph{advantage functions}. Compared to previous techniques, our approach does not need to control movement costs for the iterates, leading to logarithmic regret.


%% file: section_introduction.tex
\newcommand{\mdpe}{\textsf{MDP}\text{-}\textsf{E}}
\newcommand{\mdpex}{\mdpe\xspace}
\newcommand{\ogd}{\textsf{OGD}}
\newcommand{\ogdx}{\ogd\xspace}

\newcommand{\Alg}{\mathrm{alg}}
\newcommand{\Regret}{\mathrm{Regret}}
\newcommand{\Kclass}{\cK}
\newcommand{\olwm}{\textsf{OLwM}}
\newcommand{\olwmx}{\olwm\xspace}
\newcommand{\olwa}{\textsf{OLwA}}
\newcommand{\olwax}{\olwa\xspace}
\newcommand{\KReg}{\cK_0\text{-}\Reg}
\newcommand{\olws}{\textsf{OLwS}}
\newcommand{\olwsx}{\olws\xspace}
\newcommand{\loss}{\ell}
\newcommand{\poly}{\mathrm{poly}}

\newcommand{\R}{\mathbb{R}}
\newcommand{\Rx}{R_{x}}
\newcommand{\Ru}{R_{u}}
\newcommand{\pialg}{\pi^{\Alg}}

\newcommand{\dynamics}{x_{t+1}= A x_t + Bu_t + w_t,}

Reinforcement learning and control consider the behavior of an agent making decisions in a dynamic environment in order to suffer minimal loss. In light of recent practical breakthroughs in data-driven approaches to continuous RL and control \citep{lillicrap2015continuous,mnih2015human,silver2017mastering}, there is great interest in applying these techniques in real-world decision making applications. However, to reliably deploy data-driven RL and control in physical systems such as self-driving cars, it is critical to develop principled algorithms with provable safety and robustness guarantees. At the same time, algorithms should not be overly pessimistic, and should be able to take advantage of benign  environments whenever possible.

In this paper we develop algorithms for online linear-quadratic control which ensure robust worst-case performance while optimally adapting to the environment at hand.  Linear control has traditionally been studied in settings where the dynamics of the environment are either governed by a well-behaved stochastic process or driven by a worst-case process to which the learner must remain robust in the $\cH_{\infty}$ sense. We consider an intermediate approach introduced by \cite{agarwal2019online} in which disturbances are non-stochastic but performance is evaluated in terms of \emph{regret}. This benchmark forces the learner's control policy to achieve near optimal performance on any {specific} disturbance process encountered.

Concretely, we consider a setting in which the state evolves according to linear dynamics:
\begin{align}
  \label{eq:dynamics}
  \dynamics
\end{align}
where $x_t \in \R^{\dimx}$ are states, $u_t \in \bbR^{\dimu}$ are inputs, and $A\in\bbR^{\dimx\times\dimx}$ and $B\in\bbR^{\dimx\times{}\dimu}$ are system matrices known to the learner.  We refer to $w_t \in \R^{\dimx}$ as the \emph{disturbance} (or, ``noise''), which we assume is selected by an \emph{adaptive} adversary and satisfies $\nrm*{w_t}\leq{}1$; we let $\matw$ refer to the entire sequence $w_{1:T}$. We consider fixed quadratic costs of the form $\ell(x,u) := x^\top \Rx x + u^\top \Ru u$, where $\Rx,\Ru\psdgeq{}0$ are given. This model encompasses  noise which is uncorrelated ($\cH_{2}$), worst-case ($\cH_{\infty}$), or governed by some non-stationary stochastic process. The model also approximates control techniques such as feedback linearization and trajectory tracking \citep{slotine1991applied}, where $A$ and $B$ are the result of linearizing a known nonlinear system and the disturbances arise due to systematic errors in linearization rather than from a benign noise process. 

For any policy $\pi$ that selects controls based on the current state and disturbances observed so far, we measure its performance over a time horizon $T$ by
\begin{align*}
\cost(\pi;\matw) = \sum_{t=1}^{T}\ls(x_t^{\pi},u_t^{\pi}),
\end{align*}
the total cost incurred by following $u_t=\pi_t(x_t,w_{1:t-1})$. Letting $\pi^{K}$ denote a state-feedback control law of the form $\pi_t^{K}(x)=-Kx$ for all $t$, the learning algorithm's goal is to minimize
\begin{align*}
\Reg = \cost(\pialg;\matw) - \inf_{K\in\cK}\cost(\pi^{K};\matw),
\end{align*}
where $\pialg$ denotes the learner's policy and $\cK$ is an appropriately defined set of stabilizing controllers. Thus, $\pialg$ has low regret when its performance nearly matches the optimal controller $K \in \cK$ on the specific, realized noise sequence. While the class $\cK$ contains the optimal $\cH_{\infty}$ and $\cH_{2}$ control policies, we also develop algorithms to compete with a more general class of stabilizing linear controllers, which may fare better for certain noise sequences (\pref{app:generalization}).
~\\~\\
\noindent\textbf{\hbox{Achieving logarithmic regret in adversarial online control.}}~~
\cite{agarwal2019online} introduced the adversarial LQR setting we study and provided an efficient algorithm with $\sqrt{T}$-regret. Subsequent works \citep{agarwal2019logarithmic,simchowitz2020improper} have shown that logarithmic regret is possible when the disturbances follow a semi-adversarial process with persistent excitation. Our main result is to achieve logarithmic regret for fully adversarial disturbances, provided that costs are known and quadratic.

\subsection{Contributions\label{ssec:contributions}}
We introduce \mainalg (\pref{alg:main}), a new algorithm for online linear control with adversarial disturbances which attains polylogarithmic regret.
\newtheorem*{thm:ub_informal}{Theorem \ref*{thm:main_algo} (informal)}
\begin{thm:ub_informal}
  \mainalg attains regret $\bigoh\prn*{\log^3 T}$, where $\bigoh$ hides factors polynomial in relevant problem parameters. 
\end{thm:ub_informal}
\mainalg has comparable computational efficiency to previous methods. We show in \pref{app:generalization} that the algorithm also extends to a more general benchmark class of linear controllers with internal state, and to ``tracking'' loss functions of the form $\ell_t(x,u) := \ell(x - a_t,u - b_t)$. Some conceptual contributions are as follows.
~\\~\\
\noindent\textbf{\hbox{When is logarithmic regret possible in online control?}}~~\citet{simchowitz2020naive} and \cite{cassel2020logarithmic} independently show that logarithmic regret is impossible in a minimax sense if the system matrices $(A,B)$ are unknown, even when disturbances are i.i.d. gaussian. Conversely, our result shows that if $A$ and $B$ are known, logarithmic regret is possible even when disturbances are adversarial. Together, these results paint a clear picture of when logarithmic regret is achievable in online linear control. We note, however, that our approach heavily leverages the structure of linear control with \emph{strongly convex, quadratic} costs. We refer the reader to the related work section for discussion of further structural assumptions that facilitate logarithmic regret.

\paragraph{Addressing trajectory mismatch.}
\mainalg represents a new approach to a problem we call \emph{trajectory mismatch} that arises when considering policy regret in online learning problems with state. In dynamic environments, different policies inevitably visit different state trajectories. Low-regret  algorithms must address the mismatch between the performance of the learner's policy $\pialg$ on its own realized trajectory and the performance of each benchmark policy $\pi$ on the alternative trajectories it would induce.  Most algorithms with policy regret guarantees \citep{even2009online,zimin2013online,yadkori2013online,arora2012online,anava2015online,abbasi2014tracking,cohen2018online,agarwal2019online,agarwal2019logarithmic,simchowitz2020improper} adopt an approach to addressing this trajectory mismatch that we refer to as ``online learning with stationary costs'', or \olws. At each round $t$, the learner's adaptive policy $\pialg$ commits to a policy $\pi^{(t)}$, typically from a benchmark class $\Pi$. The goal is to ensure that the iterates $\pi^{(t)}$ attain low regret on a proxy sequence of  \emph{stationary}  cost functions $\pi \mapsto \lambda_t(\pi)$ that describes the loss the learner would suffer at stage $t$ under the \emph{fictional} trajectory that would arise if she had played the policy $\pi$ at all stages up to time $t$ (or in some cases, on the corresponding steady-state trajectory as $t \to \infty$).  Since the stationary cost does not depend on the learner's state, low regret on the sequence $\{\lambda_t\}$ can be obtained by feeding these losses directly into a standard online learning algorithm. To relate regret on the proxy sequence back to regret on the true sequence, most approaches use that the iterates produced by the online learner are sufficiently slow-moving. \pref{app:olws_limits} explains both the general \olwsx paradigm, and its instantiation for online control, in further detail.

The main technical challenge \mainalg overcomes is that for the stationary costs that arise in our setting, no known algorithm produces iterates which move sufficiently slowly to yield logarithmic regret via \olwsx (\pref{app:exp_concave_perils}). We adopt a new approach for online control we call \emph{online learning with advantages}, or $\olwa$, which abandons stationary costs, and instead considers the control-theoretic advantages of actions relative to the unconstrained offline optimal policy $\pistar$. Somewhat miraculously, we find that these advantages remove the explicit dependence on the learner's state, thereby eliminating the issue of trajectory mismatch described above. In particular, unlike \olwsx, we \emph{do not} need to verify that the iterates produced by our algorithm change slowly. 

\subsection{Our approach: Online learning with advantages}
In this section we sketch the \emph{online learning with advantages} (\olwa) technique underlying \mainalg. Let $\pistar$ denote the optimal unconstrained policy given knowledge of the entire disturbance sequence $\matw$, and let $\Qstar_t(x,u;\matw)$ be the associated Q-function (this quantity is formally defined in \pref{def:optimal}). The \emph{advantage} with respect to $\pistar$, $\advstar_t(u;x,\matw)\ldef\Qstar_t(x,u;\matw)-\Qstar_t(w,u,\pistar(x);\matw)$, describes the difference between the total cost accumulated by selecting action $u$ in state $x$ at time $t$ and subsequently playing according to the optimal policy $\pistar$, versus choosing $u = \pistar_t(x;\matw)$ as well.\footnote{Since we use losses rather than rewards, ``advantage'' refers to the advantage of $\pistar$ over $u$ rather than the advantage of $u$ over $\pistar$; the latter terminology is more common in reinforcement learning.} By the well-known performance difference lemma \citep{kakade2003sample}, the relative cost of a policy is \emph{equal} the sum of the advantages under the states visited by said policy:\footnote{See \pref{lem:pd} in \pref{app:algorithm} for a general statement of the performance difference lemma. The invocation of the performance difference lemma here is slightly different from other results on online learning in MDPs such as \cite{even2009online}, in that the role of $\pi$ and $\pistar$ is swapped.}
\begin{align}
\cost(\pi;\matw) - \cost(\pistar;\matw) = \sum_{t=1}^T \advstar_t(u^{\pi}_t;x^{\pi}_t,\matw). \label{eq:cost_perf_diff}
\end{align}
With this observation, the regret of any algorithm $\pialg$ to a policy class $\Pi$ can be expressed as:
\iftoggle{icml}
{
	\begin{multline}
\Reg(\pialg;\matw,\Pi) =  \sum_{t=1}^T \advstar_t(u^{\Alg}_t;x^{\Alg}_t,\matw) \\-  \inf_{\pi \in \Pi} \sum_{t=1}^T \advstar_t(u^{\pi}_t;x^{\pi}_t,\matw). \label{eq:reg_perf_diff}
\end{multline}
}
{
	\begin{align}
\Reg(\pialg;\Pi,\matw) =  \sum_{t=1}^T \advstar_t(u^{\Alg}_t;x^{\Alg}_t,\matw) -  \inf_{\pi \in \Pi} \sum_{t=1}^T \advstar_t(u^{\pi}_t;x^{\pi}_t,\matw). \label{eq:reg_perf_diff}
\end{align}
}

The expression~\pref{eq:reg_perf_diff} suggests that a reasonable approach might be to run an online learner on the functions $\pi \mapsto \advstar_t(u^{\pi}_t;x^{\pi}_t,\matw)$. However, there are two issues. First, the advantages in the first sum are evaluated on the states $x^{\Alg}_t$ under $\pialg$, and in the second sum under the comparator trajectories $x^{\pi}$ (trajectory mismatch). Second, like $\pistar$ itself, the advantages require knowledge of all future disturbances, which are not yet known to the learner at time $t$. We show---somewhat miraculously---that if the control policies are parametrized using a particular optimal control law, the advantages do not depend on the state, and can be approximated using only finite lookahead.
\newtheorem*{thm:reg_decomp_informal}{Theorem \ref*{thm:main_reg_decomp} (informal)}
\begin{thm:reg_decomp_informal}\label{thm:adv_informal}
  For control policies $\pi$ with a suitable parametrization, the mapping $\pi \mapsto \advstar_t(u^{\pi}_t;x^{\pi}_t,\matw)$ can be arbitrarilily-well approximated by a function $\pi \mapsto \advhat_{t;h}(\pi;w_{1:t+h})$ which (1) does not depend on the state, (2) can be determined by the learner at time $t+h$, and (3) has a simple quadratic structure.
\end{thm:reg_decomp_informal}
The ``magic'' behind this theorem is that the functional dependence of the unconstrained optimal policy $\pistar(x;\matw)$ on the state $x$ is linear, and does not depend $\matw$ (\pref{thm:pistar_form}). As a consequence, the state-dependent portion of $\pistar$ can be built into the controller parametrization, leaving only the $\matw$-dependent portion up to the online learner. In light of this result, we use online learning to ensure low regret on the sequence of loss functions $f_t(\pi) \ldef \advhat_{t;h}(\pi;w_{1:t+h})$; we address the fact that $f_t$ is only revealed to the learner after a delay of $h$ steps via a standard reduction \citep{joulani2013online}. We then show that for an appropriate controller parameterization $f_{t}(\pi)$ is exp-concave with respective to the learner's policy and hence second-order online learning algorithms attain regret \citep{hazan2007logarithmic}.

We refer the reader to \pref{app:olws_limits} for an in-depth overview of the \olwsx framework, its relationship to \olwa, and challenges associated with using these techniques to achieve logarithmic regret.

\subsection{Related work}

\noindent\textbf{\hbox{Linear control for known systems.}}~~\citet{cohen2018online} establish $\sqrt{T}$ regret for online control of known linear systems under stochastic noise and time varying quadratic cost. \citet{agarwal2019online} achieve $\sqrt{T}$-regret with both \emph{adversarial} disturbances and time varying, adversarially chosen loss functions $\ell_t$ via a reduction to online convex optimization with memory \citep{anava2015online}. Their approach adopts a ``disturbance-action''  policy parameterization (or, \dapx), which we utilize as well (\pref{def:dap}). Certain previous results achieve logarithmic regret  by making assumptions that ensure stationary costs are strongly convex, allowing for logarithmic regret and movement cost via \citet{anava2015online} or similar arguments. \citet{abbasi2014tracking} consider an online tracking problem with known system parameters \emph{zero} exogenous noise. The absence of noise enables an approach based on \mdpe{}  (see \pref{app:mdpe}), for which the relevant Q-functions in this setting are strongly convex, leading to logarithmic regret. More recently \citet{agarwal2019logarithmic} showed that in the noisy setting the stationary costs $\lambda_t$ themselves are strongly convex in a disturbance-action parametrization, provided that the loss functions $\loss_t$ are strongly convex and the noise covariance is well-conditioned, which also leads to logarithmic regret. \cite{simchowitz2020improper} show that this approach extends to ``semi-adversarial'' disturbances with a well-conditioned stochastic component and a possibly adversarial component. Our results (with the restriction that costs are quadratic) give the first logarithmic regret bounds for the fully adversarial setting and, to the best of our knowledge, give the first instance in online control where an exp-concave but not strongly convex parametrization attains logarithmic regret. 

\paragraph{Linear control for unknown systems.} For \emph{unknown} systems, various works \citep{abbasi2011regret,faradonbeh2018optimality,cohen2019learning,mania2019certainty} establish $\sqrt{T}$-regret for fixed quadratic losses and stationary stochastic noise, which is optimal for this setting \citep{simchowitz2020naive,cassel2020logarithmic}.  Because of the stochastic nature of these problems, purely statistical techniques suffice. By combining these techniques with OCO with memory \citep{anava2015online}, other recent works have addressed both unknown dynamics and adversarial noise \citep{hazan2019nonstochastic,simchowitz2020improper}. \citep{cassel2020logarithmic} show that logarithmic regret \emph{is} achievable under stochastic noise for systems $(A,B)$ where only $A$ is unknown, or where only $B$ is unknown and the optimal controller satisfies a non-degeneracy assumption.

\paragraph{Online reinforcement learning.}
Online linear control belongs to a broader line of work on online reinforcement learning in (known or unknown) Markov decision processes with adversarial costs or transitions. Given the staggering breadth of work in this direction from the online learning, control, and RL communities, we focus on past contributions which are most closely related to our setting. As discussed earlier, essentially all prior approaches to online RL abide by the \olwsx paradigm. Perhaps the first result in this direction is the \mdpe{} algorithm of \citet{even2009online}, which attains $\sqrt{T}$ policy regret in a tabular MDP with known stationary dynamics and adversarially chosen rewards. Subsequent works \citep{yadkori2013online} achieves $\sqrt{T}$-regret in a tabular setting where the both the rewards and transition kernels are selected by an adversary. A parallel line of work on adversarial tabular MDPs considers the episodic setting \citep{zimin2013online,rosenberg2019online}, which alleviates the need to bound the movement costs between iterates.

\paragraph{Policy regret.}
All of the approaches described so far can be viewed as special cases of the general problem of minimizing \emph{policy regret} in online learning. A finite-memory formulation of the policy regret benchmark was popularized by \citet{arora2012online}. \citet{anava2015online} generalize this result to the \emph{online convex optimization with memory} setting and demonstrate that many popular online learning algorithms naturally produce slow-moving iterates, yielding near-optimal policy regret bounds (see \pref{app:oco_memory} for detailed discussion). These results have found immediate application in online linear control \citep{agarwal2019online,hazan2019nonstochastic,simchowitz2020improper}. However, the analysis of \citet{anava2015online} does not extend to give fast rates for the exp-concave loss functions which arise in our setting.

\subsection{Preliminaries \label{sec:preliminaries}}
We consider the linear control setting in \pref{eq:dynamics}. For normalization, we assume $\|w_t\| \le 1$ for all $t$. We also assume $x_1=0$. A comprehensive summary of all notation used throughout the paper is provided in \pref{table:notation} in the appendix.

\paragraph{Policies and trajectories.} We consider policies $\pi$ parameterized as functions of $x_t$ and $\matw$ via $u_t = \pi_t(x_t;\matw)$. We assume that, when selecting action $u_t$ at time $t$, the learner has access to all states $x_{1:t},u_{1:t-1}$, as well as $w_{1:t-1}$ (the latter assumption is without loss of generality by the identity $w_s = x_{s+1} - Ax_{s} - Bu_s$). Thus, a policy is said to be \emph{executable} if $\pi_t(x;\matw)$ depends only on $x$ and $w_{1:t-1}$, i.e. $\pi(x;\matw)=\pi(w;w_{1:t-1})$. For analysis purposes, we also consider \emph{non-executable} whose value at time $t$ may depend on the entire sequence $\matw$. For a policy $\pi$ and sequence $\matw$, we let $x_t^{\pi}(\matw),u_t^{\pi}(\matw)$ denote the resulting states and input trajectories (which we note depend only on $w_{1:t-1}$). For simplicity, we often write $x_t^\pi$ and $u_t^\pi$, supressing the $\matw$-dependence. We shall let $\pi^{\Alg}$ refer to the policy selected by the learner's algorithm, and use the shorthand $x_t^{\Alg}(\matw)$, $u_t^{\Alg}(\matw)$ to denote the corresponding trajectories. Given a class of policies $\Pi$, the regret of the policy $\pialg$ is given by
\begin{align*}
\Reg(\pialg;\Pi,\matw) = \cost(\pialg;\matw) - \inf_{\pi \in \Pi}\cost(\pi;\matw).
\end{align*}
We consider a benchmark class of policies induced by state feedback control laws $\pi^{K}_t(x)=-Kx$, indexed by  matrices $K \in \R^{\dimu \times \dimx}$.
\paragraph{Linear control theory.}
We say that a linear controller $K \in \R^{\dimu \dimx}$ is \emph{stabilizing} if $A-BK$ is \emph{stable}, that is $\rho(A-BK) < 1$ where $\rho(\cdot)$ denotes the spectral radius.\footnote{For a possibly asymmetric matrix $A$, $\rho(A)=\max\crl*{\abs*{\lambda}\mid{}\text{$\lambda$ is an eigenvalue for $A$}}$.} We assume the system $(A,B)$ is \emph{stabilizable} in the sense that there exists a stabilizing controller $K$. For any stabilizable system, there is a unique positive semidefinite solution $\Pinf\psdgeq{}0$ to the \emph{discrete algebraic
  Riccati equation} (henceforth, \dare),
  \begin{align}
    \label{eq:dare}
    P = A^{\trn}PA + \Rx - A^{\trn}PB(\Ru+B^{\trn}PB)^{-1}B^{\trn}PA.
  \end{align}
The solution $\Pinf$ to \pref{eq:dare} is an intrinsic property of the system \pref{eq:dynamics} with $(A,B)$ and characterizes the optimal infinite-horizon cost for control in the absence of noise \citep{bertsekas2005dynamic}. Our algorithms and analysis make use of this parameter, as well as the corresponding optimal state feedback controller $\Kinf\ldef{}(\Ru+B^{\trn}\Pinf{}B)^{-1}B^{\trn}\Pinf{}A$. We also use the steady-state covariance matrix $\Siginf\ldef{}\Ru+B^{\trn}\Pinf{}B$ and closed-loop dynamics matrix $\Aclinf\ldef{}A-B\Kinf$.

\paragraph{Competing with state feedback.}
While $\Kinf$ represents the (asymptotically) optimal control law in the presense of uncorrelated, unbiased stochastic noise, $\pi^{\Kinf}$ may not be the optimal state feedback policy in hindsight for a given sequence of adversarial perturbations $w_t$.
We compete with linear controllers that satisfy a quantitative version of the stability property.
\begin{definition}[Strong Stability \citep{cohen2018online}]
\label{def:ss}
We say that $A-BK\in \R^{\dimx\times\dimx}$ is $(\kappa,\gamma)$-strongly stable if there exists matrices $H,L \in \R^{\dimx\times\dimx}$ such that $A - BK= HLH^{-1}$, $\|H\|_{\op}\|H\|_{\op}^{-1} \le \kappa$ and $\|L\|_{\op} \le \gamma$.
\end{definition}
Given parameters $(\kappapi,\gammapi)$, we consider the benchmark class
	\begin{align*}
	\Kclass_0 = \crl*{K\in\bbR^{\dimu\times\dimx}\mid{} \text{$A-BK$ is
    $(\kappapi,\gammapi)$-strongly stable and $\nrm*{K}_{\op}\leq\kappapi$}}.
	\end{align*}
        \pref{lem:strongly_stable} (\pref{app:lqr_structural}) shows that the closed-loop dynamics for
  $\Kinf$ are always $(\kappast,\gammast)$-strongly stable for
  suitable $\gammast,\kappast$. We assume that $\cK_0$ is
chosen such that $\kappast\leq{}\kappapi$ and $\gammast\leq{}\gammapi$.\footnote{This assumption only serves to keep notation compact.} Our algorithms minimize policy regret to the class of induced policies for $\Kclass_0$:
\begin{align*}
\KReg(\pi^{\Alg};\matw) := J_T(\pi^{\Alg};\matw) - \inf_{K \in \Kclass_0}J_T(\pi^K;\matw).
\end{align*}

\paragraph{Problem parameters.}
Our regret bounds depend on the following basic parameters for the LQR problem:
$\Psist\ldef\max\crl*{1,\nrm*{A}_{\op},\nrm*{B}_{\op},\nrm*{\Rx}_{\op},\nrm*{\Ru}_{\op}}$, ${\betast\ldef\max\crl*{1,\eigmin^{-1}(\Ru),\eigmin^{-1}(\Rx)}}$, and  ${\Gammast\ldef{}\max\crl*{1,\nrm*{\Pinf}_{\op}}}$.
  
\paragraph{Additional notation.}\label{ssec:other_prelim}
	We adopt non-asymptotic big-oh notation: For functions
	$f,g:\cX\to\bbR_{+}$, we write $f=\bigoh(g)$ if there exists some constant
	$C>0$ such that $f(x)\leq{}Cg(x)$ for all $x\in\cX$. We use $\bigoht(\cdot)$ so suppress logarithmic dependence on system parameters, and we use $\bigohs(\cdot)$ to suppress \emph{all} dependence on system parameters.\iftoggle{icml}{}{

	}
	For a vector $x\in\bbR^{d}$, we let $\nrm*{x}$ denote the euclidean
	norm and $\nrm*{x}_{\infty}$ denote the element-wise $\ls_{\infty}$
	norm. For a matrix $A$, we let $\nrm*{A}_{\op}$ denote the
        operator norm. If $A$ is symmetric, we let $\eigmin(A)$ denote the
	minimum eigenvalue. When $P\psdgt{}0$ is a positive definite matrix,
	we let $\nrm*{x}_{P}=\sqrt{\tri*{x,Px}}$ denote the induced weighted
	euclidean norm. 
	Let $\matw_{t-1} = (w_{t-1},w_{t-2},\dots, w_1,\mb{0},\mb{0},\dots)$ denote a sequence of past ws, terminating in an infinite sequence of zeros. To simplify indexing, we let $w_s \equiv \mb{0}$ for $s \le 0$, so that $\matw_{t-1} = (w_{t-1},w_{t-2},\dots)$ We also let $w_s \equiv \mb{0}$ for $s > T$. 

        \subsection{Organization}
        \pref{sec:algorithms} introduces the \mainalg algorithm, states its formal regret guarantee, and gives an overview of the algorithm's building blocks and proof techniques. \pref{sec:analysis} gives a high-level proof of the key ``approximate advantage'' theorem used by the algorithm. Omitted proofs are deferred to \pref{app:algorithm} and \pref{app:analysis}, and additional technical tools stated and proven in \pref{app:technical}. \pref{sec:beyond_linear} describes a generalization of the main algorithm and \pref{sec:conclusion} concludes with discussion and further directions.

        \pref{app:generalization} sketches extensions of \mainalg to more general settings, and \pref{app:olws_limits} gives a detailed survey of challenges associated with applying previous approaches to online reinforcement learning to obtain logarithmic regret in our setting.


%% file: section_algorithms.tex

\newcommand{\mainalgdp}{\textsf{RiccatitronDP}\xspace}
\newcommand{\Mnot}{\cM_{0}}
\newcommand{\ons}{\textsf{ONS}}
\newcommand{\onsx}{\textsf{ONS}\xpsace}

Our main algorithm, \mainalg, is described in \pref{alg:main}. The
algorithm combines several ideas.
\begin{enumerate}
\item Following \cite{agarwal2019online}, we move from linear
  policies of the form $\pi^{K}(x;\matw)=-Kx$, to a relaxed set of
  \emph{disturbance-action} (\dapx) policies of the form
  \[
\dapdef,
\]
where $\Kinf$ is linear controller arising from the \dare \pref{eq:dare}.
\item We show that the optimal unconstrained policy with full
  knowledge of the sequence $\matw$ takes the form
  $\pistar_t(x;\matw)=-K_tx-q^{\star}_t(w_{t:T})$, where $(K_t)$ is a
  particular sequence of linear controllers that arises from the
  so-called \emph{Riccati recursion}. We then show that for \emph{any} policy of the form
  $\pi_t(x;\matw)=-K_\infty - q_t(\matw)$---in particular, for the
  \dapx parameterization above---the advantage functions
  $\advstar_t(u_t^{\pi};x_t^{\pi},\matw)$ can be well approximated by
  simple quadratic functions of the form
  \[
    \nrm*{q_t(\matw)-q_t^{\star}(w_{t:T})}_{\Siginf}^2.
  \]
  This essentially removes the learner's state from the equation, and
  reduces the problem of control to that of predicting the optimal controller's bias
  vector $q^{\star}_t(w_{t:T})$. The remaining challenge is that the optimal
  bias vectors depend on the future disturbances, which are not available
  to the learner at time $t$.
\item We show that the advantages can be truncated to require only
  finite lookahead, thereby reducing the problem to \emph{online
    learning with delayed feedback}. We then apply a generic reduction
  from delayed online learning to classical online learning
  \citep{joulani2013online}, which proceeds by running multiple copies
  of a base online learning algorithm over separate subsequences of rounds.
\item Finally---leveraging the structure of the disturbance-action
  parameterization---we show that the resulting online learning
  problem is exp-concave. As a result, we can apply a second-order
  online learning algorithm---either online Newton step 
  (\ons,  \cite{hazan2007logarithmic}) given in \pref{alg:ons}, or Vovk-Azoury-Warmuth (\vaw, \cite{Vovk98,AzouryWarmuth01}) detailed in \pref{alg:vaw}---as our base learner to obtain logarithmic regret.
\end{enumerate}
Together, these components give rise to the scheme in
\pref{alg:main}. At time $t$, the algorithm plays the action
$u_t=-\Kinf{}x_t-q^{M_t}(\matw_{t-1})$, where $M_t$ is provided by the
\ons{} (or \vaw) instance responsible for the current round. The algorithm then
observes $w_{t}$ and uses this to form the approximate advantage
function for time $t-h$, where $h$ is the lookahead distance. The
advantage is then used to update the \ons{}/\vaw
instance responsible for the next round. The main regret guarantee for
this approach is as follows.
\begin{restatable}{theorem}{mainalgo}
  \label{thm:main_algo}
  For an appropriate choice of parameters, \mainalg ensures
  \begin{align*}
    \KReg \leq{} \bigohs(\dimx\dimu\log^{3}T),
  \end{align*}
  where $\bigohs$ suppresses polynomial dependence on system
  parameters. Suppressing only logarithmic dependence on system parameters, the regret is at most
  \begin{align*}
	\bigoht\prn*{\dimx\dimu\log^{3}T \cdot 
    \betast^{11}\Psist^{19}\Gammast^{11}\kappapi^{8} (1-\gammapi)^{-4}
    }.
  \end{align*}
\end{restatable}
In the remainder of this section we overview the algorithmic building blocks of \mainalg and the key ideas of the
proof. \pref{ssec:dap} reviews disturbance-action policy
parametrization. \pref{ssec:advantages_linear_control} describes the
formal construction of advantages $\advstar_t$ for linear control, and
the regret decomposition which ensues. \pref{ssec:adv_without_states}
presents the approximate advantages $\advhat_{t;h}$, which have
numerous properties amenable to online control---notably, no explicit
dependence on system state. \pref{ssec:main_alg_description}
introduces the delayed online learning reduction and 
\pref{sec:exp_concave} instantiates the reduction with online Newton
step (\pref{alg:ons}). \pref{sec:sharpen} uses Vovk-Azoury-Warmuth as the base learner
to sharpen the final regret bound (\pref{alg:vaw}). Extensions can be found in
\pref{sec:beyond_linear} and \pref{app:generalization}.
\input{algorithms}
\subsection{Disturbance-action policies \label{ssec:dap}}
Cost functionals parametrized by state feedback controllers (e.g., $K
\mapsto \cost(\pi^K;\matw)$) are generally non-convex
\citep{fazel2018global}.
To enable the use of tools from online convex
optimization, we adopt a convex \emph{disturbance-action} controller
parameterization introduced by \cite{agarwal2019online}.\vspace{-15pt}
\begin{definition}[Disturbance-action policy]\label{def:dap}  Let $M =
  (M^{[i]})_{i=1}^m$ denote a sequence of matrices $M^{[i]} \in
  \R^{\dimu\times{}\dimx}$.  We define the corresponding disturbance-action policy {\normalfont(\dap)} $\pim$ as
\iftoggle{icml}
{
  $\dapdef$.
}
{
  \begin{align}
  \dapdef.\label{eq:dap}
\end{align} 
}
\end{definition}
We work with \dap{}s for which the sequence $M$ belongs to the
set \begin{equation}
\label{eqn:dapset}
\cM(m,R,\gamma) := \{M = (M^{[i]})_{i=1}^m: \|M^{[i]}\|_{\op}\le
  R\gamma^{i-1}\},\end{equation}
where $m$, $R$, and $\gamma$ are algorithm parameters. We note that \dap{}s can be defined with general stabilizing
controllers $K \ne \Kinf$, but the choice $K = \Kinf$ is critical in the design and analysis of our main algorithm.

The first lemma we require is a variant of a result of \citet{agarwal2019online},
which shows that disturbance-action policies are sufficiently rich
enough to approximate all state feedback laws.
\begin{restatable}[Expressivity of \dapx]{lemma}{disturbancesufficient}
  \label{lem:disturbance_sufficient}
  Suppose we choose our set of
disturbance-action matrices as 
\begin{equation}
  \label{eq:dapset}
\cM_0\ldef{}\cM(m,\Rst, \gammapi),\quad\text{where}\quad m=(1-\gammapi)^{-1}\log((1-\gammapi)^{-1}T),\quad\text{and}\quad \Rst=2\betast\Psist^{2}\Gammast\kappapi^{2}.
\end{equation}
Then for all $\matw$, we have
  \begin{align*}
    \inf_{M\in\cM_0}\cost(\pi\ind{M};\matw)
      &\leq{} \inf_{K\in\cK_0}\cost(\pi^{K};\matw) + \Capx,
  \end{align*}
  where $\Capx\leq{}\bigoh(\betast^{2}\Psist^{8}\Gammast^{2}\kappapi^{7}(1-\gammapi)^{-2})$.
\end{restatable}
\begin{algorithm}[t]
  \begin{algorithmic}[1]
    \State \textbf{parameters}: Learning rate
    $\eta>0$, regularization parameter $\veps>0$, convex constraint set $\cC$.
      \Statex{}\algcomment{OCO with exp-concave costs $f_k(z)$, where $z\in\cC\subset\bbR^{d}$.}
    \iftoggle{icml}
    {
      \State \textbf{initialize}: $d\leftarrow\mathrm{dim}(\cC)$,  $z_0\gets\mb{0}_{d}$, 
    $E_0\gets\veps\cdot{}I_{d}$
    }
    {
      \State \textbf{initialize}: 
      \Statex ~~~~Let $d=\mathrm{dim}(\cC)$. 
    \Statex~~~~Set
    $z_1\in\cC$ and
    $E_0=\veps\cdot{}I_{d}$.
    }
    \For{$k=1,2,\dots$:}
    \State{} Play $z_{k}$ and \textbf{receive} gradient $\grad_{k}\ldef{}\grad{}f_k(z_k)$.
    \State{} $E_{k} \gets{}E_{k-1} +\grad_{k}\grad_{k}^{\trn}$.
    \State{} $\wt{z}_{k+1} \gets{} z_{k} - \eta E_{k}^{-1}\grad_{k}$.
    \State{} $z_{k+1}\gets\argmin_{z\in\cC}\nrm*{z-\wt{z}_{k+1}}^{2}_{E_{k}}$.
    \EndFor
  \end{algorithmic}
  \caption{Online Newton Step ($\ons(\veps,\eta,\cC,\Sigma)$)}
  \label{alg:ons}
\end{algorithm}
\begin{algorithm}[t]
  \begin{algorithmic}[1]
    \State \textbf{parameters}: Regularization parameter $\veps>0$,
    convex constraint set $\cC$, cost matrix $\Sigma\psdgt{}0$.
          \Statex{}\algcomment{OCO with costs $f_k(z):= \|A_k z - b_k\|_{\Sigma}^2$,
      where $A_k\in\bbR^{d_1\times{}d_2}$, $b_k\in\bbR^{d_1}$ and $z\in\cC\subset\bbR^{d_2}$.}
    \iftoggle{icml}
    {
      \State \textbf{initialize}: $d\leftarrow\mathrm{dim}(\cC)$,  $z_0\gets\mb{0}_{d}$, 
    $E_0\gets\veps\cdot{}I_{d}$
    }
    {
      \State \textbf{initialize}: 
      \Statex ~~~~Let $d_2=\mathrm{dim}(\cC)$. 
    \Statex~~~~Set
    $E_0=\veps\cdot{}I_{d_2}$.
    }
    \For{$k=1,2,\dots$:}
    \State{} \textbf{receive} matrix $A_k\in\bbR^{d_1\times{}d_2}$.
    \State{} $E_{k}\gets{}E_{k-1} + A_{k}^{\trn}\Sigma{}A_{k}$.
    \State{} $z_{k} \gets{}
    \argmin_{z\in\cC}\crl*{\tri*{z,{\textstyle-2\sum_{i=1}^{k-1}A_i^{\trn}\Sigma{}b_i}}
      + \nrm*{z}_{E_{k}}^{2}}$
    \State{} Play $z_k$ and \textbf{receive} feedback $b_k\in\bbR^{d_1}$.
    \EndFor
  \end{algorithmic}
  \caption{Vector-valued Vovk-Azoury-Warmuth ($\vaw(\veps,\cC,\Sigma)$)}
  \label{alg:vaw}
\end{algorithm}

We refer the reader to \pref{app:algorithm_supporting} for a proof. Going forward, we define
  \begin{align}
  \Dq = \wt{O}\prn*{
    \betast^{5/2}\Psist^{3}\Gammast^{5/2}\kappapi^{2}(1-\gammapi)^{-1}
    },
    \label{eq:dq}
  \end{align}
  which serves as an upper bound on $\nrm*{q_t^{M}}$ for $M\in\cM_0$,
  as well as other certain other bias vector sequences that arise in the
  subsequent analysis. In light of \pref{lem:disturbance_sufficient},
  the remainder of our discussion will directly bound regret with respect to \dap{}s:
\begin{align}
\dapReg(\pi;\matw) := \cost(\pi;\matw) - \inf_{M \in \Mnot}\cost(\pim;\matw). \label{eq:dap_reg}
\end{align}
We note in passing that \dap{}s are actually rich enough to compete
with a broader class of linear control policies with internal state; this extension is addressed in
\pref{app:other_regret_benchmarks}.
\subsection{Advantages in linear control \label{ssec:advantages_linear_control}}
To proceed, we adopt the \olwax paradigm, which minimizes approximations to
the \emph{advantages} (or, differences between the Q-functions)
relative to the optimal unconstrained policy $\pistar$ given access to
the entire sequence $\matw$.  Recalling $\ls(x,u) = \nrm*{x}^{2}_{\Rx}
+ \nrm*{u}^{2}_{\Ru}$, we define the optimal controller $\pistar$ and
associated Q-functions and advantages by induction.
\begin{definition}\label{def:optimal}
  The optimal Q-function and policy at time $T$ are given by
  \[\Qstar_T(x,u;\matw)=\ls(x,u),\quad\pistar_T(x;\matw) = \min_u
  \Qstar_T(x,u;\matw) = 0,\quad \text{and}\quad
  \Vstar_T(x;\matw)=\ls(x,0)=\nrm*{x}_{\Rx}^{2}.\] For each timestep
$t<T$, the optimal Q-function and policy are given by
  \iftoggle{icml}
  {
    $\Qstar_{t}(x,u;\matw) = \nrm*{x}_{Q}^{2} + \nrm*{u}_{R}^{2} +
      \Vstar_{t+1}(Ax+Bu+w_t; \matw)$, 
      $\pistar_t(x;\matw)=\argmin_{u\in\bbR^{\dimu}}\Qstar_{t}(x,u;\matw)$, with value function 
    $\Vstar_{t}(x;\matw) = \min_{u\in\bbR^{d}}\Qstar_{t}(x,u;\matw) = \Qstar_{t}(x,\pistar_t(x;\matw);\matw)$.
  }
  {
  \begin{align*}
    &\Qstar_{t}(x,u;\matw) = \nrm*{x}_{Q}^{2} + \nrm*{u}_{R}^{2} +
      \Vstar_{t+1}(Ax+Bu+w_t; \matw),\\
      &\pistar_t(x;\matw)=\argmin_{u\in\bbR^{\dimu}}\Qstar_{t}(x,u;\matw),\\
    &\Vstar_{t}(x;\matw) = \min_{u\in\bbR^{\dimu}}\Qstar_{t}(x,u;\matw) = \Qstar_{t}(x,\pistar_t(x;\matw);\matw).
  \end{align*}
}
The advantage function for the optimal policy is $\adv_t^{\star}(u;x,\matw) := \Qstar_t(x,u;\matw) - \Qstar_t(x,\pistar_t(x;\matw);\matw)$.
\end{definition}
The advantage function $\advstar_t(u;x,\matw)$ represents the
total excess cost incurred by selecting a control
$u\neq{}\pistar_t(x;\matw)$ at state $x$ and time $t$, assuming we
follow $\pistar$ for the remaining rounds. We have
$\adv_t^{\star}(u;x,\matw)  \ge 0$ since, by Bellman's optimality
condition, $\pistar_t(x;\matw)$ is a minimizer of $\Qstar(x,u;\matw)
$.

The advantages arise in our setting through application of the
performance difference lemma (\pref{lem:pd}), which we recall
states that for any policy $\pi$, the regret to $\pistar$ is equal to
the sum of advantages under the trajectory induced by $\pi$,
i.e. $\cost(\pi;\matw) - \cost(\pistar;\matw)=\sum_{t=1}^T
\advstar_t(u^{\pi}_t;x^{\pi}_t,\matw)$. To analyze \mainalg, we apply
this identity to obtain the regret decomposition
  \iftoggle{icml}
  {
    \begin{multline}
\dapReg(\pi;\matw) = \sum_{t=1}^T \adv_t^{\star}(u^{\pi}_t;x^{\pi}_t,\matw) \\ - \inf_{M \in \cM_0}\sum_{t=1}^T \adv_t^{\star}(u^{\pim}_t;x^{\pim}_t,\matw) \label{eq:adv_rep}
\end{multline}
  }
  {
    \begin{align*}
\dapReg(\pi;\matw) = \sum_{t=1}^T \adv_t^{\star}(u^{\pi}_t;x^{\pi}_t,\matw) - \inf_{M \in \cM_0}\sum_{t=1}^T \adv_t^{\star}(u^{\pim}_t;x^{\pim}_t,\matw). \label{eq:adv_rep}
\end{align*}
  }
This decomposition is \emph{exact}, and avoids the pitfalls of the usual
stationary cost-based regret decomposition associated with the
classical \olwsx approach (cf. \pref{app:olws_limits}). Our goal going
forward will be to treat these advantages as ``losses'' that can be
fed into an appropriate online learning algorithm to select controls. However, this approach presents three challenges: (a) the advantages for
the policy $\pi$ are evaluated on the trajectory
$x^{\pi}_t$, while the advantages for
comparator are evaluated under the trajectory induced by $\pim$; (b)
the advantage is a difference in Q-functions that considers \emph{all}
future expected reward. In particular,
$\adv_t^{\star}(\cdot;\cdot,\matw)$ depends on all future $w_t$s,
including those not yet revealed to the learner; (c) the functional
form of the advantages is opaque, and it is not clear that any online
learning algorithm can achieve logarithmic regret even if they were able to evaluate $\adv_{t}^{\star}$ at time $t$.
\subsection{Approximate advantages \label{ssec:adv_without_states}}
Our main structural result---and the starting point for \mainalg---is the following observation. Let $\pi$ be any policy of the form $\pi_t(x;\wpast) = -\Kinf x -
q^{M_t}(\wpast)$, where $M_t = M_t(\wpast)$ are arbitrary functions of
past $w$, and where $\Kinf$ is the infinite horizon Riccati optimal
controller. Then $\adv_t^{\star}(u_t^{\pi};x_t^{\pi},\matw)$ is well-approximated
by an \emph{approximate advantage function} $\advhat_{t;h}( M
;\matw_{t+h}) $ which (a) \emph{does not} depend on the state, and (b)
depends on only a small horizon $h$ of future disturbances, and (c) is
a pure quadratic function of $M$, and thereby amenable to fast
(logarithmic) rates for online learning. Let $h$ be a
horizon/lookahead parameter. Defining
\begin{align}
\qstar_{\infty;h}(w_{1:h+1}) &\ldef
                               \sum_{i=1}^{h+1}\Sigma_{\infty}^{-1}B^{\trn}(\Aclinf^{\trn})^{i-1}P_{\infty}w_{i},
                               \intertext{the approximate advantage
                               function is}
\advhat_{t;h}( M ;\matw_{t+h}) &\ldef \|q^{M}(\wpast) - \qstar_{\infty;h}(w_{t:t+h}) \|_{\Sigma_{\infty}}^2. \label{eq:advhat}
\end{align}
The following theorem facilitates the use of the approximate advantages.
\newcommand{\Cadv}{C_{\mathrm{adv}}}
\begin{restatable}{theorem}{mainregretdecomp}
\label{thm:main_reg_decomp} Let $\pi$ be any policy of the form $\pi_t(x;\matw) = -\Kinf x -  q^{M_t}(\wpast)$, where $M_t =
M_t(\matw) \in \Mst$. Then, by choosing the horizon parameter as $h=2
(1-\gammast)^{-1}\log(\kappast^{2}\betast^{2}\Psist\Gammast^{2}T^{2})$,
we have
\begin{align*}
\sum_{t=1}^T \left| \advstar_t(u^{\pi}_t;x^{\pi}_t,\matw)  -  \advhat_{t;h}( M_t ;\matw_{t+h}) \right| \le \Cadv,
\end{align*}
where $\Cadv=\bigoht\prn*{
  \betast^{11}\Psist^{19}\Gammast^{11}\kappapi^{8}(1-\gammapi)^{-4}\log^{2}T
  }.$
\end{restatable}
The proof of this theorem constitutes a primary technical contribution
of our paper, and is proven in
\pref{sec:thm_main_reg_decomp}. Briefly, the idea behind the result is
that optimal policy $\pistar$ itself satisfies
$\pistar_t(x;\matw)\approx{}-\Kinf{}x-q^{\star}_{\infty;h}(w_{t:t+h})$
whenever $h$ is sufficiently large and
$t\leq{}T-\bigoh_{\star}(\log{}T)$, combined with the fact that
$\advstar_{t}$ has a simple quadratic structure. This characterization for is why it is essential to consider
advantages with respect to the \emph{optimal} policy $\pistar$, and
why our \dap{}s use the controller $\Kinf$ as opposed to an arbitrary
stabilizing controller as in \cite{agarwal2019online}. 
\subsection{Online learning with delays\label{ssec:main_alg_description}}
An immediate consequence of \pref{thm:main_reg_decomp} is that for any algorithm
(in particular,
\mainalg)
which selects $\pi_t(x;\matw)=-\Kinf{}x-q^{M_t}(\matw_{t-1})$, we have
\iftoggle{icml}
{
  $\cost(\pi;\matw) - \inf_{M \in \cM}\cost(\pi^{(M)}) \le \sum_{t=1}^T \advhat_{t;h}( M_t ;\matw_{t+h}) - \inf_{M \in \cM} \sum_{t=1}^T\advhat_{t;h}( M ;\matw_{t+h}) + \bigoh\prn*{\Cadv}$
}
{
\begin{align}
  \cost(\pi;\matw) - \inf_{M \in \cM_0}\cost(\pi^{(M)}) \le \sum_{t=1}^T \advhat_{t;h}( M_t ;\matw_{t+h}) - \inf_{M \in \cM_0} \sum_{t=1}^T\advhat_{t;h}( M ;\matw_{t+h}) + 2\Cadv.\label{eq:approx_regret}
\end{align}
}
This is simply an online convex optimization problem with
$M_1,\ldots,M_T$ as iterates---the only catch is that the ``loss''
at time $t$, $\advhat_{t;h}( M_t ;\matw_{t+h})$, can only be evaluated after
observing $w_{t:h}$, which will not be revealed to the learner until
after round $t+h$. This is therefore an instance of online learning with \emph{delays}, namely, the loss function suffered at time $t$ is only available at times $t + h + 1$ (note that $w_{t}$ is revealed at time $t+1$).
To reduce the problem of minimizing regret on the approximate
advantages in \pref{eq:approx_regret} to classical online learning
without delays, we use a simple black-box reduction.

Consider a generic online convex optimization setting where, at each
time $t$, the learner proposes an iterate $z_t$, then suffers cost
$f_t(z_t)$ and observes $f_t$ (or some function of it). Suppose we have an algorithm for
this non-delayed setting that guarantees that for every sequence, $ \sum_{t=1}^{T}f_t(z_t) -
    \inf_{z\in\cC}\sum_{t=1}^{T}f_t(z_t) \leq{} R(T)$, where $R$ is
    increasing in $T$. Now consider the same setting with delay $h$,
    and let  $\tau(t)=(t-1)\mod(h+1)+1 \in\brk*{h+1}$.  We use the
    following strategy: Make $h+1$ copies of the based algorithm. At round $t$, observe $z_t$, predict $z_t$ using the
output of instance $\tau(t)$, then update instance $\tau(t+1)$ using
the loss $f_{t-h}(z_{t-h})$ (which is now available).
\begin{restatable}[cf. \citet{joulani2013online}]{lemma}{delayreduction}
  \label{lem:delay_reduction}
The generic delayed online learning reduction has
  regret at most
    \begin{align*}
    \sum_{t=1}^{T}f_t(z_t) -
    \inf_{z\in\cC}\sum_{t=1}^{T}f_t(z) \leq{} (h+1)R(T/(h+1)),
  \end{align*}
  where $R(T)$ is the regret of the base instance.
\end{restatable}
\pref{lem:delay_reduction} shows that minimizing the regret in
\pref{eq:approx_regret} is as easy as minimizing regret in the
non-delayed setting, up to a factor of $h=\bigoht(\log{}T)$. For completeness, we provide a proof \pref{app:supporting_online_learning}. All that
remains is to specify the base algorithm for the reduction. 

\subsection{Exp-concave online learning}
\label{sec:exp_concave}
We have reduced the problem of obtaining logarithmic regret for online
control to obtaining logarithmic regret for online learning with
approximate advantages of the form in \pref{eq:approx_regret}. A
sufficient condition to obtain fast rates in online learning is strong
convexity of the loss \cite{hazan2016introduction}, but while the advantages
$\advhat_{t;h}(M;\matw_{t+h})$ are strongly convex with respect to
$q^{M}(\matw)$, they are not strongly convex with respect to the
parameter $M$. Itself. Fortunately, logarithmic regret can also be
achieved for loss functions that satisfy a weaker condition called exp-concavity \citep{hazan2007logarithmic,PLG}.
\begin{definition}\label{def:exp_concave} A function $f: \cC \to \R$
  is $\alpha$-exp-concave if $\nabla^2 f(z) \succeq \alpha (\nabla f(z))(\nabla f(z))^\top$ for all $z \in \cC$.
\end{definition}
Intuitively, an exp-concave function $f$ exhibits strong
curvature along the directions of its gradient, which are precisely the
directions along which $f$ is sensitive to change. This property
holds for linear regression-type losses, as the following
standard lemma (\pref{app:supporting_online_learning}) shows.
\begin{restatable}{lemma}{expconcavequadratic}
  \label{lem:exp_concave_quadratic}
  Let $A\in\bbR^{d_1\times{}d_2}$, and consider the function $f(z) = \nrm*{Az-b}_{\Sigma}^{2},$
  where $\Sigma\succeq 0$. If we restrict to $z\in\bbR^{d_2}$ for which $f(z)\leq{}R$,
  then $f$ is $(2R)^{-1}$-exp-concave.
\end{restatable}
Observe that the approximate advantage functions $\advhat_{t;h}(M;\matw_{t+h})$ are
indeed have the form $f(z) = \|A z -  b\|_{\Sigma}^2$ (viewing the
map $M\mapsto{}q^{M}(\matw_{t-1})$ as a linear operator), and thus
satisfy exp-concavity for appropriate $\alpha>0$. To take advantage
of this property we use online
Newton step (\ons, \pref{alg:ons}), a second-order online convex
optimization algorithm which guarantees logarithmic regret for
exp-concave losses.
\begin{lemma}[\citet{hazan2016introduction}]
  \label{lem:ons}
  Suppose that $\sup_{z,z'\in\cC}\nrm*{z-z'}\leq{}D$,
  $\sup_{z\in\cC}\nrm*{\grad{}f_t(z)}\leq{}G$, and that each loss $f_k$ is
  $\alpha$-exp-concave. Then by setting
  $\eta=2\max\crl*{4GD,\alpha^{-1}}$ and $\veps=\eta^{2}/D$, the online
  Newton step algorithm guarantees
  \begin{align*}
    \sum_{k=1}^{T}f_k(z_k) - \inf_{z\in\cC}\sum_{k=1}^{T}f_k(z) \leq{} 5(\alpha^{-1}+GD)\cdot{}d\log{}T.
    \end{align*}
  \end{lemma}

\paragraph{Putting everything together.}
With the regret decomposition in
terms of approximate advantages (\pref{thm:main_reg_decomp}) and the
blackbox-reduction for online learning with delays
(\pref{lem:delay_reduction}), the design and analysis of \mainalg
(\pref{alg:main}) is rather simple. In view of
\pref{lem:disturbance_sufficient}, we initialize the set $\cM_0$
sufficiently large to compete with the appropriate state-feedback
controllers (\pref{line:Minit}). Using \pref{thm:main_reg_decomp}, our
goal is to obtain a regret bound for the approximate advantages in
\pref{eq:approx_regret}. In view of the delayed online
learning reduction \pref{lem:delay_reduction}, we initialize $h+1$
base online learners
(\pref{line:ons_init}). Since the approximate advantages $\advhat_t$
are pure quadratics, we use online Newton step for the base learner,
which ensures logarithmic regret via \pref{lem:ons}.

\subsection{Sharpening the regret bound}
\label{sec:sharpen}
With online Newton step as the base algorithm, \mainalg has regret 
$\bigohs(\dimx\dimu\sqrt{\dimx\wedge\dimu}\log^{3}T)$. The
$\dimx\dimu$ factor comes from the hard dependence on $\dim(\cC)$ in
the \ons{} regret bound (\pref{lem:ons}), and the
$\sqrt{\dimx\wedge\dimu}$ factor is an upper bound on the Frobenius
norm for each $M\in\Mst$. We can obtain improved dimension dependence by replacing
\ons{} with a vector-valued variant of the classical
Vovk-Azoury-Warmuth algorithm (\vvaw), described in \pref{alg:ons}. The \vaw algorithm goes beyond the generic exp-concave
online learning
setting and exploits the quadratic structure of the
approximate advantages.  
\pref{thm:vvaw} in \pref{app:vaw} shows that its regret depends only
logarithmically on the Frobenius norm of the parameter vectors, so it avoids the
$\sqrt{\dimx\wedge\dimu}$ factor paid by \ons{} (up to a log term). This leads to a final regret bound of
$\bigohs(\dimx\dimu\log^{3}T)$ for \mainalg. The runtime for both
algorithms is identical.

The calculation for the final regret bound,
including dependence on problem parameters and specification for the
learning rate parameters in \pref{alg:main}, is carried out in \pref{app:main_algo_proof}.


%% file: algorithms.tex

\begin{algorithm}[t]
  \setstretch{1.1}
  \begin{algorithmic}[1]
    \State \textbf{parameters}:
    \Statex{}~~~~Horizon $h$, \dap{}
    length $m$, radius $R$, decay parameter $\gamma$.
    \Statex{}~~~~Online Newton
    parameters $\eta_{\mathrm{ons}}$, $\veps_{\mathrm{ons}}$, or
    Vovk-Azoury-Warmuth parameter $\veps_{\mathrm{vaw}}$.
    \iftoggle{icml}
    {
      \State \textbf{initialize}  base learners $\ons^{(1)},\dots,\ons^{(h+1)}\gets\ons(\epsilon_{\mathrm{ons}},\eta_{\mathrm{ons}},\Mst)$ (\pref{alg:ons}), where define the set
    $\Mst \gets \cM(m,R,\gamma)$ (\pref{eqn:dapset})\label{line:Minit} \label{line:ons_init}.
    }
    {
          \State \textbf{initialize}:
      \Statex{}~~~~Let
    $\Mst \leftarrow \cM(m,R,\gamma)$ (\pref{eqn:dapset}).\label{line:Minit} 
    %
      \Statex{}~~~~Option I: Instantiate base learners
      $\alg^{(1)},\dots,\alg^{(h+1)}$ as
      $\ons(\veps_{\mathrm{ons}},\eta_{\mathrm{ons}},\Mst)$ (\pref{alg:ons})\label{line:ons_init}.
            \Statex{}~~~~Option II: Instantiate base learners $\alg^{(1)},\dots,\alg^{(h+1)}$ as $\vaw(\veps_{\mathrm{vaw}},\Mst,\Siginf)$.
    }
    \State Let $\tau_t= (t-1)\mod(h+1)+1\in\brk*{h+1}$.
    \For{$t=1,\ldots,T$:}
    \Statex{}~~~~~\algcomment{Predict using base learner $\tau_t$.}
        \State{}Let $M_t$ denote the $k_t$-th iterate produced \iftoggle{icml}{\Statex{}~~~~~~~~~~}{}by $\alg^{(\tau_t)}$ where $k_t \leftarrow \floor{t/(h+1)}$.
    \State{}Play $u_{t}=-K_{\infty}x_t - q^{M_t}(\wpast)$, where $q^{M_t}$ is as \iftoggle{icml}{\Statex{}~~~~~~~~~~~~}{}in \pref{def:dap}.
    \iftoggle{icml}
    {
    , and observe $x_{t+1}$ and $w_{t}$.
    }
    {
    \State{}Observe $x_{t+1}$ and $w_{t}$.
    }
    
    \Statex{}~~~~~\algcomment{Update base learner $\tau_{t+1}$.}
    \iftoggle{icml}
    {
      \State \textbf{if} $t \geq{}h+1$, \textbf{then} feed $\advhat_{t-h;h}(\cdot;\matw_{t})$  to $\alg^{(\tau_{t}+1)}$,\Statex{}~~~~~~~~~ where $\advhat$ is as in \pref{eq:advhat}. 
    }
    {
      \If{$t \geq{}h+1$}
      \Statex{}~~~~~~~~~\algcomment{Approximate advantage from Eq.\;\eqref{eq:advhat}.}
      \State{}Update $\alg^{(\tau_{t}+1)}$ with loss function $\advhat_{t-h;h}(M ;\matw_{t}) = \|q^{M}(\matw_{t-h-1}) - \qstar_{\infty;h}(w_{t-h:t}) \|_{\Sigma_{\infty}}^2$.
      \EndIf
    }
    \EndFor
  \end{algorithmic}
  \caption{Riccatitron}
  \label{alg:main}
\end{algorithm}


%% file: section_analysis.tex
\newcommand{\Ckinf}{C_{\Kinf}}
We now prove the key ``approximate advantage'' theorem
(\pref{thm:main_reg_decomp}) used in the analysis of \mainalg. The
roadmap for the proof is as follows:
\begin{enumerate}
\item In \pref{ssec:true_opt}, we show that the unconstrained optimal
  policy takes the form
  $\pistar_t(x;\matw) = - K_t x_t - \qstar_{t}(\matw)$, where
  $\qstar_t(\matw)$ depends on all future disturbances, and where
  $K_t$ is the finite-horizon solution to the Riccati recursion
  (\pref{def:dp}).
\item 
  Next, \pref{ssec:remove_state} presents an intermediate version of
  the approximate advantage theorem for policies of the form
  $\pilearn_t(x;\matw) = -K_t x_t - q^{M_t}(\wpast)$. Because any such policy has the same state dependence as the optimal policy $\pistar$, we
  are able to show that $\advstar_t(u^{\pilearn}_t;x^{\pilearn}_t,\matw)$ has
  \emph{no state dependence}. Moreover, the linear structure of the
  dynamics and quadratic structure of the losses ensures that
  $\advstar_t(u^{\pilearn}_t;x^{\pilearn}_t,\matw)$ is a
  quadratic of the form $\| q^{M_t}(\wpast) - \qstar_t(\wr[t])\|_{\Sigma_t}^2$, where
  $\Sigma_t$ is a finite-horizon approximation to $\Sigma_{\infty}$,
  and $\qstar_t(\wr[t])$ is the bias vector of the optimal controller.
  \item Finally (\pref{ssec:infinite_horizon}), we use stability of
  the Riccati recursion to show that
  $\qstar_t(\matw)$ can be replaced with a term that depends only on
  $\matw_{t+h}$, up to a small error. Similarly, we show that $\Sigma_t$ can be
  replaced by $\Siginf$ and $K_t$ by $\Kinf$.
\end{enumerate}
This argument implies that a slightly modified analogue of
\mainalg which replaces infinite-horizon quantities
($\Kinf$, $\Siginf$,...) with finite-horizon analogues from the
Riccati recursion attains a similar regret. We state \mainalg with the
  infinite horizon analogues to simplify presentation, as well as implementation.
\subsection{A closed form for the true optimal policy \label{ssec:true_opt}}
  Our first result characterizes the optimal unconstrained optimal
  controller $\pistar$ given full knowledge of the disturbance sequence
  $\matw$, as well as the corresponding value function. To begin, we
  introduce a variant of the \emph{Riccati recursion}.
    \begin{definition}[Riccati recursion]\label{def:dp}
     Define $P_{T+1}=0$ and $\bias_{T+1}=0$ and consider the recursion:
     \iftoggle{icml}
     {
      with state-cost matrix $\Pt=  \Rx + A^{\trn}\Ptpl{}A -
     A^{\trn}\Ptpl{}B\Sigt^{-1}B^{\trn}\Ptpl{}A$, input-cost $\Sigt=\Ru+B^{\trn}\Ptpl{}B$, controller $
      \Kt = \Sigt^{-1}B^{\trn}\Ptpl{}A$,  and  $\ct(\wr) = (A-B\Kt)^{\trn}(\Ptpl{}w_{t} + \ctpl(\wr[t+1])).$
     }
     {
        \begin{align*}
    &\Pt=  \Rx + A^{\trn}\Ptpl{}A -
      A^{\trn}\Ptpl{}B\Sigt^{-1}B^{\trn}\Ptpl{}A,\\
      &\Sigt=\Ru+B^{\trn}\Ptpl{}B,\\
      &\Kt = \Sigt^{-1}B^{\trn}\Ptpl{}A,\\
      &\ct(\wr) = (A-B\Kt)^{\trn}(\Ptpl{}w_{t} + \ctpl(\wr[t+1])).
    \end{align*}
     }We also define corresponding closed loop matrices via $\Aclt = A-B\Kt$.
  \end{definition}
  When $\En\brk*{w_t} = 0$ for all times $t$, the optimal
  controller is the state feedback law $\pi_t(x) = - K_t x_t$, and
  $K_t\to\Kinf$ as $t\to{}-\infty$. The following theorem shows that
  for arbitrary disturbances the optimal controller applies the
  same state feedback law, but with an extra bias term that depends on
  the disturbance sequence.
  \begin{theorem}
    \label{thm:pistar_form}
    The optimal controller is given by $\pistar_t(x,\matw) = -\Kt{}x -
    \qstar_t(\wr[t])$, where
    \begin{equation}
      \label{eq:qstar}
      \qstar_t(\wr[t]) =
      \sum_{i=t}^{T-1}\Sigma_t^{-1}B^{\trn}\prn*{\prod_{j=t+1}^{i}\Aclt[j]^{\trn}}P_{i+1}w_i
      .
    \end{equation}
    Moreover, for
    each time $t$ we have 
    \begin{equation}
      \label{eq:vstar}
      \Vstar_t(x;\matw) = \nrm*{x}_{P_t}^{2} + 2\tri*{x,\bias_t(\wr[t])}
      + f_t(\wr[t]),\end{equation}
    where $f_t$ is a function that does not depend on the state $x$.
  \end{theorem}
  \pref{thm:pistar_form} is a special case of a more general result,
  \pref{thm:pistar_form_general}, proven in \pref{app:technical}.

\subsection{Removing the state\label{ssec:remove_state}}
We now use the characterization of $\pistar$ to show that the
advantages $\advstar_t(u_t^{\pilearn};x_t^{\pilearn},\matw)$ have a
particularly simple structure when we consider policies of the form
$\pilearn_t(x;\matw) = -K_t x_t- q_t(\wpast)$, where $q_t(\matw)$ is an
arbitrary function of $\matw$. For such policies, $\advstar_t$is a
quadratic function which does not depend explicitly on the state.
  \begin{lemma}\label{lem:advstar}
    Consider a policy $\pilearn_t(x)$
    of the form $\pilearn_t(x;\matw) = -K_t x_t -q_t(\matw)$. Then, for all $x$,
  \begin{align*}
    \advstar_{t}(\pilearn_t(x;\matw); x, \matw) = \|q_t(\matw) - \qstar_t(\wr[t]) \|_{\Sigma_t}^2.
  \end{align*}
  \end{lemma}
  \begin{proof}
  Since $\Qstar_t(x,\cdot;\matw)$ is a strongly convex quadratic, and
  since $\pistar_t(x;\matw) = \argmin_{u\in\bbR^{\dimu}}\Qstar_{t}(x,
  u;\matw)$,  first-order optimality conditions imply that for
  any $u$,
  \iftoggle{icml}
  {
    $
  \advstar_t(x,u;\matw) =\Qstar_t(x,u;\matw) - \Qstar_t(x,\pistar_t(x;\matw);\matw) = \| u - \pistar_t(x;\matw)\|_{\grad^{2}_u\Qstar_t(x,u;\matw)}^2$.
  } 
  {
  \begin{align*}
  \advstar_t(u;x,\matw) =\Qstar_t(x,u;\matw) - \Qstar_t(x,\pistar_t(x;\matw);\matw) = \| u - \pistar_t(x;\matw)\|_{\grad^{2}_u\Qstar_t(x,u;\matw)}^2.
  \end{align*}
  }
  A direct computation based on \pref{eq:vstar} reveals that $
  \grad^{2}_u\Qstar_t(x,u;\matw) = R + B^{\trn}P_{t+1}B = \Sigma_t$,
  so that $ \advstar_t(u;x,\matw) = \| u -
  \pistar_t(x;\matw)\|_{\Sigma_t}^2$. Finally, since
  $\pistar_t(x;\matw) = -K_t x - \qstar_t(\wr[t])$, we have that if $u = \pilearn_t(x;\matw) = -K_t x_t - q_t(\matw)$, then the states in the expression $u - \pistar_t(x;\matw)$ cancel, leaving $u - \pistar_t(x;\matw) = -(q_t(\matw) -  \qstar_t(\wr[t]))$. 
\end{proof}
\subsection{Truncating the future and passing to infinite horizon \label{ssec:infinite_horizon}}
The next lemma---proven in \pref{app:analysis}---shows that we can truncate $\qstar_t(\wr[t])$ to only depend on
disturbances at most $h$ steps in the future.
  \begin{restatable}{lemma}{qstartruncate}
  \label{lem:qstar_truncate}
  For any $h\in\brk*{T}$ define a truncated version of $\qstar_t$ as follows:
  \begin{equation}
    \label{eq:qstar_truncate}
    \qstar_{t;t+h}(w_{t:t+h}) = \sum_{i=t}^{(t+h)\wedge{}T-1}\Sigma_t^{-1}B^{\trn}\prn*{\prod_{j=t+1}^{i}\Aclt[j]^{\trn}}P_{i+1}w_i.
  \end{equation}
  Then for any $t$ such that $t+h<T-\bigoht(\betast\Psist^{2}\Gammast)$, setting $\gammab = \frac{1}{2}(1+\gammast)<1$, we have the bound $\nrm*{\qstar_{t:t+h}(w_{t:t+h})
    - \qstar_{t}(w_{t:T})}
  \leq{} \kappast^{2}\betast^{2}\Psist\Gammast^{2}(T-h)\gammab^{h}$, which is geometrically decreasing in $h$. 
\end{restatable}
Going forward we use that both $\qstar_{t}$ and
$\qstar_{t:t+h}$ have norm at most
$\betast\Psist\Gammast\kappast(1-\gammast)^{-1}\rdef{}D_{\qstar}$ (\pref{lem:qt_bound}). 
As an immediate corollary of \pref{lem:qstar_truncate}, we approximate the advantages using finite lookahead.
\begin{restatable}{lemma}{truncateqregret}
  \label{lem:truncate_regret}
Consider a policy $\pilearn_t(x;\matw) = -K_t
x_t -q_t(\matw)$, and suppose that $\nrm*{q_t}\leq{}\Dq$, where $\Dq\geq{}D_{\qstar}$. If we choose $h=2
(1-\gammast)^{-1}\log(\kappast^{2}\betast^{2}\Psist\Gammast^{2}T^{2})$, we
are guaranteed that 
\begin{align*}
  \sum_{t=1}^{T}\abs*{\advstar_{t}(u_t^{\pilearn};x_{t}^{\pilearn},\matw)
  - \| q_t(\matw) -  \qstar_{t;t+h}(w_{t:t+h})\|_{\Sigma_t}^2}
  \leq \Ctrunc,
  \end{align*}
  where $\Ctrunc\leq{}\bigoht(\Dq^{2}\betast\Psist^{4}\Gammast^{2}(1-\gammast)^{-1}\log{}T)$.
\end{restatable}
At this point, we have established an analogue of
\pref{thm:main_reg_decomp}, except that we are still using
state-action controllers $\Kt{}$ rather than $\Kinf{}$, and
the approximate advantages in \pref{lem:truncate_regret} are using the
finite-horizon counterparts of $\Siginf$ and $q_{\infty;h}$. The
following lemmas shows that we can pass to these
infinite-horizon quantities by paying a small approximation cost.
\begin{restatable}{lemma}{kinftokt}
\label{lem:kinf_to_kt}
Let policies $\pi_t(x;\matw) = -\Kinf x - q_t(\matw)$ and
  $\pilearn_t(x;\matw) = -K_t x - q_t(\matw)$ be given, where $q_t$ is
  arbitrary but satisfies $\nrm*{q_t}\leq{}\Dq$ for some
  $D_q\geq{}1$. Then
  \begin{align*}
    \abs*{\cost(\pilearn,\matw)-\cost(\pi,\matw)} \leq{} \Ckinf,
  \end{align*}
  where $\Ckinf\leq{} \bigoht\prn*{
    \kappast^{4}\betast^{6}\Psist^{13}\Gammast^{6}(1-\gammast)^{-2}\Dq^{2}\cdot{}\log{}(D_qT)
      }$.
\end{restatable} 
\begin{restatable}{lemma}{qttoqinf}
  \label{lem:qt_to_qinf}
Let $(q_t)_{t=1}^{T}$ be an arbitrary sequence with
$\nrm*{q_t}\leq{}\Dq$ for some $\Dq\geq{}\Dqst$. Then it holds that
\begin{align*}
\left|\sum_{t=1}^{T}\| q_t -  \qstar_{t;t+h}(w_{t:t+h})\|_{\Sigma_t}^2
  - \| q_t -
                 \qstar_{\infty;h}(w_{t:t+h})\|_{\Sigma_{\infty}}^2\right| 
      \le{}  \underbrace{\bigoht\prn*{
    \Dq^{2}\cdot{}\betast^{4}\Psist^{7}\Gammast^{4}\kappast^{2}(1-\gammast)^{-1}h\log(\Dq{}T)
  }}_{\rdef C_{q_{\infty},\Sigma_{\infty}}}.
\end{align*}
\end{restatable}
Combining these results immediately yields the proof of
\pref{thm:main_reg_decomp}; details are given in \pref{app:analysis}.


%% file: section_general_class.tex
In view of \pref{sec:algorithms} and \pref{sec:analysis}, it should be clear the disturbance-action
parameterization \pref{eq:dap} used in \mainalg serves only to facilitate the use of
tools from online convex optimization. By appealing to tools from the
more general online learning framework, we can derive rates for
generic, potentially nonlinear benchmark policy classes.

Suppose we wish to compete with a benchmark class $\Pi$ where each
$\pi\in\Pi$ takes the form $\pi(x;\matw)=-\Kinf{}x -
q_{t}^{\pi}(\matw_{t-1})$, and suppose that the learner's policy takes
the form $\pialg(x;\matw)=-\Kinf{}x-q_{t}^{\mathrm{alg}}(\matw_{t-1})$. The development so far implies
that as long as $\nrm*{q_t^{\pi}}$ is uniformly bounded for all
$\pi\in\Pi$, we have
\begin{equation}
\Reg(\pialg;\Pi,\matw) =
\sum_{t=1}^{T}\nrm*{q_{t}^{\mathrm{alg}}(\matw_{t-1})-q^{\star}_{\infty;h}(\matw_{t+h})}_{\Siginf}^{2}
-
\inf_{\pi\in\Pi}\sum_{t=1}^{T}\nrm*{q_{t}^{\pi}(\matw_{t-1})-q^{\star}_{\infty;h}(\matw_{t+h})}_{\Siginf}^{2}
+C_{\mathrm{err}},\label{eq:general}
\end{equation}
where $C_{\mathrm{err}}$ is a logarithmic approximation error term. We
can appeal to the generic delayed online learning reduction once more to
reduce this problem to online supervised learning. Consider the following protocol for online learning: At time $t$:
Receive $w_{t-1}$, predict $\wh{q}_t\in\bbR^{\dimu}$, then receive
$q^{\star}_t\in\bbR^{\dimu}$. If we have an algorithm for this
protocol that ensures
\newcommand{\Rosl}{R_{\mathrm{OSL}}}
\begin{equation}
\sum_{t=1}^{T}\nrm*{\wh{q}_t-q^{\star}_t}_{\Siginf}^{2}
-
\inf_{\pi\in\Pi}\sum_{t=1}^{T}\nrm*{q_{t}^{\pi}(\matw_{t-1})-q^{\star}_t}_{\Siginf}^{2}\leq{}\Rosl(T),
\end{equation}
for every sequence, then the delayed online learning reduction enjoys
regret $(h+1)\Rosl(T/(h+1))$ for the delayed problem
\pref{eq:general}. For example, since the loss
$\wh{q}\mapsto\nrm*{\wh{q}-\qstar}_{\Siginf}^2$ is exp-concave, we can
apply Vovk's aggregating algorithm
\citep{vovk1990aggregating,vovk1995game} to guarantee
\[
  \Reg(\pialg;\Pi,\matw) \leq{} \bigoh_{\star}\prn*{
    \log\abs*{\Pi}\cdot\log{}T
    }
  \]
  for any finite class of  policies. More generally, one can derive fast rates for arbitrary nonparametric
classes of benchmark policies via the offset Rademacher complexity-based minimax
bounds given in \cite{rakhlin2014nonparametric}.


%% file: conclusion.tex
We have presented the first efficient algorithm with logarithmic regret for online linear control with arbitrary adversarial disturbance sequences. Our result highlights the power of online learning with advantages, and we are hopeful that this framework will find broader use. Numerous questions naturally arise for future work: Does our framework extend to more general loss functions, or to more general classes of dynamical systems in control and reinforcement learning? Can our results be extended to handle partial observed dynamical systems? Can we obtain $\sqrt{T}$-regret for adversarial disturbances in unknown systems, as is possible in the stochastic regime?

\subsection*{Acknowledgements} DF acknowledges the support of TRIPODS award \#1740751. MS is generously supported by an Open Philanthropy AI Fellowship. We thank Ali Jadbabaie for helpful discussions.


%% file: app_notation.tex

This appendix is organized as follows. \pref{app:generalization}
describes extensions of \mainalg. \pref{app:other_regret_benchmarks}
demonstrates that $\mainalg$ competes with richer benchmark class that includes arbitrary linear controllers with internal state; \pref{app:tracking_gen} extends the algorithm to consider ``tracking costs'' studied by \cite{abbasi2014tracking}; \pref{app:varying_costs} explains how the algorithm can accomodate time-varying quadratic costs, provided that they are known to the learner in advance.

\pref{app:olws_limits} explains challenges associated with using online learning with stationary costs (\olws) to attain logarithmic regret in our setting. This appendix also provides a unifying (albeit informal) treatment of existing \olwsx approaches. In addition, \pref{app:mdpe} highlights the differences between \mainalg{} and \mdpex \citep{even2009online}, a variant of \olwsx which is superficially similar to our approach.

The remaining three appendices are dedicated to proving our main
results. \pref{app:technical} collects some basic structural results
for linear quadratic control which we use throughout the appendix, and
\pref{app:additional_proofs} describes a variant of the performance
difference lemma \citep{kakade2003sample} which is used in our analysis. \pref{app:algorithm} provides the missing proofs from Section~\ref{sec:algorithms}.  Importantly, \pref{app:main_algo_proof} proves Theorem~\ref{thm:main_algo}, and \pref{app:vaw} establishes a regret guarantee for the vector-valued \vaw{} algorithm (\pref{alg:vaw}). Finally, \pref{app:analysis} supplies the missing proofs from Section~\ref{sec:analysis}, culminating in the proof of \pref{thm:main_reg_decomp}.

Notation used throughout the main paper and appendix is collected in \pref{table:notation}.

\clearpage
\begin{center}
	\begin{longtable}{| l | l |}
	\hline
	\textbf{Notation} & \textbf{Definition}  \\
	\hline
	$A,B$ & system matrices (dynamics, Eq.~\pref{eq:dynamics})\\
	$w_t$/$w_{1:t}$/$\matw$ & disturbance at time $t$/from $1,\dots,t$/from $1,2,\dots$\\
	$\pi$ & control policy\\
	$K$, $\pi^K$ & static feedback controller, induced policy \\
	$M,\pi^{(M)}$ & \dap{} controller (\pref{def:dap}), induced policy\\
	$\pialg$ & policy selected by the learner \\
	\hline
	$\dimx,\dimu$ & state/input dimension\\
	$\matx_t^{\pi}(\matw),\matu_t^{\pi}(\matw)$ & state/input at time $t$ under policy $\pi$ and disturbances $\matw$\\
	$\matx_t^{K}(\matw),\matu_t^{K}(\matw)$ & state/input at time $t$ under policy $\pi^K$\\
	$\loss,\Rx,\Ru$ & cost function $\loss(x,u) = x^\top \Rx x + u^\top \Ru$ \\
	$J_T(\pi;\matw)$ & cost of policy $\pi$, $\sum_{t=1}^T\loss(\matx_t^{\pi}(\matw),\matu_t^{\pi}(\matw))$\\
	\hline
	$\Reg(\pialg;\Pi,\matw)$ & regret with benchmark $\Pi$: $\cost(\pialg;\matw) - \inf_{\pi \in \Pi}\cost(\pi;\matw)$\\
	$\Kclass_0 $ & benchmark class of strongly stable feedback controllers \\
	$\KReg(\pi^{\Alg};\matw)$ & regret benchmark with compartor $\Kclass_0$: $J_T(\pi^{\Alg};\matw) - \inf_{K \in \Kclass_0}J_T(\pi^K;\matw)$.\\
	$\cM_0$  & benchmark class of \dap{} controllers (parameterized as  $\cM(m,R,\gamma)$ in \pref{def:dap})\\
	$\dapReg(\pialg;\matw)$ & regret relative to $\cM_0$ (Eq.~\pref{eq:dap_reg})\\
	\hline
	$\Pinf$ & solution to the $\dare$ (Eq.~\pref{eq:dare})\\
	$\Kinf$ & optimal infinite horizon LQR controller\\
	$\Sigmainf$ & optimal infinite horizon  LQR covariance\\
	$\Aclinf$ & closed loop system $A- B\Kinf$ under optimal infinite-horizon controller\\
	\hline
	$\kappast,\gammast$ & strong stability parameters for $\Aclinf$ (see \pref{def:ss}) \\
	$\kappa_0,\gamma_0$ & strong stability parameters for $\Kclass_0$ (see \pref{def:ss}) \\
	$\Psist$& $\max\crl*{1,\nrm*{A}_{\op},\nrm*{B}_{\op},\nrm*{\Rx}_{\op},\nrm*{\Ru}_{\op}}$\\
	$\betast$ & $\max\crl*{1,\eigmin^{-1}(\Ru),\eigmin^{-1}(\Rx)}$\\
	$\Gammast$ & $\max\crl*{1,\nrm*{\Pinf}_{\op}}$\\
	\hline
	$\pistar$ & unconstrained optimal policy (\pref{def:optimal})\\
	$\Qstar$, $\Vstar$ &Q-function and value function under $\pistar$ (\pref{def:optimal})\\
	$\advstar$ & advantage, defined as $\Qstar_t(x,u;\matw) - \Qstar_t(x,\pistar_t(x;\matw);\matw)$\\
	\hline 
	$q_t$ & generic bias-predicting term (e.g., $\pi(x,\matw)=-\Kinf-q_t(\matw)$)\\
	$q_t^M$ &  bias-predicting term in \dap{} (\pref{def:dap})\\
	$\qstar_{\infty;h}(w_{1:h+1})$ & truncated approximation for $\qstar_t$-function in $\pistar$ (defined above Eq.~\pref{eq:advhat}).\\
	$\advhat_{t;h}( M ;\matw_{t+h})$ & approximate advantage (Eq.~\pref{eq:advhat})\\
	\hline
	$\cC$ & generic constraint set for online optimization\\
	\ons, $\ons(\veps,\eta,\cC)$ & \makecell[l]{Online Newton
          Step (\pref{alg:ons}) with learning rate $\eta$,
          regularization parameter $\veps$}\\
	\vaw, $\vaw(\veps,\cC,\Sigma)$ & Vovk-Azoury-Warmuth
        (\pref{alg:vaw}) with regularization parameter $\veps$, cost $\Sigma$\\
	\hline
	$P_t,\Sigt,\Kt,\Aclt$ & finite-horizon optimal analogues of $\Pinf,\Sigmainf,\Kinf,\Aclinf$ (\pref{def:dp_time_varying})\\
	$\qstar_t(w_{t:T})$ &  bias function for $\pistar$
        ($\pistar_t(x;w_{t:T}) = -\Kt x -\qstar_t(w_{t:T})$; Eq.\pref{eq:qstar})\\
	\hline
	$\Hclinf$,$\Lclinf$ &  matrices that
  witness strong stability of $\Aclinf$ (\pref{def:ss})\\
 	$\Lclt$ & $\Hclinf^{-1}\Aclt{}\Hclinf$ (used to show strong stability of $\Aclt$)\\
$\Aclt[i\to{}t]$& $\Aclt[t]\Aclt[t-1]\cdots\Aclt[i+1]$, with 
convention  $\Aclt[t\to{}t]=I$\\
$\gammab$ & $\frac{1}{2}(1+\gammast)<1$\\
$\Deltast$ & $4\cdot{}\betast{}\Psist^{2}\Gammast\log(2\Psist\Gammast\kappast(1-\gammast)^{-1})$ (``decay time'')\\
 $\Tstab$ & $T-\Deltast$\\
	\hline
	\hline
	\caption{Summary of notation \label{table:notation}}
	\end{longtable}
	\end{center}

%% file: section_generalizations.tex
\newcommand{\PtT}{P_{t;T}}
\newcommand{\PtplT}{P_{t;T}}
\newcommand{\SigtT}{\Sigma_{t;T}}
\newcommand{\KtT}{K_{t;T}}
\newcommand{\AclT}[1][t]{A_{cl,#1;T}}

\newcommand{\ctT}{q_{t;T}}
\newcommand{\ctplT}{q_{t+1;T}}

\subsection{Alternative regret benchmarks \label{app:other_regret_benchmarks}}

Throughout the main paper we only considered benchmarks based on
linear feedback controllers of the form $u_t = -K x_t$, where $K$ is
strongly stabilizing. We now show that \dapx controllers (and
consequently \mainalg) can be used to compete with a more general
class of linear controllers with \emph{internal state}. We use an
argument from \citet{simchowitz2020improper}. Consider $m_Q \in \bbN$, and  controller of the form
\begin{align}
\pi_t ^{[Q]} (x;\matw) =  - \Kinf x + \sum_{i=0}^{m_Q-1} Q^{[i]}x_{t-i}^{\Kinf}(\matw), \label{eq:dfc}
\end{align}
where $x_t^{\Kinf}(\matw)$ denotes the state that would arise at time $t$ if the linear selected
the optimal linear control law $u_s^{\Kinf}(\matw) = -\Kinf
x_s^{\Kinf}(\matw)$ for all $s<t$. We note that this counterfactual
can be computed from $w_{1:t-1}$. By \citet{simchowitz2020improper}, to show that the $\dap$
parameterization competes with controllers with internal state, it
suffices to show that the parameterization competes with controllers
of the form \pref{eq:dfc}. To see this is indeed the case, observe that since $\Kinf$
stabilizes the system $(A,B)$, we have
\begin{align*}
x_{s}^{\Kinf}(\matw) = \sum_{i=0}^{h} (A-B\Kinf)^{i}w_{s-i-1} \pm e^{-\Omega(h (1-\gamma_{\infty}))},
\end{align*}
where we use $\pm$ in an informal, vector-valued sense.
Hence, we can render
\begin{align}
\pi ^{[Q]}_t(x;\matw) =  - \Kinf x + \sum_{i=0}^{m_Q-1} \sum_{j=0}^{h} Q^{[i]} (A-B\Kinf)^{j} w_{t-(i+j+1)} \pm e^{-\Omega(h (1-\gamma_{\infty}))}.
\end{align}
It follows that setting $m = m_Q+h$, we can approximate the above behavior with an $m$-length controller of the form $M^{[k]} = \sum_{i=0}^{m-1} \sum_{j=0}^{h} Q^{[i]} (A-B\Kinf)^{j}  \I_{i+j + 1 = k}$ captures the policy \pref{eq:dfc}. 

To formalize the extension, one must also verify that for some
reasonable $R,m$, the sequence $M$ above lies in the set 
\begin{align*}\cM(m,R,\gamma) := \{M = (M^{[i]})_{i=1}^m: \|M^{[i]}\|_{\op}\le
  R\gamma^{i-1}\},
 \end{align*}
that is, the sequence enjoys geometric decay with parameter
$\gamma$. This decay can be achieved in numerous ways, e.g. taking
$\gamma = 1/m$ and inflating $R$ by a factor of $e$. At the extreme,
one can show that the constraint set $\cM(m,R,\gamma)$ can be replaced with a set which \emph{does not} enforce geometric decay,
\begin{align*}
\widetilde{\cM}(m,R) := \{M = (M^{[i]})_{i=1}^m: \|M^{[i]}\|_{\op}\le
  R,~ \forall i\},
 \end{align*}
 at the expense of suffering a larger polynomial in $\log{}T$ in the final regret bound. We omit the details in the interest of brevity.

\subsection{Tracking moving targets \label{app:tracking_gen}}

\newcommand{\matwbar}{\bar{\matw}}
\newcommand{\ar}[1][t]{a_{#1:T}}
\newcommand{\br}[1][t]{b_{#1:T}}
\newcommand{\wbarr}[1][t]{\bar{w}_{#1:T}}
\newcommand{\wbar}{\bar{w}}

We next show that \mainalg generalizes to a setting with moving
targets (or, ``adversarial targets'') previously studied without
adversarial noise by \cite{abbasi2014tracking}. In this setting, for a
sequence of \emph{targets} $a_{1:T}$, $b_{1:T}$, the learner's loss at
time $t$ is given by
\begin{align*}
\ell_t(x,u) = \ell(x - a_t,u - b_t) = \nrm*{x-a_t}_{\Rx}^2 + \nrm*{x-b_t}_{\Rx}^2.
\end{align*}
Let us adopt the shorthand $\bar{w}_t = (w_t,a_t,b_t)$, and  $\matwbar
= (w_{1:T},a_{1:T},b_{1:T})$. \pref{thm:pistar_form_general}---proven
in \pref{app:technical}---shows that if we define
\begin{equation}
      \label{eq:qstar_general_two}
       \qstar_t(\wbarr[t]) = \Sigma_t^{-1}\prn*{ - R_u b_t + B^{\trn}\sum_{i=t}^{T-1}\prn*{\prod_{j=t+1}^{i}\Aclt[j]^{\trn}}P_{i+1}w_i  + B^{\trn}\sum_{i=t+1}^{T-1}\prn*{\prod_{j=t+1}^{i-1}\Aclt[j]^{\trn}}( K_i^\top \Ru b_i - \Rx a_i)},
     \end{equation}
     where $K_t$, $\Sigma_t$, and so on are given by the Riccati
     Recursion (\pref{def:dp}), then the optimal unconstrained controller is given
    by $\pistar_t(x;\matwbar) = -\Kt{}x - \qstar_t(\wbarr)$. Retracing
    our steps from the special case without moving targets, we have
    the following generalization of \pref{lem:advstar}.
     \begin{lemma}[Advantages for Moving
       Targets]\label{lem:advstar_movement} Consider a policy
       $\pilearn_t(x)$ of the form $\pilearn_{t}(x) = -K_{t} x_t -
       q_t(\matwbar)$. For all $x$, we have
  \begin{align*}
    \advstar_{t}(\pilearn_{t}(x);x,\matwbar) = \|q_{t}(\matwbar) - \qstar_{t}(\wbarr[t]) \|_{\Sigma_{t}}^2,
  \end{align*}
  where $ \qstar_{t}(\wbarr[t])$ is given by \pref{eq:qstar_general_two}.
  \end{lemma}
  To extend \mainalg to this setting, we define truncated versions of $\qstar$ and $\advstar$ analoguous to to the without-moving-targets case (Eq.~\pref{eq:advhat}). With $\matwbar_t \ldef ((w_t,a_t,b_t),(w_{t-1},a_{t-1},b_{t-1}),\dots)$), we define
  \begin{align}
&\qstar_{\infty;h,\text{move}}(\wbar_{1:h+1}) \ldef \Sigmainf^{-1}\prn*{ - R_u b_t + B^{\trn}\sum_{i=1}^{h+1}(\Aclinf^\trn)^{i-1}\Pinf w_i  + B^{\trn}\sum_{i=2}^{h+1}(\Aclinf^\trn)^{i-2}( \Kinf^\top \Ru b_i - \Rx a_i)} \nonumber,\\
&\advhat_{t;h,\text{move}}( M ;\matwbar_{t+h}) \ldef \|q_t^{M}(\wpast) - \qstar_{\infty;h}(\wbar_{t:t+h+t}) \|_{\Sigma_{\infty}}^2 \label{eq:advhat_move}.
\end{align}
We simply run \mainalg with the new approximate advantage functions $\advhat_{t;h,\text{move}}( M ;\matwbar_{t+h})$ from 
\pref{eq:advhat_move} replacing their without-moving-targets variants from \pref{eq:advhat}. Logarithmic regret follows by the same arguments.
\subsection{Varying quadratic costs \label{app:varying_costs}}
\newcommand{\Rxt}{R_{t;x}}
\newcommand{\Rut}{R_{t;u}}
As a final generalization, we show that our analysis generalizes to
time varying quadratic losses $\loss_t(x,u) = x^\top \Rxt x + u^\top
\Rut u$, provided the cost matrices $\Rxt$ and $\Rut$ are known to the
learner ahead of time. Of course, this extension generalizes further
to ``tracking'' losses of the form $\loss_t(x,u) = \|x -
a_t\|_{\Rxt}^2  + \|u - u_t\|_{\Rut}^2$ as in the previous section.

To perform this generalization, we consider the following variant of
the Riccati Recursion.
\begin{definition}[Time-varying Riccati Recursion]\label{def:dp_time_varying}
     Define $P_{T+1}=0$ and $\bias_{T+1}=0$ and consider the recursion:
        \begin{align*}
    &P_t=  \Rxt + A^{\trn}\PtplT{}A -
      A^{\trn}\Ptpl{}B\Sigma_t^{-1}B^{\trn}\PtplT{}A,\\
      &\SigtT=\Rut+B^{\trn}P_{t+1}B,\\
      &\KtT = \Sigma_t^{-1}B^{\trn}P_{t+1}A,\\
      &\bias_t(\wr) = (A-BK_t)^{\trn}(P_{t+1}w_{t} + \bias_{t+1}(\wr[t+1])).
    \end{align*}
  \end{definition}
  We similarly define closed-loop matrices $\Aclt = (A-BK_t)$. The form of the optimal policy generalizes in the obvious way 
   \begin{align*}
     \pistar_{t}(x;\matw) = -K_{t}x - \qstar_{t}(\wr[t]),\quad\text{and}\quad
     \qstar_{t}(\wr[t]) =
      \sum_{i=t}^{T-1}\Sigma_{t}^{-1}B^{\trn}\prn*{\prod_{j=t+1}^{i}\Aclt[j]^{\trn}}P_{i+1}w_i.
    \end{align*}
    Advantages take the form
     \begin{align*}
       \advstar_{t;T}(\pilearn_{t}(x);x,\matw) = \|q_{t}(\matw) - \qstar_{t}(\wr[t]) \|_{\Sigma_{t}}^2.
  	\end{align*}
  	Note that compared to the fixed-cost setting, we cannot
        leverage the existence of the ``steady-state'' matrix $\Pinf$ here. Nonetheless, we can still truncate the dependence on the
        future by using the vectors
        $\qstar_{t;T}(w_{t},\dots,w_{t+h},0,\dots,0)$ to create
        approximate advantages with finite lookahead, which can then be
        used within the \mainalg scheme.


%% file: app_stationary.tex
\newcommand{\ftil}{\widetilde{f}}
\newcommand{\xalg}{x^{\Alg}}
\newcommand{\ualg}{u^{\Alg}}

This section highlights the technical challenges encountered when attempting to apply \olws{} to attain logarithmic regret in online control with adversarial disturbances. In addition to highlighting the advantages (no pun intended) of our \olwax approach, this appendix may serve as an informal tutorial of prior approaches for online control problems. The section is organized as follows:
\begin{enumerate}
	\item \pref{app:olws_overview} gives an intuitive overview of the \olws{} paradigm, explaining that the regret encountered by the learner incurs a `stationarization' cost reflecting the mismatch between the costs induced by the learner's actual visited trajectory and the trajectories considered by the stationary costs.
	\item \pref{app:advantage over OLws} explains that the standard approach for bounding stationarization cost is in terms of a ``movement cost'', which measures the cumulative differences between succesive policies $\pi_t$: informally, $\sum_{t=1}^T\|\pi_t - \pi_{t-1}\|$. Pointing forward to \pref{app:exp_concave_perils}, we explain that this is the major barrier to obtaining logarithmic regret in our setting. In contrast, stationarization/movement costs \emph{do not} arise in our analysis of \olwa{}, leading to our main result.
	\item \pref{app:olws_for_online_control} reviews in greater detail how the \olwsx paradigm has been applied to online control with adversarial disturbances. \pref{app:oco_memory} covers the OCO-with-memory framework due to \cite{anava2015online}. \pref{app:movement_costs_online_control} shows how \cite{agarwal2019online} instantiate this framework for online control with the \dap{} parametrization, detailing the (approximate) stationary cost functions $f_{t;h}(M)$ that arise and the corresponding movement cost in the regret analysis. Examining these loss functions, \pref{app:stat_cost_exp_conc_nut_strong_convex} shows that they are exp-concave but not strongly convex.
	\item  \pref{app:exp_concave_perils} demonstrates that---in the OCO-with-memory framework---the movement cost for sequences of exp-concave but non-strongly convex functions can scale as $\Omega(\sqrt{T})$ in the worst case. This implies that any analysis which uses a black-box reduction to OCO-with-memory with bounded movement cost cannot guarantee rates faster that $\bigoh(\sqrt{T})$.\footnote{Note that this argument does not preclude the possibility that \olws{} algorithms can attain logarithmic regret; rather, it demonstrates the an analysis which passes through movement costs for arbitrary exp-concave stationary costs is insufficient.}
\end{enumerate}
Finally, \pref{app:mdpe} compares our \olwax approach to the \mdpex algorithm proposed by \cite{even2009online}. These two algorithms are superficially similiar, because they both consider  control-theoretic advantages. Despite these similarities, we note that \mdpex is still an instance of \olws, and therefore succumbs to the limitations described above. In addition, we highlight that the analysis of \mdpex is ill-suited to settings with adversarial dynamics, such as the one considered in this work.

\subsection{Overview of OLwS}
\label{app:olws_overview}
In this section, we give an overview of the online learning with stationary costs ($\olwsx$) framework for online control and discuss some challenges associated with using it to attain logarithmic regret for online linear control. In \olws, one defines the stationary costs
\begin{align}
\lambda_t(\pi;\matw) := \ell(x_t^{\pi}(\matw),u_t^{\pi}(\matw)),\label{eq:lookback}
\end{align}
which is the cost suffered that would be suffered at time $t$ had the policy $\pi$ had used at all previous rounds \citep{yadkori2013online,anava2015online,agarwal2019online,simchowitz2020improper}.\footnote{Many works also consider ``steady state'' costs obtained by taking $t \to \infty$ for a given policy \citep{even2009online,abbasi2014tracking,cohen2018online,agarwal2019logarithmic}, but this formulation is ill-posed in our setting due to the adversarial dynamics.} By construction, $\lambda_t(\pi;\matw)$ does not depend on the state of the system. Moreover, if $\pi$ is an executable policy (i.e., $\pi_t(x;\matw)$ depends only on $x$ and $w_{1:t-1}$), then $\lambda_t(\pi;\matw)$ can be determined exactly at time $t$. At each round $t$, $\olws$ selects a policy $\pi^{(t)}$ to minimize regret on the sequence $\lambda_t(\pi;\matw)$, and follows $u_t = \pi_t^{(t)}(x;\matw)$. The total regret is decomposed as
\begin{align}
\Reg(\pialg;\matw,\Pi) =  \Big(\underbrace{\sum_{t=1}^T \ell(x_t^{\Alg}, u_t^{\Alg})-  \lambda_t(\pi^{(t)};\matw) }_{\text{(stationarization cost)}} \Big) +   \Big(\underbrace{\sum_{t=1}^T \lambda_t(\pi^{(t)};\matw)-  \inf_{\pi \in \Pi} \sum_{t=1}^T  \lambda_t(\pi;\matw)}_{(\lambda\text{-regret})} \Big). \label{eq:lookback_reg}
\end{align}
\subsection{Avoiding stationarization cost: Our advantage over OLwS\label{app:advantage over OLws}} \olwsx optimizes stationary costs $\lambda_t(\pi)$, which correspond to the loss suffered by the learner at round $t$ if policy $\pi$ had been played for every time up to $t$. To relate the stationary costs to the learner's cost, the \olwsx proposes the bounding the following movement cost:
\begin{align}
\text{movement cost } := \sum_{t}\|\pi^{(t)} - \pi^{(t-1)}\|, \quad \text{(informal)}, \label{eq:movement_cost}
\end{align}
To our knowledge, all known applications of \olwsx bound the stationarization cost via the movement cost \pref{eq:movement_cost}. When the movement costs are small, the learner's state at time $t$, $x^\Alg_t$, is similar to the states that would be obtained by selecting $\pi^{(t)}$ at all time $s < t$, namely $x^{\pi}_t$. \pref{app:olws_limits} explains how the cost \pref{eq:movement_cost} arises  in more detail.
While standard online learning algorithms ensure $\sqrt{T}$-movement cost, online gradient descent (\ogd) has the property that if $\lambda_t$ is strongly convex (in a suitable parametrization), the movement cost is $\log T$ \citep{anava2015online}. Since \ogdx also ensures logarithmic regret on the $\lambda_t$-sequence, the algorithm ensures logarithmic regret overall.

The natural stationary costs that arise in our problem are exp-concave \citep{hazan2011newtron}, a property that is stronger than convexity but weaker than strong convexity. Exp-concave functions $\lambda_t$ are strongly convex in the local geometry induced by $(\nabla \lambda_t)(\nabla \lambda_t)^\top $, but not necessarily in other directions. This is sufficient for logarithmic regret, but as we explain in \pref{app:exp_concave_perils}, known methods cannot leverage this property to ensure logarithmic bounds on the relevant movement cost. 
Herein lies the advantage of $\olwa$: by considering the \emph{future} costs of an action (by way of the advantage-proxy $\advhat$) rather than the stationary costs, we avoid the technical challenge of bounding the movement cost in the elusive exp-concave regime.

\subsection{Applying OLwS to online control \label{app:olws_for_online_control}}
\subsubsection{Policy regret and online convex optimization with memory \label{app:oco_memory}}
A useful instantiation of the \olws{} paradigm is the policy-regret setting introduced by  \cite{arora2012online}, which considers stationary costs with finite memory.  This work considers online learning with loss functions $f_t(z_t,\dots,,z_{t-h})$, and defines policy regret for the iterate sequence $\{z_t\}_{t=1}^{T}$ as $\sum_{t}f(z_{t},\dots,z_{t-h}) - \inf_z \sum_{t}f_t(z,\dots,z)$. Algorithms for this setting work with a  \emph{unary} loss $\widetilde{f}_t(z) = z \mapsto f_t(z,\dots,z)$,which can be viewed as a special case of the stationary cost $\lambda_t$ defined above where $z \in \cC$ encodes a policy and $f_t(z,\dots,z)$ is the loss suffered if $z$ had been selected throughout the game. \cite{arora2012online} take this approach in an expert setting and \cite{anava2015online} consider a setting where $\ftil$ is an arbitrary convex loss, which they call \emph{Online Convex Optimization with Memory}. In this setting, the stationarization cost arises via the decomosition
\begin{align}
\text{(policy regret)} =  \Big(\underbrace{\sum_{t=1}^T f_t(z_t,\dots,z_{t-h}) - \ftil_t(z_t) }_{\text{(stationarization cost)}} \Big) +   \Big(\underbrace{\sum_{t=1}^T \ftil_t(z_t) -  \inf_{z \in \cC} \sum_{t=1}^T  \ftil_t(z)}_{(\lambda\text{-regret})} \Big). \label{eq:lookback_reg_policy}
\end{align}
When $f_t$ are Lipschitz in all arguments, \cite{anava2015online} bound the stationarization cost in terms of movement cost for the iterates. They show
\begin{align*}
\sum_{t=1}^T f_t(z_t,\dots,z_{t-h}) - \ftil_t(z_t)  \le \sum_{t=1}^T |f_t(z_t,\dots,z_{t-h}) - \ftil_t(z_t)| \le \bigoh\prn*{h^2 L}\sum_{t=1}^T \|z_t - z_{t-1}\|,
\end{align*}
where $L$ is an appropriate Lipschitz constant. Note that this inequality formalizes \pref{eq:movement_cost} for this setting.

\cite{anava2015online} demonstrated that many popular online convex optimization algorithms naturally produce slow-moving iterates, leading to policy regret bounds in \pref{eq:lookback_reg_policy}. In particular, they show that applying online gradient descent on the unary losses leads to $\poly(h) \cdot \sqrt{T}$-policy regret when $\widetilde{f}_t$ are convex and Lipschitz, and $\frac{\poly(h)}{\alpha} \cdot \log T$-policy regret when $\widetilde{f}_t$ are $\alpha$-strongly convex. Notably, \cite{anava2015online} do not show that logarithmic regret is attainable for the more general family of exp-concave losses, which are more natural for the setting in this paper.

\subsubsection{ OCO with memory for online control\label{app:movement_costs_online_control}}
Now, following \cite{agarwal2019online}, we apply OCO with memory to the linear control setting using the \dapx parametrization (\pref{def:dap}), where $\pi^{(t)}$ is given by $\pi^{(M_t)}$, for a matrix $M \in \cM_0 = \cM(m,R,\gamma)$ selected at time $t$. We will specialize the \olwsx{} decomposition \pref{eq:lookback_reg_policy} and explain how to bound each term.

\cite{agarwal2019online} show stationary costs $\lambda_t$ in \pref{eq:lookback} can be approximated up to
arbitrarily accuracy by functions $\ftil_{t;h}(M)$, which depend only on the most recent $m+h$ disturbances $w_{t-(m+h):t}$, and where $h = \mathrm{poly}(\log T, \frac{1}{1-\gamma})$. They also show that the suffered loss $\loss(\xalg_t,\ualg_t)$ can be approximate via $f_{t;h}(M_{t-h:t})$, which also depends on recent disturbances, and which specializes to $\widetilde{f}_{t:h}(M)$ when $M_{t-h:t} = (M,\dots,M)$. Precisely, for $M \in \cM$, define the inputs.
\begin{align*}
u_{s}(M;\matw) = \sum_{i=1}^m M^{[i]} w_{s-i}
\end{align*}
Then the functions $f_{t:h}(M_{t:t-h})$ and $\widetilde{f}_{t:h}(M)$ take the form
\begin{align}
f_{t;h}(M_{t-h:t}) &:= \loss\prn*{\alpha_t(\matw) + \sum_{i=1}^m \Psi_i u_{t-i}(M_{t-i};\matw), u_{t}(M_t;\matw)- \Kinf \sum_{i=1}^m \Psi_i u_{t-i}(M_{t-i};\matw) }, \nonumber
\\
\ftil_{t;h}(M) &:=\loss\prn*{\alpha_t(\matw) + \sum_{i=1}^m \Psi_i u_{t-i}(M;\matw), u_{t}(M;\matw)- \Kinf \sum_{i=1}^m \Psi_i u_{t-i}(M;\matw) },\label{eq:stat_control_cost}
\end{align}
where $\alpha_t(\matw)$ is a function of $w_{1:t}$ and $A,B$ but \emph{not} of the learner's inputs, and $\Psi_i = (A-B\Kinf)^{i-1}B$. With this parameterization, the regret decomposition for OCO with memory takes the form
\begin{align}
\dapReg(\pialg;\matw) =  \Big(\underbrace{\sum_{t=1}^T f_{t;h}(M_{t-h:t})  - \ftil_{t;h}(M_t) }_{\text{(stationarization cost)}} \Big) +
\Big(\underbrace{\sum_{t=1}^T \ftil_{t;h}(M_t)-  \inf_{\pi \in \Mst} \sum_{t=1}^T  \ftil_{t;h}(M)}_{(\lambda\text{-regret})} \Big) + \bigoht\prn*{1}. \label{eq:lookback_reg_dap}
\end{align}
In this setting $f_{t;h}(\cdot)$ is Lipschitz so---following arguments from \cite{anava2015online}---we have
\begin{align*}
\text{(stationarization cost)} \le h^2 \poly(m,R,(1-\gamma)^{-1})\cdot \sum_{t=1}^T \|M_{t} - M_{t-1}\|_{\mathrm{F}},
\end{align*}
where the right-hand side is a movement cost for the iterates (formalizing \pref{eq:movement_cost}), and where we recall $h$ is the memory horizon, $m,R,\gamma$ are the parameters defining the set of \dap{} controllers $\cM_0$, and $\|M_{t} - M_{t-1}\|_{\mathrm{F}} = \sqrt{\sum_{i \ge 0}\|M_{t}^{[i]} - M_{t-1}^{[i]}\|_{\mathrm{F}}^2}$ (which induces the standard Euclidean geometry for online gradient descent). 

In general, the bound on the movement cost will depend on the choice of regret minimization algorithm. Many natural algorithms ensure bounds on the movement cost which are on the same order as their bounds on regret. For example, exponential weights and online gradient descent ensure $\sqrt{T}$-bounds \citep{even2009online,yu2009markov,anava2015online}, and for strongly convex losses, FTL and online gradient descent ensure logarithmic movement \citep{abbasi2014tracking,anava2015online}.

\subsubsection{The stationary costs for DAP are exp-concave but not strongly convex\label{app:stat_cost_exp_conc_nut_strong_convex}}
For the \dapx parametrization, the functions that naturally arise are exp-concave, but not necessarily strongly convex. To see this, consider the loss $\loss(x,u) = \|x\|^2 + \|u\|^2$. We obseve that the stationary costs considered by \cite{agarwal2019online}, made explicit in \pref{eq:stat_control_cost}, are the sum of two quadratic functions of the form considered in \pref{lem:exp_concave_quadratic}, and are thus exp-concave.\footnote{The sum of two $\alpha$-exp-concave functions is $\frac{\alpha}{2}$-exp-concave. For a proof, observe that $(\nabla (f+g))^\top (\nabla (f+g))^\top \preceq 2(\nabla f)(\nabla f)^\top  + 2(\nabla g)(\nabla g)^\top$. Hence, if $f$ and $g$ are $\alpha$-exp concave, we have $\nabla^2 (f+g) \succeq \alpha (\nabla f)(\nabla f)^\top  + \alpha(\nabla g)(\nabla g)^{\top}\succeq \frac{\alpha}{2}(\nabla (f+g))^\top (\nabla (f+g))^\top$.} However, $\ftil_{t;h}(M)$ is \emph{not} strongly convex in general.  For example, if the noise sequence is \emph{constant}, say $w_t = w_{t-1} = \dots = w_1$, then $\nabla u_t(M;\matw)$ is identical for all $t$ and thus $\nabla^2 \widetilde{f}_{t;h}(M)$ is a rank-one matrix.

\subsection{Movement costs in general exp-concave online learning\label{app:exp_concave_perils}}
In this section, we explain the challenge of achieving low-movement cost in the OCO-with-memory with framework, which elucidates the broader challenge of relating stationary costs to regret in \olwsx. We give an informal sketch for an exp-concave OCO-with-memory setting in which the online Newton step algorithm (\pref{alg:ons}) fails to achieve logarithmic regret. Consider a simple class of functions with scalar domain and length-$1$ memory:
\begin{align*}
f_t(z_1,z_2) = (1 - (w_t z_1 + w_{t-1}z_2))^2, \quad \ftil_t(z) = f_t(z,z),
\end{align*}
where $(w_{t})$ are parameters chosen by the adversary.
We use the constraint set $z \in \cC:= [-1/5,1/5]$. Policy regret (paralleling \pref{eq:lookback_reg_dap}) is given by
\begin{align*}
\text{(policy regret)} &= \sum_{t=1}^T f_t(z_t,z_{t-1}) - \inf_{z} \sum_{t=1}^T \ftil_t(z) \\
&= \underbrace{\left(\sum_{t=1}^T f_t(z_t,z_{t-1}) - \ftil_t(z_t) \right)}_{\text{(stationarization cost)}}  + \underbrace{\left(\sum_{t=1}^T \ftil_t(z_t)   - \inf_{z} \sum_{t=1}^T \ftil_t(z) \right)}_{\text{(}\lambda-\text{regret)}}.
\end{align*}
We now construct a sequence of loss functions where the $\lambda$-regret for \ons{} is logarithmic, but where standard upper bounds on stationary cost can grow as $\Omega(\sqrt{T})$. Consider the sequence $w_t = (-1)^t + \frac{\mu}{2}$, where $\mu = 1/\sqrt{T}$. We see that $\ftil_t(z) = (1 - \mu z)^2$. We remark that this function is only $\mu^2 = 1/T$-strongly convex, so that the guarantees for strongly convex online gradient descent are vacuous \cite{hazan2016introduction}, necessitating the use of \ons.

Let us see what happens if we try to leverage exp-concavity. From \pref{lem:exp_concave_quadratic}, $\ftil_t(z)$ are $\frac{1}{4}$-exp-concave on the set $\cC$. Hence, if we run \ons{} (\pref{alg:ons}) with an appropriate learning rate, $\lambda$-regret scales logarithmically \citep{hazan2016introduction}:
\begin{align*}
\sum_{t=1}^T \ftil_t(z_t) - \min_{z \in \cC} \sum_{t=1}^T \ftil_t(z) \le \bigoh\prn*{\log T}.
\end{align*}
Let us now turn to the stationarization cost, $\sum_{t=1}^T f_t(z_t,z_{t-1}) - \ftil_{t}(z_t)$. The approach of \cite{anava2015online}, is to bound the per-step errors, $|f_t(z_t,z_{t-1}) - \ftil_{t}(z_t)|$. We can directly see that 
\begin{align*}
f_t(z_t,z_{t-1}) - \ftil_{t}(z_t) &= (1 - w_t z_t - w_{t-1}z_t)^2 - (1 - w_t z_t - w_{t-1}z_{t-1})^2\\
&= -2w_{t-1}(1 - w_t z_t)^\top (z_t - z_{t-1}) + w_{t-1}^2(z_{t-1}+z_{t})(z_{t}-z_{t-1})^2\\
&= (-2w_{t-1}(1 - w_t z_t)^\top + w_{t-1}^2(z_{t-1}+z_{t})) (z_t - z_{t-1}).
\end{align*}
For $\mu$ sufficiently small and $z \in \cC$, we can check that $|(-2w_{t-1}(1 - w_t z_t)^\top + w_{t-1}^2(z_{t-1}+z_{t}))| \ge \frac{1}{16}$, so that
\begin{align*}
|f_t(z_t,z_{t-1}) - \ftil_{t}(z_t)| \ge \frac{|z_t - z_{t-1}|}{16}.
\end{align*}
Thus, we have
\begin{align*}
\sum_{t=1}^T |f_t(z_t,z_{t-1}) - \ftil_{t}(z_t)| \ge \frac{1}{16}\cdot \text{(movement cost)}, \quad \text{ where }\quad\text{(movement cost)} = \frac{1}{6}\sum_{t=1}^T|z_{t}-z_{t-1}|.
\end{align*}
We now show that this movement cost is large. For simplicity, we keep our discussion informal to avoid navigating the projection step in \ons. Without projections, we have
\begin{align*}
|z_t - z_{t-1}| = \left|\epsilon + \sum_{s=1}^{t-1} \nabla^2 \ftil_s(z_s)\right|^{-1}(\nabla \ftil_{t-1}(z_{t-1})).
\end{align*}
Observe that for each $z \in \cC$, $\nabla^2 \ftil_s(z) = \mu^2 = 1/T$, so that  we have $\epsilon + \sum_{s=1}^{t-1} \nabla^2 \ftil_s(z_s) = (1+\epsilon)$. On the other hand, for $z \in \cC$, we can lower bound $|\nabla \ftil_{t-1}(z)| \ge \frac{\mu}{2}  = \frac{1}{2\sqrt{T}}$. Hence, 
\begin{align*}
\text{(movement cost)} = \frac{1}{16}\sum_{t=1}^T|z_{t}-z_{t-1}| \ge \frac{\sqrt{T}}{32(1+\epsilon)}.
\end{align*}
Here, we note that the standard implementation perscribes $\epsilon$ to be constant, giving us $\Omega(\sqrt{T})$ movement. Moreover, increasing $\epsilon$ will degrade the corresponding regret bound, preventing logarithmic combined regret.
Note that increasing $\epsilon$ to $1/T^{1/4}$ will partially mitigate the movement cost, but at the expense of increasing the regret on the $\ftil_t$ sequence.

\subsection{Comparison with MDP-E \label{app:mdpe}}
\mdpex \citep{even2009online} is an instantiation of \olwsx for MDPs with known non-adversarial dynamics and time varying adversarial losses $\loss_t$. In this setting the stationary costs $\lambda_t(\pi)$ represent the long-term costs of a policy $\pi$ under the loss $\loss_t$ (if one prefers, the loss can be treated as fixed, and $w_t$ can encode loss information). To achieve low regret on the $\lambda_t$-sequence, \mdpex maintains policy iterates $\{\pi^{(t)}_x\}$ for all states $x$, and selects its action according to the policy for the corresponding current state: \[u_t^{\Alg} \gets \pi^{(t)}_{x^{\mathrm{alg}}_t}(x_t^{\Alg}).\] The policy sequence $\pi^{(t)}_x$ is selected to minimize regret on a certain $Q$-function: $\lambda_{t,x}(\pi) : \pi \mapsto Q^{\pi}(x,\pi(x))$ (here, policies and Q-functions are regarded as stationary). Under the assumption that the dynamics under benchmark policies are also stationary, achieving low regret on each $\{\lambda_{t,x}\}$-sequence simultaneously for all $x$ ensures low $\lambda$-regret (in the sense of Eq.\eqref{eq:lookback_reg}) over the trajectory $x_t^\Alg$.\footnote{See the proof of \citet[Theorem 5]{even2009online}, which uses that the induced state distribution does not change with $t$.} 
As a consequence, \mdpex is ill-suited to settings with adversarially changing dynamics. Since \olwax considers Q-functions and advantages defined with respect to an \emph{fixed} policy $\pistar$, it does not require benchmark policies to have stationary dynamics (which is important, since our adversarial disturbance setting does not have stationary dynamics).

Moreover, like the stationary costs, the functions $\lambda_{t,x}(\pi)$ describe long-term performance under $\pi$, and still need to be related to the learner's realized trajectory, typically via a bound on the movement cost of the policies. As described earlier, the analysis of \olwax does not require bounding the movement cost.


%% file: appendix_technical.tex

\subsection{Structural results for LQR}
\label{app:lqr_structural}
In this section we provide a number of useful structural properties
for the optimal controller for linear dynamical systems with quadratic
costs and arbitrary bounded disturbances. Even though
the results in this section concern the optimal finite-horizon
controllers, we prove bounds on various regularity properties for the
controllers that depend only on control-theoretic parameters for the
optimal infinite-horizon controller in the noiseless setting, which is
an intrinsic parameter of the dynamical system. All proofs are
deferred to \pref{app:lqr_proofs}.

For the results in this section and the remainder of the appendix we use that $\Aclinf$ is
$(\kappast,\gammast)$-strongly stable.
  \begin{lemma} 
          \label{lem:strongly_stable}
  Let $\gammast=\nrm{I-\Pinf^{-1/2}\Rx\Pinf^{-1/2}}_{\op}^{1/2}$, and ${\kappast=\nrm{\Pinf^{1/2}}_{\op}\nrm{\Pinf^{-1/2}}_{\op}}$. 
  Then the closed loop system $\Aclinf$ is $(\kappast,\gammast)$-strongly stable.
\end{lemma}
\begin{dproof}[\pfref{lem:strongly_stable}]
  Recall \citep{bertsekas2005dynamic} that the infinite-horizon Lyapunov matrix $\Pinf$ satisfies the
equation
  \[
    \Aclinf^{\trn}\Pinf\Aclinf- \Pinf + \Rx = 0.
  \]
  Since $\Pinf\psdgt{}0$, if we set $H=\Pinf^{-1/2}$ and
  $L=\Pinf^{1/2}\Aclinf\Pinf^{-1/2}$,  we deduce from this expression that
    \[
      L^{\trn}L - I + \Pinf^{-1/2}\Rx\Pinf^{-1/2}=0,
    \]
    and in particular
    $\nrm*{L}_{\op}^{2}\leq{}\nrm{I-\Pinf^{-1/2}\Rx\Pinf^{-1/2}}_{\op}<1$.
  \end{dproof}
 \begin{lemma}
   \label{lem:stable_bound}
   Let $A$ be $(\kappa,\gamma)$-strongly stable. Then for any $i\geq{}0$,
   \[
     \nrm*{A^{i}}_{\op}\leq{}\kappa\gamma^{i}.
   \]
 \end{lemma}
   \begin{dproof}[\pfref{lem:stable_bound}]
    Let $A=HLH^{-1}$, where $H$ and $L$ witness the strong stability
    property. Then we have
    \[
      \nrm*{A^{i}}_{\op}\leq{}\kappa\nrm*{L^{i}}_{\op}\leq{}\kappa\gamma^{i}.
    \]
  \end{dproof}
  \paragraph{Additional notation.}
  For the remainder of the appendix we adopt the following
  notation.   We let $\Hclinf$ and $\Lclinf$ denote the matrices that
  witness strong stability of $\Aclinf$, so that $\Aclinf=\Hclinf\Lclinf\Hclinf^{-1}$ and we have
$\nrm*{\Hclinf}_{\op}\cdot\nrm{\Hclinf^{-1}}_{\op}\leq\kappast$ and
$\nrm{\Lclinf}_{\op}\leq\gammast<1$. We also define $\Lclt = \Hclinf^{-1}\Aclt{}\Hclinf$, where we recall
that $(\Aclt)_{t=1}^{T}$ denote the closed-loop dynamics arising from
the Riccati recursion. We define
$\Aclt[i\to{}t]=\Aclt[t]\Aclt[t-1]\cdots\Aclt[i+1]$, with the
convention that $\Aclt[t\to{}t]=I$. Finally, we define $\gammab=\frac{1}{2}(1+\gammast)$,
$\Deltast=4\cdot{}\betast{}\Psist^{2}\Gammast\log(2\Psist\Gammast\kappast(1-\gammast)^{-1})$,
and $\Tstab=T-\Deltast.$

  \subsubsection{Properties of the optimal policy}
  \label{app:optimal}
  Recall that \pref{thm:pistar_form} characterizes the optimal
  unconstrained policy given full knowledge of $\matw$. Rather than
  directly proving this theorem, we state and prove a more general
  version, \pref{thm:pistar_form_general}, which generalizes the characterization to the setting of
\pref{app:tracking_gen} in which losses include adversarially chosen
targets. The optimal policy for this setting is defined as follows.
\begin{restatable}[Optimal policy, Q-function, advantage]{definition}{optimalgeneral} \label{def:optimal_general}
Assume $a_T, b_T = 0$, and recall that $\matwbar = (w_{1:T},a_{1:T},b_{1:T})$.
  Define $\Qstar_T(x,u;\matwbar)=\ls(x,u)$,  $\pistar_t(x;\matwbar) =
  \min_u \Qstar_T(x,u) = 0$, and $\Vstar_T(x;\matwbar)=\ls(x,0)$. For
  each $t<T$ define
  \begin{align*}
    &\Qstar_{t}(x,u;\matwbar) = \nrm*{x - a_t}_{Q}^{2} + \nrm*{u - b_t}_{R}^{2} +
      \Vstar_{t+1}(Ax+Bu+w_t; \matwbar),\\
      &\pistar_t(x;\matwbar)=\argmin_{u\in\bbR^{\dimu}}\Qstar_{t}(x,u;\matwbar),\\
    &\Vstar_{t}(x;\matwbar) = \min_{u\in\bbR^{d}}\Qstar_{t}(x,u;\matwbar) = \Qstar_{t}(x,\pistar_t(x;\matwbar);\matwbar).
  \end{align*}
  Finally, define $\adv_t^{\star}(u;x,\matwbar) := \Qstar_t(x,u;\matwbar) - \Qstar_t(x,\pistar_t(x;\matwbar);\matwbar)$.
  \end{restatable}

\begin{theorem}[Generalization of \pref{thm:pistar_form}]\label{thm:pistar_form_general}
    Set $\wbarr[t] = (\wr[t],\ar[t],\br[t])$. For each time $t$, we have $\Vstar_t(x;\matwbar) = \nrm*{x}_{P_t}^{2} + 2\tri*{x,\bias_t(\wbarr[t])}
      + f_t(\wbarr[t])$,
      where $f_t$ is a function that does not depend on the state $x$
      and $\bias_t$ is defined recursively with $\bias_{T+1}=0$ and 
      \[
        \bias_t(\bar{w}_{t:T}) =
        (A-BK_t)^{\trn}(P_{t+1}w_t+\bias_{t+1}(\bar{w}_{t+1:T})) +
        K_t^{\trn}\Ru{}b_t - \Rx{}a_t.
      \]
      Moreover, if we define
    \begin{equation}
      \label{eq:qstar_general}
       \qstar_t(\wbarr[t]) = \Sigma_t^{-1}\prn*{ - R_u b_t + B^{\trn}\sum_{i=t}^{T-1}\prn*{\prod_{j=t+1}^{i}\Aclt[j]^{\trn}} P_{i+1}w_i  + B^{\trn}\sum_{i=t+1}^{T-1}\prn*{\prod_{j=t+1}^{i-1}\Aclt[j]^{\trn}}( K_i^\top \Ru b_i - \Rx a_i)},
    \end{equation}
    then the optimal controller is given
    by $\pistar_t(x;\matwbar) = -\Kt{}x - \qstar_t(\wbarr)$. 
 \end{theorem}

\begin{lemma}
\label{lem:closed_loop_bound}
For all $\tau_1\leq{}\tau_2$, we have
\[
\nrm*{\prod_{t=\tau_1}^{\tau_2}\Aclt^{\trn}}_{\op}\leq{}\sqrt{\frac{\nrm*{\Pinf}_{\op}}{\eigmin(\Rx)}}\leq{}\betast^{1/2}\Gammast^{1/2}.
\]
\end{lemma}

\begin{lemma}
  \label{lem:stable_sequence}
  Let $\Deltast=4\cdot{}\betast{}\Psist^{2}\Gammast\log(2\Psist\Gammast\kappast(1-\gammast)^{-1})=\bigoht(\betast\Psist^{2}\Gammast)$,
  and let $\gammabar=\frac{1}{2}(1+\gammast)$. Then it holds that
  \[
    \nrm*{\Lclt}_{\op}\leq{}\gammabar<1,\quad\forall{}t\leq{}\Tstab\ldef{}T-\Deltast.
  \]
\end{lemma}

\begin{lemma}
  \label{lem:closed_loop_refined}
  Let $\tau_1\leq{}\tau_2$ be fixed. Then we have
  \[
    \nrm*{\prod_{t=\tau_1}^{\tau_2}\Aclt[t]^{\trn}}_{\op}
    \leq{} \kappast\nrm*{\prod_{t=\tau_1}^{\tau_2}\Lclt[t]^{\trn}}_{\op}
    \leq{} \kappast^{2}\betast^{1/2}\Gammast^{1/2}\cdot{}\gammabar^{\tau_2\wedge\Tstab-\tau_1\wedge\Tstab}.
    \]
\end{lemma}

\begin{lemma}
  \label{lem:qt_bound}
  Let $\matw$ be any sequence with $\nrm*{w_t}\leq{}1$. Let
  $t\in\brk*{T}$ and $h\geq{}0$ be given. Then we have
  \begin{align}
    &\nrm*{\qstar_t(\wr[t])}\vee \nrm*{\qstar_{t;t+h}(w_{t:t+h})}
      \leq{}
          \bigoht\prn*{
      \betast^{5/2}\Psist^{3}\Gammast^{5/2}\kappast^2(1-\gammast)^{-1}
    }
\rdef\Dqst\label{eq:qstar_full_bound},
      \intertext{and}
    &\nrm*{\qstar_{\infty;h}(w_{t:t+h})} \leq{}\betast\Psist\Gammast\kappast(1-\gammast)^{-1}\rdef\Dqinf.
      \label{eq:qstar_inf_bound}
  \end{align}
\end{lemma}
\begin{lemma}
  \label{lem:state_expression}
  Let policies $\pi_t(x;\matw) = -\Kinf x - q_t(\matw)$ and
  $\pilearn_t(x;\matw) = -K_t x - q_t(\matw)$ be given, where $q_t$ is
  arbitrary. Then the states for both controllers are given by
  \[
  x^{\pi}_{t+1}(\matw)=\sum_{i=1}^{t}\Aclinf^{t-i}w_i -
  \sum_{i=1}^{t}\Aclinf^{t-i}Bq_i(\matw)
  \quad\text{and}\quad   x_{t+1}^{\pilearn}(\matw)=\sum_{i=1}^{t}\Aclt[i\to{}t]w_i - \sum_{i=1}^{t}\Aclt[i\to{}t]Bq_i(\matw).
\]
\end{lemma}

  \begin{restatable}{lemma}{atbound}
    \label{lem:at_bound}
    Let $\alpha\geq{}1$ be given. Define $\Delta =
    C\cdot{}\betast\Psist^{2}\Gammast\log(\kappast^{2}\Psist\Gammast(1-\gammast)^{-1}\cdot\alpha{}T^{3})$,
  where $C>0$ is a numerical constant. If $C$ is sufficiently large,
  then for every $t\leq{}T-\Delta\leq{}\Tstab$ we are guaranteed that
      \begin{equation}
    \label{eq:ki_to_kinf}
    \nrm*{\Kt-\Kinf}_{\op}\leq{} \frac{1}{\kappast^{2}\Psist\cdot(\alpha{}T^{3})},\quad\text{and}\quad     \nrm*{\Aclt[i\to{}t]-\Aclinf^{t-i}}_{\op} \leq{} \frac{1}{\alpha{}T^{2}}\quad\forall{}t\leq{}T-\Delta.
  \end{equation}
\end{restatable}

\begin{lemma}
  \label{lem:state_action_bound}
  Let policies $\pi_t(x;\matw) = -\Kinf x - q_t(\matw)$ and
  $\pilearn_t(x;\matw) = -K_t x - q_t(\matw)$ be given, where $q_t$ is
  arbitrary but satisfies $\nrm*{q_t}\leq{}\Dq$ for some
  $\Dq\geq{}1$. Then for all $t\in\brk{T}$, we have
  \[
    \nrm*{x^{\pi}_{}(\matw_t)}\leq{}2\kappast\Psist(1-\gammast)^{-1}\Dq,\quad\text{and}\quad
    \nrm*{u^{\pi}_{t+1}(\matw)}\leq{}  3\kappast\betast\Psist^{3}\Gammast(1-\gammast)^{-1}\Dq,
  \]
  as well as
  \[
      \nrm*{x_{t+1}^{\pilearn}(\matw)}
      \leq{}\bigoht\prn*{
        \kappast^{2}\betast^{3/2}\Psist^{3}\Gammast^{3/2}(1-\gammast)^{-1}\cdot\Dq
      },
    \]
    and
    \[
      \nrm*{u_{t+1}^{\pilearn}(\matw)}
        \leq{}\bigoht\prn*{
        \kappast^{2}\betast^{5/2}\Psist^{5}\Gammast^{5/2}(1-\gammast)^{-1}\cdot\Dq
      }.
    \]
  \end{lemma}
\subsubsection{Proofs from Appendix \ref*{app:optimal}}
\label{app:lqr_proofs}

\begin{dproof}[\pfref{thm:pistar_form_general}]
We first prove that the identity for the value function,
 \begin{align*}
      \Vstar_t(x;\wbarr[t]) = \nrm*{x}_{P_t}^{2} + 2\tri*{x,\bias_t(\wbarr[t])}
      + f_t(\wbarr[t]),
    \end{align*}
    holds by
induction. Observe that at time $T$ we indeed have
$\Vstar_T(x,w_T)=\nrm*{x}_{\Rx}^{2}=\nrm*{x}_{P_T}^{2}$, where we recall $a_T,b_T = 0$ by assumption. Now suppose,
that at time $t+1$ we have
\begin{align*}
\Vstar_{t+1}(x;\wbarr[t+1]) = \nrm*{x}_{P_{t+1}}^{2} + 2\tri*{x,\bias_{t+1}(\wbarr[t+1])}
    + f_{t+1}(\wbarr[t+1]).
  \end{align*}
  We prove that the same holds for time $t$ using the following lemma.
  \begin{restatable}{lemma}{lqronestep}
    \label{lem:lqr_bias_one_step}
    Let $P_1\psdgt{}0$, $c_1$, $a_0$, and $b_0$ be given and define
$V_1(x) = \nrm*{x}^{2}_{P_1} + 2\tri*{x,c_1}$ and 
\begin{equation}
V_0(x,w,a_0,b_0) = \nrm*{x - a_0}_{\Rx}^{2} 
+ \min_{u}\crl*{
\nrm*{u - b_0}^{2}_{\Ru} + V_1(Ax+Bu+w)
}.\label{eq:v0}
\end{equation}
Then we have
\begin{equation}
  \label{eq:one_step_v2}
    V_0(x,w,a_0,b_0) = \nrm*{x}^{2}_{P_0} + 2\tri*{x,c_0} + f(w,a_0,b_0,c_1),
\end{equation}
where 
\begin{align*}
  &P_0 = \Rx + A^{\trn}P_1A -
  A^{\trn}P_1B\Sigma_0^{-1}B^{\trn}P_1A,\\
  &\Sigma_0 = \Ru+B^{\trn}P_1B,\\
  &K_0 = \Sigma^{-1}_0B^{\trn}P_1A,\\
  &c_0= (A- BK_0)^\top(P_1w +c_1) +  K_0^\top \Ru b_0 - \Rx a_0.
\end{align*}
Furthermore, letting $u^{\star}$ denote the minimizer in \pref{eq:v0}, we
have
\begin{equation}
u^{\star} = -\Sigma_0^{-1}B^{\trn}(P_1(Ax+w) +c_1 - R_u b_0)=-K_0x -\Sigma_0^{-1}(B^{\trn}(P_1w +c_1) - R_u b_0).
\label{eq:ustar}
\end{equation}
\end{restatable}
\begin{dproof}[\pfref{lem:lqr_bias_one_step}]
Since the minimization problem in \pref{eq:v0} is strongly convex with
respect to $u$, we conclude from first-order conditions that
\begin{align*}
  B^{\trn}P_1(Ax+B\ustar+w) + \Ru(\ustar - b_0) + B^{\trn}c_1 = 0,
\end{align*}
Rearranging, 
\begin{align*}
\ustar = -( \Ru + B^\trn P_1 B)^{-1} (B^\trn P_1 A + B^\trn c_1 + P_1w-\Ru{}b_0) = -K_0x -\Sigma_0^{-1}(B^{\trn}(P_1w +c_1) - \Ru b_0),
\end{align*}
which proves \pref{eq:ustar}. Next, observe that for any $u$, we have
\begin{align*}
\nrm*{u - b_0}^{2}_{\Ru} + V_1(Ax+Bu+w) &= u^\top \Sigma_0 u + 2u^\top(B^{\trn}(P_1 Ax + P_1w +c_1) - \Ru b_0)\\
&\quad+ x^\top A^\top P_1 A x + 2 x^\top A^\top (P_1 w + c_1) + g(w,c_1,b_0),
\end{align*}
where $g(w,c_1,b_0)$ is a function of $w$, $c_1$, and $b_0$ but not $x$ or
$w$. Next, observe that for any $\Sigma\psdgt{}0$ and $v$, $\min_{u} u^\top \Sigma u + 2\langle v, u \rangle = - v^\top \Sigma^{-1} v$. Hence, 
\begin{align*}
&\min_{u} \nrm*{u - b_0}^{2}_{\Ru} + V_1(Ax+Bu+w)\\ &= -\|B^{\trn}(P_1 Ax + P_1w +c_1) - \Ru b_0\|_{\Sigma_0^{-1}}^2\\
&\quad+ x^\top A^\top P_1 A x + 2 x^\top A^\top (P_1 w + c_1) + g(w,c_1,b_0),\\
&= x^\top A^\top (P_1 - P_1 B \Sigma_0^{-1} B^{\trn} P_1) A x \\
&\quad-  2(B^\trn(P_1w +c_1) - \Ru b_0)^\top \Sigma_{0}^{-1} B^\trn P_1 A x  + 2(P_1 w + c_1)^\top A x + \widetilde{g}(w,c_1,b_0),
\end{align*}
for an appropriate function $\widetilde{g}$. We can further simplify
the part of this expression that is linear in $x$ to 
\begin{align*}
&-2(B^\trn(P_1w +c_1) - \Ru b_0)^\top \Sigma_{0}^{-1} B^\trn P_1 A x  + 2(P_1 w + c_1)^\top A x\\
&~~~~=-2(B^\trn(P_1w +c_1) - \Ru b_0)^\top K_0 x  + 2(P_1 w + c_1)^\top A x \\
&~~~~=  2(P_1w +c_1)^\top(A- BK_0) x + 2 b_0^\top \Ru K_0 x ,
\end{align*}
which yields
\begin{align*}
\min_{u} \crl*{\nrm*{u - b_0}^{2}_{\Ru} +  V_1(Ax+Bu+w)} &= x^\top A^\top (P_1 - P_1 B^\top \Sigma_0^{-1} B P_1) A x \\
 &\quad+  2(P_1w +c_1)^\top(A- BK_0) x  + 2 b_0^\top \Ru K_0 x + \widetilde{g}(w,c_1,b_0).
\end{align*}
Therefore,
\begin{align*}
V_0(x,w,a_0,b_0) = x^\top P_0 x - 2 a_0^\top \Rx x  +  2 b_0^\top \Ru K_0 x +  2(P_1w +c_1)^\top(A- BK_0) x + \widetilde{g}(w,c_0,b_0)  + \|a_0\|_{\Rx}^2.
\end{align*}
This yields the lemma with $c_0 = (A- BK_0)^\top(P_1w +c_1) +  K_0^\top \Ru b_0 - \Rx a_0$, and $f(w,a_0,b_0,c_1) = \widetilde{g}(w,c_1,b_0) + \|a_0\|_{\Rx}^2$. 
\end{dproof}

Applying \pref{lem:lqr_bias_one_step} with $P_1=P_{t+1}$ and
$c_1=\bias_{t+1}(\wbarr[t+1])$, and using the definition of $\Qstar_t$ from
\pref{def:optimal} we see that we indeed have
\begin{align*}
\Vstar_t(x;\wbarr[t]) = \nrm*{x}_{P_t}^{2} + 2\tri*{x,\bias_t(\wbarr[t])}
    + f_t(\wbarr[t]),
  \end{align*}
  and that
  \begin{align*}
    \pistar_t(x;\matw) = -K_tx - \Sigma_t^{-1}(B^{\trn}(P_{t+1}w_t+\bias_{t+1}(\wbarr[t])) - R_u b_t).
  \end{align*}
    Unfolding the recursion, we also see that for each $t$,
  \begin{align*}
    \bias_t(\wbarr[t]) = \sum_{i=t}^{T-1}
    \prn*{\prod_{j=t}^{i}\Aclt[j]^{\trn}}P_{i+1}w_i +  \sum_{i=t}^{T-1}
    \prn*{\prod_{j=t}^{i-1}\Aclt[j]^{\trn}}( K_i^\top \Ru b_i - \Rx a_i) ,
\end{align*}
with the convention that the empty product is equal to $1$.
Thus, we indeed have
\begin{align*}
  \qstar_t(\wbarr[t]) = \Sigma_t^{-1}\prn*{ - R_u b_t + B^{\trn}\sum_{i=t}^{T-1}\prn*{\prod_{j=t+1}^{i}\Aclt[j]^{\trn}} P_{i+1}w_i  + B^{\trn}\sum_{i=t+1}^{T-1}\prn*{\prod_{j=t+1}^{i-1}\Aclt[j]^{\trn}}( K_i^\top \Ru b_i - \Rx a_i)}.
\end{align*}
\end{dproof}
\begin{dproof}[\pfref{lem:closed_loop_bound}]
  Consider the noiseless LQR setup where
  \[
    x_{t+1}=Ax_{t}+Bu_t.
  \]
  The optimal policy for this setup is given by $u_t=-\Kt{}x$. For each $t\leq{}s$, let $x^{\star}_s(x_t=x)$ and $u^{\star}_s(x_t=x)$ respectively denote
the value of the state $x_s$ and control $u_s$ if we begin with
$x_t=x$ and follow the optimal policy until time $s$. Let $\Vf_{t}(x)$ denote the optimal finite-horizon value function for
this noiseless setup, which satisfies
\[
\Vf_{t}(x)\leq{}\tri*{\Pinf{}x,x},
\]
and
\[
\Vf_{t}(x)=\sum_{s=t}^{T}\nrm*{x^{\star}_s(x_t=x)}_{\Rx}^{2}+\nrm*{u^{\star}_s(x_t=x)}_{\Ru}^{2}.
\]
Note that $(x^{\star}_s(x_t=x))^{\trn}=x^{\trn}\prod_{r=t}^{s-1}\Aclt[r]^{\trn}$, and
that we have in particular that
\[
\nrm*{x^{\star}_s(x_t=x)}_{\Rx}^{2}\leq{}\tri*{\Pinf{}x,x},
\]
and so
$\nrm*{x^{\star}_s(x_t=x)}^{2}\leq{}\tri*{\Pinf{}x,x}/\eigmin(\Rx)$. Choosing
$t=\tau_1$ and $s=\tau_{2}+1$, we have
\[
\nrm*{\prod_{t=\tau_1}^{\tau_2}\Aclt[t]^{\trn}x}^{2}\leq{}\frac{\tri*{\Pinf{}x,x}}{\eigmin(\Rx)}.
\]
  The result now follows by recalling the definition of the spectral norm.
\end{dproof}
\begin{dproof}[\pfref{lem:stable_sequence}]
  First observe that for any $t$, we have
  \begin{align*}
    \nrm*{\Lclt}_{\op}
    \leq
    \nrm*{\Lclinf}_{\op} +
    \nrm*{\Lclt-\Lclinf}_{\op}
    &\leq{}\gammast + \kappast\nrm*{\Aclt-\Aclinf}_{\op}\\
    &\leq{}\gammast + \kappast\nrm*{B}_{\op}\nrm*{\Kt-\Kinf}_{\op}.
  \end{align*}
To bound the error between the infinite-horizon optimal controller
$\Kinf$ and the finite-horizon controller $\Kt$, we appeal to the
following lemma.
\begin{lemma}[\cite{dean2018regret}, Lemma E.6;
  \cite{lincoln2006relaxing}, Proposition 1]
  \label{lem:value_iteration}
  Let $\nu=2\nrm*{\Pinf}_{\op}\cdot\prn*{\tfrac{\nrm*{A}_{\op}^{2}}{\eigmin\prn*{\Rx}}\vee\tfrac{\nrm*{B}_{\op}^{2}}{\eigmin\prn*{\Ru}}}$. Then
for all $0\leq{}t\leq{}T$, it holds that
\[
\nrm*{\Kt-\Kinf}_{\op}\leq{}\nrm*{\Pt-\Pinf}_{\Sigt}^{2}\leq{}\nrm*{\Pinf}_{\op}\prn*{1+\tfrac{1}{\nu}}^{-(T-t+1)}.
\]
In particular, for $\nust\ldef{}2\betast\Psist^{2}\Gammast$, we have
\[
\nrm*{\Kt-\Kinf}_{\op}\leq{}\nrm*{\Pt-\Pinf}_{\Sigt}^{2}\leq{}\Gammast\exp\prn*{-\frac{1}{2\nust}(T-t+1)}.
\]
\end{lemma}
\pref{lem:value_iteration} implies that if we set
$\Delta=2\nust{}\log(\nrm*{\Pinf}_{\op}/\veps)$, we have
$\nrm*{\Kt-\Kinf}_{\op}\leq{}\veps$ for all $t\leq{}T-\Delta$. To get
the final result, we choose $\veps =
\frac{1}{2}(1-\gammast)/(\kappast(1\vee\nrm*{B}_{\op}))$.
\end{dproof}

\begin{dproof}[\pfref{lem:closed_loop_refined}]
  Assume for now that $\tau_2\leq{}\Tstab$; if not, the result follows
  trivially from \pref{lem:closed_loop_bound}. We write
  \[
    \nrm*{\prod_{t=\tau_1}^{\tau_2}\Lclt[t]^{\trn}}_{\op}
    \leq{} \nrm*{\prod_{t=\tau_1}^{\tau_2\wedge{}\Tstab}\Lclt[t]^{\trn}}_{\op} \cdot\nrm*{\prod_{t=\Tstab+1}^{\tau_2}\Lclt[t]^{\trn}}_{\op}.
  \]
  For the first term, we have
  \[
    \nrm*{\prod_{t=\tau_1}^{\tau_2\wedge{}\Tstab}\Lclt[t]^{\trn}}_{\op}
    \leq{}
    \prod_{t=\tau_1}^{\tau_2\wedge{}\Tstab}\nrm*{\Lclt[t]^{\trn}}_{\op}
    \leq \gammabar^{\tau_2\wedge\Tstab-\tau_1},
  \]
  using \pref{lem:stable_sequence}. The second term is bounded using
  \pref{lem:closed_loop_bound} as
  \[
    \nrm*{\prod_{t=\Tstab+1}^{\tau_2}\Lclt[t]^{\trn}}_{\op}
    \leq{}\kappast\nrm*{\prod_{t=\Tstab+1}^{\tau_2}\Aclt[t]^{\trn}}_{\op}
    \leq{} \kappast\betast^{1/2}\Gammast^{1/2}.
  \]
\end{dproof}

\begin{dproof}[\pfref{lem:qt_bound}]
  We first bound $\qstar_t$ and $\qstar_{t;t+h}$.  Let $t\in\brk*{T}$ be fixed. Then we have
  \begin{align*}
    \nrm*{\qstar_t(\wr[t])} &=
                              \nrm*{\sum_{i=t}^{T-1}\Sigma_t^{-1}B^{\trn}\prn*{\prod_{j=t+1}^{i}\Aclt[j]^{\trn}}P_{i+1}w_i}\\
                            &\leq{}\nrm*{\Sigma_t^{-1}}_{\op}\nrm*{B}_{\op}\max_{i>t}\nrm*{P_{i+1}}_{\op}\sum_{i=t}^{T-1}\nrm*{\prod_{j=t+1}^{i}\Aclt[j]^{\trn}}_{\op}.\\
    &\leq{}\betast{}\Psist\Gammast\prn*{1+\sum_{i=t+1}^{T-1}\nrm*{\prod_{j=t+1}^{i}\Aclt[j]^{\trn}}_{\op}}.
  \end{align*}
  Furthermore, the same argument shows that we have
  \[
\nrm*{\qstar_{t;t+h}(w_{t:t+h})} \leq{} \betast{}\Psist\Gammast\prn*{1+\sum_{i=t+1}^{T-1}\nrm*{\prod_{j=t+1}^{i}\Aclt[j]^{\trn}}_{\op}},
\]
as well.
  If $i>\Tstab$, we trivially bound the summand as
  $\betast^{1/2}\Gammast^{1/2}$ using
  \pref{lem:closed_loop_bound}. Otherwise, we have
  $t+1\leq{}i\leq\Tstab$, and we use
  \pref{lem:closed_loop_refined}, which gives
  \[
    \nrm*{\prod_{j=t+1}^{i}\Aclt[j]^{\trn}}_{\op} \leq{}
    \kappast^{2}\betast^{1/2}\Gammast^{1/2}\cdot{}\gammabar^{i-(t+1)}.
  \]
  Summing across the two cases, we have
  \begin{align*}
        \nrm*{\qstar_t(\wr[t])}
    &\leq{}\betast{}\Psist\Gammast\prn*{1+\betast^{1/2}\Gammast^{1/2}\Deltast
      + 
      \kappast^{2}\betast^{1/2}\Gammast^{1/2}\sum_{i=t+1}^{\Tstab}\gammab^{i-(t+1)}}\\    &\leq{}\betast{}\Psist\Gammast\prn*{1+\betast^{1/2}\Gammast^{1/2}\Deltast
      + 
                                                                                           \kappast^{2}\betast^{1/2}\Gammast^{1/2}\sum_{i=0}^{\infty}\gammab^{i}}\\
    &\leq{}\betast{}\Psist\Gammast\prn*{1+\betast^{1/2}\Gammast^{1/2}\Deltast
      + 
      2\kappast^{2}\betast^{1/2}\Gammast^{1/2}(1-\gammab)^{-1}}\\
    &\leq
      2\betast^{3/2}\Psist\Gammast^{3/2}(\Deltast + \kappast^{2}(1-\gammab)^{-1}).
  \end{align*}
  Recalling the definition of $\Deltast$, this is at most
  \[
    \bigoht\prn*{
      \betast^{5/2}\Psist^{3}\Gammast^{5/2}\kappast^2(1-\gammast)^{-1}
    }.
  \]
To bound $\qstar_{\infty;h}$, recall that we have
\[
  \qstar_{\infty;h}(\matw_{h+1}) \ldef \sum_{i=1}^{h+1}\Sigma_{\infty}^{-1}B^{\trn}(\Aclinf^{\trn})^{i-1}P_{\infty}w_{i}.
\]
It immediately follows that we have
\begin{align*}
  \nrm*{\qstar_{\infty;h}(\matw_{h+1})} \leq{}
                                          \nrm*{\sum_{i=1}^{h+1}\Sigma_{\infty}^{-1}B^{\trn}(\Aclinf^{\trn})^{i-1}P_{\infty}w_{i}}
                                        &\leq{}
                                          \betast\Psist\Gammast\sum_{i=1}^{h+1}\nrm*{\Aclinf^{i-1}}_{\op}.
\end{align*}
We may further upper bound this by
\begin{align*}
  \kappast\betast\Psist\Gammast\sum_{i=1}^{h+1}\nrm*{\Lclinf^{i-1}}_{\op}
  \leq{}
  \kappast\betast\Psist\Gammast\sum_{i=1}^{h+1}\gammast^{i-1}
  \leq{} \kappast\betast\Psist\Gammast(1-\gammast)^{-1}.
\end{align*}
  
\end{dproof}

  \begin{dproof}[\pfref{lem:at_bound}]
By a change of variables, we have
  \[
    \nrm*{\Aclt[i\to{}t]-\Aclinf^{t-i}}_{\op}
    \leq{} \kappast\nrm*{\Lclt[i\to{}t]-\Lclinf^{t-i}}_{\op}.
  \]
  Let us drop the ``cl'' subscript to keep notation succinct. Recall that for all
  $t\leq{}\Tstab$, $\nrm*{\Lt}_{\op}\leq{}\gammab<1$, and that
  $\nrm*{\Linf}_{\op}\leq{}\gammast<1$. We proceed by a telescoping
  argument:
  \[
    L_{i\to{}t}-\Linf^{t-i}
    = L_t(L_{i\to{}t-1}-\Linf^{t-i-1})
    + \Linf^{t-i-1}(L_t-\Linf),
  \]
  and so
  \[
    \nrm*{L_{i\to{}t}-\Linf^{t-i}}_{\op}
    = \gammab\nrm*{L_{i\to{}t-1}-\Linf^{t-i-1}}_{\op}
    + \gammast^{t-i-1}\nrm*{L_t-\Linf}_{\op}.
  \]
  Proceedings backwards in the same fashion, we have
  \begin{align*}
    \nrm*{\Aclt[i\to{}t]-\Aclinf^{t-i}}_{\op}
    &\leq{}\kappast\gammab^{t-i-1}\sum_{j=i+1}^{t}\nrm*{L_j-\Linf}_{\op}\\
    &\leq{}\kappast^{2}\gammab^{t-i-1}\sum_{j=i+1}^{t}\nrm*{\Aclt[j]-\Aclinf}_{\op}\\
    &\leq{}\kappast^{2}\Psist\gammab^{t-i-1}\sum_{j=i+1}^{t}\nrm*{K_j-\Kinf}_{\op}.
  \end{align*}
  Using \pref{lem:value_iteration}, we are guaranteed that by setting
  \[
    \Delta = C\cdot{}\betast\Psist^{2}\Gammast\log(\kappast^{2}\Psist\Gammast(1-\gammast)^{-1}\cdot\alpha{}T^{3})\geq{}\Deltast,
  \]
  where $C$ is a sufficiently large constant, we have
  \[
    \nrm*{\Kt-\Kinf}_{\op}\leq{}
    \frac{1}{\kappast^{2}\Psist\cdot(\alpha{}T^{3})}\quad\forall{}t\leq{}T-\Delta,
    \]
  and in particular,
  \[
    \nrm*{\Aclt[i\to{}t]-\Aclinf^{t-i}}_{\op} \leq{} \frac{1}{\alpha{}T^{2}}.
  \]
\end{dproof}

\begin{dproof}[\pfref{lem:state_action_bound}]
  We first handle the policy $\pi$. Observe the state at each step is given by
\[
  x^{\pi}_{t+1}(\matw)=\sum_{i=1}^{t}(A-B\Kinf)^{t-i}w_i - \sum_{i=1}^{t}(A-B\Kinf)^{t-i}Bq_i(\matw).
\]
Hence, using \pref{lem:stable_bound}, we have

\begin{align*}
  \nrm*{x^{\pi}_{t+1}(\matw_t)}
  \leq{}
  \kappast\Psist\sum_{i=1}^{t}\gammast^{t-i}(1+\max_{i\leq{}t}\nrm*{q_{i}(\matw)})
  &\leq{}
    \kappast\Psist(1-\gammast)^{-1}(1+\max_{i\leq{}t}\nrm*{q_{i}(\matw)})\\
  &\leq{} 2\kappast\Psist(1-\gammast)^{-1}\Dq.
\end{align*}
We can now bound the control as
\begin{align*}
  \nrm*{u^{\pi}_{t+1}(\matw)}
  &\leq{} \nrm*{\Kinf{}x^{\pi}_{t+1}(\matw)} + \nrm*{q_{t+1}(\matw)}\\
  &\leq{}2\kappast\betast\Psist^{3}\Gammast(1-\gammast)^{-1}\Dq + \Dq\\
   &\leq{}3\kappast\betast\Psist^{3}\Gammast(1-\gammast)^{-1}\Dq.
\end{align*}
where the second inequality uses \pref{eq:k_bounds} along with the
previous bound on $x_{t}^{\pi}$.

We now handle the policy $\pilearn$. Recall that the state reached after
playing any controller of the form
$\pilearn_t(x,\matw)=-K_t{}x-q_t(\matw)$ for every step is given by
\[
  x_{t+1}^{\pilearn}(\matw)=\sum_{i=1}^{t}\Aclt[i\to{}t]w_i - \sum_{i=1}^{t}\Aclt[i\to{}t]Bq_i(\matw),
\]
and so
\begin{align*}
  \nrm*{x_{t+1}^{\pilearn}(\matw)} \leq{}(1+\Psist\max_{1\leq{}i\leq{}t}\nrm*{q_i(\matw)})\cdot{}\sum_{i=1}^{t}\nrm*{\Aclt[i\to{}t]}_{\op}.
\end{align*}
By \pref{lem:closed_loop_bound}, we have
\begin{align*}
  \sum_{i=1}^{t}\nrm*{\Aclt[i\to{}t]}_{\op}
  &\leq{}
    \sum_{i=1}^{t}\kappast^{2}\betast^{1/2}\Gammast^{1/2}\gammab^{\Tstab\wedge{}t-\Tstab\wedge(i+1)}\\
  &\leq{}
    C\cdot{}\kappast^{2}\betast^{1/2}\Gammast^{1/2}(\Deltast
    + (1-\gammab)^{-1}),
    \end{align*}
    where $C$ is a universal constant. Recalling the value for
    $\Deltast$, this gives
    \[
      \sum_{i=1}^{t}\nrm*{\Aclt[i\to{}t]}_{\op}
      \leq{} \bigoht(\kappast^{2}\betast^{3/2}\Psist^{2}\Gammast^{3/2}(1-\gammast)^{-1}).
    \]
Hence, we can bound the state norm as
\[
      \nrm*{x_{t+1}^{\pilearn}(\matw)}
      \leq{}\bigoht\prn*{
        \kappast^{2}\betast^{3/2}\Psist^{3}\Gammast^{3/2}(1-\gammast)^{-1}\cdot\Dq
      }.
    \]
    Finally, we bound the control norm as
    \begin{align*}
      \nrm*{u_{t+1}^{\pilearn}(\matw)}
      \leq{} \nrm*{\Kt}_{\op}\nrm*{x_{t+1}^{\pilearn}(\matw)}
      + \nrm*{q_{t+1}(\matw)}.
    \end{align*}
    We use that $P_t\psdleq\Pinf$ for all $t$ to bound
    \[
      \nrm*{K_t}_{\op}\leq{}\betast\Psist^{2}\Gammast,
    \]
    which gives
    \[
      \nrm*{u_{t+1}^{\pilearn}(\matw)}
        \leq{}\bigoht\prn*{
        \kappast^{2}\betast^{5/2}\Psist^{5}\Gammast^{5/2}(1-\gammast)^{-1}\cdot\Dq
      }.
    \]
  \end{dproof}
\subsection{Performance difference lemma \label{app:additional_proofs}}
Below we state a variant of the performance difference lemma for an
abstract MDP setting that generalizes the LQR setting studied in this
paper. The setting as follows:

Begin at state $x_1\in\cX$. Then,
for $t=1,\ldots,T$:
\begin{itemize}
\item Agent selects control $u_t\in\cU$.
\item Agent observes $w_t\in\cW$ and experiences instantaneous loss $\ls(x_t,u_t,w_t)$.
  \item State evolves as $x_{t+1}\simiid \system(x_t,u_t,w_t)$, where
  $\system(x,u,w)\in\Delta(\cX)$.
\end{itemize}
We define the expected loss of a policy $\pi_t(x;\matw)$ in this setting
as 
\begin{equation}
  \label{eq:cost}
  \cost(\pi;\matw) = \En_{\pi,\matw}\brk*{\sum_{t=1}^{T}\ls(x_t,u_t,w_t)},
\end{equation}
where $\En_{\pi,\matw}$ denotes expectation with respect to the system
dynamics with $\matw$ fixed. For each policy $\pi$, we define the
action-value function for $\pi$ as follows:
\begin{equation}
  \label{eq:qhat}
  \Qhat^{\pi}_{t:\tau}(x,u;\matw_{\tau}) = \En_{\pi,\matw_{\tau}}\brk*{\sum_{s=t}^{\tau}\ls(x_s,u_s,w_s)\mid{}x_t=x,u_t=u}.
\end{equation}
The performance difference lemma can now be stated as follows.
  \begin{lemma}[Performance difference lemma]
  \label{lem:pd}
  Let $\pihat$ and $\pi$ be any pair of policies of the form
  $\pi_t(x;\matw)$ (i.e., Markovian, but with potentially arbitrary dependence on
  the sequence $\matw$). Then it holds that
  \begin{align}
    \cost(\pihat;\matw) - \cost(\pi;\matw) &= \En_{\pi,\matw}\brk*{
      \sum_{t=1}^{T}\Qhat^{\pihat}_{t}(x_t,\pihat(x_t;\matw);\matw)
      - \Qhat^{\pihat}_{t}(x_t,\pi(x_t;\matw);\matw)
                                             } \label{eq:pd1}\\
    &= \En_{\pihat,\matw}\brk*{
      \sum_{t=1}^{T}\Qhat^{\pi}_{t}(x_t,\pihat(x_t;\matw);\matw)
      - \Qhat^{\pi}_{t}(x_t,\pi(x_t;\matw);\matw)
      }.\label{eq:pd2}
  \end{align}
\end{lemma}
\begin{dproof}[\pfref{lem:pd}]
  Let $t$ be fixed. Observe that for any $x$, we have
  \begin{align*}
    &\Qhat_t^{\pi}(x,\pi_t(x;\matw);\matw)\\
    &= \ls(x, \pi_t(x;\matw),w_t)
      + \En\brk*{\Qhat_{t+1}^{\pi}(x_{t+1},\pi_{t+1}(x_{t+1};\matw);\matw)\mid{}x_t=x,u_t=\pi_t(x;\matw),\matw}.
  \end{align*}
  We can alternatively write
  \begin{align*}
    \ls(x, \pi_t(x;\matw),w_t) 
      &= \Qhat_t^{\pihat}(x,\pi_t(x;\matw);\matw) \\
    &~~~~-\En\brk*{\Qhat_{t+1}^{\pihat}(x_{t+1},\pihat_{t+1}(x_{t+1};\matw);\matw)\mid{}x_t=x,u_t=\pi_t(x;\matw),\matw}.
  \end{align*}
Combining these identities, we have
  \begin{align}
    &\Qhat_t^{\pihat}(x,\pihat_t(x;\matw);\matw) -
      \Qhat_t^{\pi}(x,\pi_t(x;\matw);\matw)\label{eq:pd_onestep}\\
    &= \Qhat_t^{\pihat}(x,\pihat_t(x;\matw);\matw) -
      \Qhat_t^{\pihat}(x,\pi_t(x;\matw);\matw) \notag\\
    &~~~~+
      \En\brk*{\Qhat_{t+1}^{\pihat}(x_{t+1},\pihat_{t+1}(x_{t+1};\matw);\matw)
      -\Qhat_{t+1}^{\pi}(x_{t+1},\pi_{t+1}(x_{t+1};\matw
);\matw)\mid{}x_t=x,u_t=\pi_t(x;\matw),\matw}.\notag
  \end{align}
  To prove the result, we simply observe that
  \[
    \cost(\pilearn;\matw)-\cost(\pi;\matw)
    = \Qhat_1^{\pilearn}(x_1,\pilearn(x;\matw);\matw)
    - \Qhat_1^{\pi}(x_1,\pi(x;\matw);\matw).
  \]
  The equality \pref{eq:pd1} now follows by applying the identity \pref{eq:pd_onestep}
  to the right-hand side above recursively. To prove \pref{eq:pd2} we use the same argument, except that we
  replace the one-step identity \pref{eq:pd_onestep} with
  \begin{align*}
    &\Qhat_t^{\pihat}(x,\pihat_t(x;\matw);\matw) -
      \Qhat_t^{\pi}(x,\pi_t(x;\matw);\matw)\\
    &= \Qhat_t^{\pi}(x,\pihat_t(x;\matw);\matw) -
      \Qhat_t^{\pi}(x,\pi_t(x;\matw);\matw) \notag\\
    &~~~~+
      \En\brk*{\Qhat_{t+1}^{\pihat}(x_{t+1},\pihat_{t+1}(x_{t+1};\matw);\matw)
      -\Qhat_{t+1}^{\pi}(x_{t+1},\pi_{t+1}(x_{t+1};\matw);\matw)\mid{}x_t=x,u_t=\pihat_t(x;\matw),\matw}.\notag
  \end{align*}
  
\end{dproof}


%% file: appendix_algorithms.tex

\newcommand{\Dx}{D_{x}}
\newcommand{\Du}{D_u}

\subsection{Proof of Theorem \ref*{thm:main_algo}}
\label{app:main_algo_proof}

\mainalgo*

\begin{dproof}[\pfref{thm:main_algo}]
Throughout the proof, we let $\pilearn$ denote the policy of \mainalg, which takes the form
$\pilearn_t(x,\matw_{t-1})=-\Kinf{}x-q^{M_t}(\matw_{t-1})$, where
$M_t=M_t(\matw_{t-1})$ is selected as in \pref{alg:main}. The proof is
split into multiple subsections.
\subsubsection{Reduction to online prediction}
As a first step, we appeal to \pref{lem:disturbance_sufficient} which,
by choosing $\Mst=\cM(m,\Rst,\gammapi)$ for
$m=(1-\gammapi)^{-1}\log((1-\gammapi)^{-1}T)$, ensures that
\begin{align*}
  \cost(\pilearn;\matw) - \inf_{K\in\cK_0}\cost(\pi^{K},\matw)
  \leq{}
  \cost(\pilearn;\matw)
  -     \inf_{M\in\cM_0}\cost(\pi\ind{M},\matw) + \Capx.
\end{align*}
Next, we recall that by the performance difference lemma
\pref{eq:cost_perf_diff}, we have that for any $M\in\Mst$,
\begin{align*}
  &  \cost(\pilearn;\matw)
  -    \cost(\pi\ind{M};\matw)
    = \sum_{t=1}^{T}\advstar_t(u^{\pilearn}_t;x^{\pilearn}_t,\matw)
    - \advstar_t(u^{\pi\ind{M}}_t;x^{\pi\ind{M}}_t,\matw).
\end{align*}
We apply \pref{thm:main_reg_decomp} to both terms in this summation individually. In particular,
by choosing $h=2
(1-\gammast)^{-1}\log(\kappast^{2}\betast^{2}\Psist\Gammast^{2}T^{2})$,
we are guaranteed that
\begin{align*}
&\cost(\pilearn;\matw)
                 -     \inf_{M\in\Mst}\cost(\pi\ind{M};\matw)  \\
  &\leq{}
    \sum_{t=1}^{T}\advhat_{t;h}( M_t ;\matw_{t+h})
    - \inf_{M\in\Mst}\sum_{t=1}^{T}\advhat_{t;h}(M ;\matw_{t+h}) + \Cadv.
\end{align*}
Defining a ``loss function'' $f_t(M) = \advhat_{t;h}( M_t
;\matw_{t+h}) = \|q_t^{M}(\wpast) - \qstar_{\infty;h}(w_{t:t+h})
\|_{\Sigma_{\infty}}^2$, the regret-like quantity above is equivalent
to
\begin{equation}
  \label{eq:regret_truncated}
\sum_{t=1}^{T}f_t(M_t)-\inf_{M\in\Mst}\sum_{t=1}^{T}f_t(M),
\end{equation}
where $\crl*{M_t}$ are the disturbance-action matrices selected by \mainalg.
\subsubsection{Applying the online Newton step algorithm}
\newcommand{\Rons}{R_{\mathrm{ONS}}}
As described in the main body, \mainalg is simply an instance of the
generic reduction from online convex optimization with delays to
vanilla online convex optimization, with either online Newton step or Vovk-Azoury-Warmuth as the base
algorithm in the reduction. For online Newton step, since we have delay $h$, \pref{lem:delay_reduction}
ensures that we have
\[
\sum_{t=1}^{T}f_t(M_t)-\inf_{M\in\Mst}\sum_{t=1}^{T}f_t(M) \leq{} (h+1)\Rons(T/(h+1)),
\]
where $\Rons(T/(h+1))$ is an upper bound on the regret of each \ons{} instance applied to its respective subsequence of
losses. Moreover, by \pref{lem:ons} we are guaranteed that if we
choose $\eta_{\mathrm{ons}}=2\max\crl*{4\Goco\Doco,\alphaoco^{-1}}$ and
  $\veps_{\mathrm{ons}}=\eta_{\mathrm{ons}}^{2}/\Doco$, then
\[
\Rons(T) \leq{} 5(\alphaoco^{-1}+\Goco\Doco)\dim(\Mst)\log{}T,
\]
where $\alphaoco$, $\Goco$, and $\Doco$ are regularity parameters for
the losses $f_t$ which are specified by the following lemma.
\begin{restatable}{lemma}{ocoproperties}
  \label{lem:oco_properties}
The weight set $\Mst$ and loss functions $f_t(M)$ in \pref{eq:regret_truncated} satisfy the following properties:
  \begin{itemize}
\item $\sup_{M,M'\in\Mst}\nrm*{M-M'}_{F}\leq{}4\betast\Psist^{2}\Gammast\kappapi^{2}(1-\gammapi)^{-1}\cdot{}\sqrt{\dimx\wedge\dimu}\rdef\Doco$.
\item $\sup_{M\in\Mst}\nrm*{\grad{}f_t(M)}_{F}\leq{}\bigoht\prn*{
    \Dq\Psist^{2}\Gammast(1-\gammapi)^{-1/2}
  }\rdef\Goco$.
\item $f_{t}$ is $\alphaoco$-exp-concave over $\Mst$, where $\alphaoco\ldef{}(4\Dq^{2}\Psist^{2}\Gammast)^{-1}$.
\end{itemize}
\end{restatable}
With this lemma, we can crudely bound the regret of \ons{} as 
\begin{align*}
  \Rons(T) &=\bigoht\prn*{(\Goco\Doco+\alphaoco^{-1})\mathrm{dim}(\Mst)\log{}T}\\
  &=\bigoht\prn*{m\dimx\dimu(\Goco\Doco+\alphaoco^{-1})\log{}T}\\
  &=\bigoht\prn*{(1-\gammapi)^{-1}\dimx\dimu(\Goco\Doco+\alphaoco^{-1})\log^{2}T}\\
  &=\bigoht\prn*{\dimx\dimu\sqrt{\dimx\wedge\dimu}\cdot{}
    \Dq^{2}\kappapi^{2}\betast\Psist^{2}\Gammast^{2}
    (1-\gammapi)^{-5/2}\log^{2}T} \\
  &\leq{}\bigoht\prn*{\dimx\dimu\sqrt{\dimx\wedge\dimu}\cdot{}
    \kappapi^{6}\betast^{6}\Psist^{8}\Gammast^{7}
    (1-\gammapi)^{-9/2}\log^{2}T}.
\end{align*}

\subsubsection{Applying the Vovk-Azoury-Warmuth algorithm}
\newcommand{\Rvaw}{R_{\mathrm{VAW}}}
\newcommand{\Qoco}{Q_{\mathrm{oco}}}
If we use \vvaw as the base algorithm instead of \ons, then \pref{lem:delay_reduction}
implies that
\[
\sum_{t=1}^{T}f_t(M_t)-\inf_{M\in\Mst}\sum_{t=1}^{T}f_t(M) \leq{} (h+1)\Rvaw(T/(h+1)),
\]
where $\Rvaw(T/(h+1))$ is an upper bound on the regret of each \vvaw{}
instance. \pref{thm:vvaw} (detailed in \pref{app:vaw}) ensures that by setting $\veps_{\mathrm{vaw}}=\nrm*{\Siginf}_{\op}\Dq^{2}\Doco^{-2}$, we have
\[
\Rvaw(T) \leq{} 5\nrm*{\Siginf}_{\op}\Dq\dim(\Mst)\log\prn*{1 + \Dq^{-2}\Doco^{2}\Qoco{}T/\dim(\Mst)},
\]
where $\Doco$ is as in \pref{lem:oco_properties} and
\[
  \Qoco\ldef{}\sup_{M\neq{}0}\frac{\nrm*{q^{M}(\matw)}}{\nrm*{M}_F}
    \leq{}
    \sup_{M\neq{}0}\frac{\sum_{i=1}^{m}\nrm*{M^{[i]}}_{\op}}{\nrm*{M}_F}
    \leq{} \sqrt{m}.
\]
Recalling that $\nrm*{\Siginf}_{\op}\leq{}2\Psist^{2}\Gammast$, 
$\Dq\leq{}\bigoht\prn*{
    \betast^{5/2}\Psist^{3}\Gammast^{5/2}\kappapi^{2}(1-\gammapi)^{-1}
  }$, and $\dim(\cM_0)=\dimx\dimu{}m=
  \bigoht(\dimx\dimu{}(1-\gammapi)^{-1}\log{}T)$ (using the choice of
  $m$ from \pref{lem:disturbance_sufficient}), we can simplify to
  \begin{align*}
    \Rvaw(T) &\leq{}
    \bigoht\prn*{\nrm*{\Siginf}_{\op}\Dq\dimx\dimu{}m\log{}T
    }\\
    &\leq{}
    \bigoht\prn*{\dimx\dimu{}\log^{2}T\cdot{}\betast^{5/2}\Psist^{5}\Gammast^{7/2}\kappapi^{2}(1-\gammapi)^{-2}
    }.
  \end{align*}

\subsubsection{Putting everything together}
We now summarize the development so far. Suppose we choose $\Mst$ as in
\pref{lem:disturbance_sufficient}, using $m=(1-\gammapi)^{-1}\log((1-\gammapi)^{-1}T)$.
\pref{lem:ons} implies that if we run \vaw as the base
algorithm in the reduction using $\veps_{\mathrm{vaw}}=\nrm*{\Siginf}_{\op}\Dq^{2}\Doco^{-2}$
and delay parameter $h=2
(1-\gammast)^{-1}\log(\kappast^{2}\betast^{2}\Psist\Gammast^{2}T^{2})$,
we have
\begin{align*}
\sum_{t=1}^{T}f_t(M_t) - \inf_{M\in\Mst}\sum_{t=1}^{T}f_t(M) &\leq{}
                                                               (h+1)\Rvaw(T/(h+1))\\
  &\leq{}\bigoht\prn*{h\cdot{}\dimx\dimu{}\log^{2}T\cdot{}\betast^{5/2}\Psist^{5}\Gammast^{7/2}\kappapi^{2}(1-\gammapi)^{-2}
    } \\
  &\leq{}\underbrace{\bigoht\prn*{\dimx\dimu{}\log^{3}T\cdot{}\betast^{5/2}\Psist^{5}\Gammast^{7/2}\kappapi^{2}(1-\gammapi)^{-3}
    }}_{\rdef\Creg}.
\end{align*}
In total, we have
\begin{align*}
  \KReg
  \leq{} \Creg
    +  \Capx  + \Cadv 
  \leq{}
    \bigoht\prn*{\dimx\dimu\log^{3}T\cdot{}
    \betast^{11}\Psist^{19}\Gammast^{11}\kappapi^{8} (1-\gammapi)^{-4}
    }.
\end{align*}
\end{dproof}

\subsection{Supporting lemmas}
\label{app:algorithm_supporting}

\disturbancesufficient*

\begin{dproof}[\pfref{lem:disturbance_sufficient}]
  Let $K\in\Knot$ be fixed. Consider a policy
\[
\pi_{t}\ind{M}(x;\matw_{t-1}) = -\Kinf{}x - q^{M}(\matw_{t-1}),
\]
Following \cite{agarwal2019online}, we set
\[
  \Mi = (K-\Kinf)(A-BK)^{i-1}.
\]
Suppose for now that $\pi\ind{M}$ and $\pi^{K}$ have
$\nrm*{x_t}\vee\nrm*{u_t}\leq{}\wt{D}$ for all $t$. Then Lemma 5.2 of
\cite{agarwal2019online} implies that
\begin{equation}
  \cost(\picheck^{M};\matw) \leq{} \cost(\pi^{K};\matw) + \bigoh\prn*{
    \wt{D} \Psist^{3}\kappapi^{5}\cdot{}m\gammapi^{m+1}T
  }.\label{eq:agarwal}
\end{equation}
Let us bound the norms for the matrices $\Mi$ that achieve this
bound. First observe that
\begin{equation}
\nrm*{\Kinf}_{\op}\leq{}\nrm*{\Siginf}^{-1}_{\op}\nrm*{A}_{\op}\nrm*{B}_{\op}\nrm*{\Pinf}_{\op}\leq\betast\Psist^{2}\Gammast,\quad\text{and}\quad\nrm*{K}_{\op}\leq{}\kappapi.\label{eq:k_bounds}
\end{equation}
Consequently, \pref{lem:stable_bound} implies that
\[
  \nrm*{\Mi}_{\op}\leq{}(\nrm*{K}_{\op}+\nrm*{\Kinf}_{\op})\kappapi\gammapi^{i-1}
  \leq{} 2\kappapi^{2}\betast\Psist^{2}\Gammast\gammapi^{i-1}.
\]
Hence, if the use controller $\pi\ind{M}$, it
would suffice to take
\[
\Mst = \crl*{M=\crl*{\Mi}_{i\in\brk*{m}}\mid{}\nrm*{\Mi}_{\op}\leq{}2\betast\Psist^{2}\Gammast\kappapi^{2}\gammapi^{i-1}}.
\]
To conclude the proof, we provide a bound on $\wt{D}$. To begin, note
that each
$M\in\Mst$ has
\begin{equation}
\nrm*{q_i^{M}(\matw_{i-1})} \leq{}
\sum_{i=1}^{m}\nrm*{\Mi}_{\op}
\leq{}2\betast\Psist^{2}\Gammast\kappapi^{2}(1-\gammapi)^{-1}\rdef\Dm.\label{eq:qm_bound}
\end{equation}
We now provide a bound on $\wt{D}$. First, observe that when $\pi$ is the static linear controller $\pi^{K}$, we have
\[
  x_{t+1}(\matw_{t}) = \sum_{i=1}^{t}(A-BK)^{t-i}w_i,
\]
and so, use \pref{lem:stable_bound}, we have
\[
  \nrm*{x_{t}(\matw_{t-1})}_{\op}\leq{}\kappapi\sum_{i=1}^{t}\gammapi^{t-i}
  \leq \kappapi(1-\gammapi)^{-1},
\]
and
$\nrm*{u_{t}(\matw_{t-1})}=\nrm*{Kx_{t}(\matw_{t-1})}_{\op}\leq{}\kappapi^{2}(1-\gammapi)^{-1}$. To
bound the radius for the policies $\pi\ind{M}$, we use
\pref{lem:state_action_bound}, along with the bound \pref{eq:qm_bound}
to get the following result.
\begin{corollary}
  \label{cor:picheck_bound}
  For any $M\in\Mst$, the controller $\pi\ind{M}$ has
  \[
    \nrm*{x^{\pi\ind{M}}_{t+1}(\matw_t)}\leq{}2\betast\Psist^{3}\Gammast\kappapi^{3}(1-\gammapi)^{-2},
    \quad\text{and}\quad
    \nrm*{u^{\pi\ind{M}}_{t+1}(\matw_t)}\leq{}3\betast^{2}\Psist^{5}\Gammast^{2}\kappapi^{3}(1-\gammapi)^{-2}.
    \]
\end{corollary}
Hence, we may take
\[
\wt{D} = 2\betast^{2}\Psist^{5}\Gammast^{2}\kappapi^{3}(1-\gammapi)^{-2},
\]
and so \pref{eq:agarwal} yields
\[
  \cost(\pi\ind{M};\matw) \leq{} \cost(\pi^{K};\matw) + \bigoh\prn*{
    \betast^{2}\Psist^{8}\Gammast^{2}(1-\gammapi)^{-2}\kappapi^{7}\cdot{}m\gammapi^{m+1}T
  }.
\]
By choosing $m=(1-\gammapi)^{-1}\log((1-\gammapi)^{-1}T)$, we are
guaranteed that
\[
  \cost(\pi\ind{M};\matw) \leq{} \cost(\pi^{K};\matw) \leq{} \Capx.
\]
As a closing remark, we observe that \pref{eq:qm_bound} implies that we may take $\Dq=\max\crl*{2\kappapi^{2}\betast\Psist^{2}\Gammast(1-\gammapi)^{-1},\Dqst}$,
as the radius for the predictions $q_t^{M}$ by the learner,  benchmark
class, and optimal policy. Hence, recalling the value for $\Dqst$ from
\pref{lem:qt_bound}, we may take
\[
  \Dq \leq{} \bigoht\prn*{
    \betast^{5/2}\Psist^{3}\Gammast^{5/2}\kappapi^{2} (1-\gammapi)^{-1}
  }.
\]
    \end{dproof}

\ocoproperties*
    
    \begin{dproof}[\pfref{lem:oco_properties}]
  For the first property, observe that for each $M\in\Mst$, we have
  \begin{align*}
    \nrm*{M}_{F}
    = \sqrt{\sum_{i=1}^{m}\nrm*{\Mi}_{F}^{2}}
    &\leq{}\sqrt{\dimx\wedge\dimu}
      \sqrt{\sum_{i=1}^{m}\nrm*{\Mi}_{\op}^{2}}\\
    &\leq{}\sqrt{\dimx\wedge\dimu}\cdot{}
      2\kappa^{2}\betast\Psist^{2}\Gammast\sqrt{\sum_{i=1}^{m}\gammapi^{2(i-1)}}\\
      &\leq{}\sqrt{\dimx\wedge\dimu}\cdot{} 2\kappa^{2}\betast\Psist^{2}\Gammast(1-\gammapi)^{-1}.
  \end{align*}
  The bound for $\Doco$ now follows by triangle inequality.

  For the second property, we directly prove that $f_t$ is Lipschitz as
    \newcommand{\Mipr}{M'^{[i]}}
  follows: For any $M,M'\in\Mst$,
  \begin{align*}
    &\nrm*{\qstar_{t:t+m}(w_{t:t+m})-q^{M}(\matw_{t-1})}_{\Sigma_t}^{2}
    -
    \nrm*{\qstar_{t:t+m}(w_{t:t+m})-q^{M'}(\matw_{t-1})}_{\Sigma_t}^{2}\\
    &\leq{}
      2\nrm*{\Sigma_t}_{\op}D_q\nrm*{q^{M}(\matw_{t-1})-q^{M'}(\matw_{t-1})}\\
    &=
      2\nrm*{\Sigma_t}_{\op}D_q\nrm*{\sum_{i=1}^{m}(\Mi-\Mipr)w_{t-i}}.
  \end{align*}
  We finish the bound as follows:
  \[
    \sum_{i=1}^{m}\nrm*{\Mi-\Mipr}_{\op}
    \leq{}    \sum_{i=1}^{m}\nrm*{\Mi-\Mipr}_{F}
        \leq{}    \sqrt{m}\nrm*{M-M'}_{F}.
      \]
      To simplify the bound, we use that
      $\nrm*{\Sigma_t}_{\op}\leq{}\Psist^{2}\Gammast$ and that $\sqrt{m}=\bigoht((1-\gammapi)^{-1/2})$.

For the third property, we observe that that $f_t(M)$ obeys the structure in
\pref{lem:exp_concave_quadratic}, since $q^{M}(\matw_{t-1})$ is a
linear mapping from $\prod_{i=1}^{m}\bbR^{\dimu\times{}\dimx}$ to
$\bbR^{\dimu}$, and since $\Sigma_t\psdgt{}0$. Thus, to prove the
exp-concave property, we simply bound the range of the loss as
\[
  \nrm*{\qstar_{t:t+m}(w_{t:t+m})-q^{M}(\matw_{t-1})}_{\Sigma_t}^{2}
  \leq{} 2\Dq^{2}\nrm*{\Sigma_t}_{\op}
  \leq{}2\Dq^{2}\Psist^{2}\Gammast.
\]
  
\end{dproof}

\subsection{Vector-valued Vovk-Azoury-Warmuth algorithm}
  \label{app:vaw}
  \input{appendix_vaw}

\subsection{Supporting lemmas for online learning \label{app:supporting_online_learning}}

  \delayreduction*
  \begin{dproof}[\pfref{lem:delay_reduction}]
    Let $\cI_i$ denote the rounds in which instance $i$ was used. Then we have
  \begin{align*}
  \Reg &= \sup_{z\in\cC}\crl*{\sum_{t=1}^{T}f_t(z_t) - \sum_{t=1}^{T}f_t(z)} \\
  &= \sup_{z\in\cC}\crl*{\sum_{i=1}^{h+1}\sum_{t\in\cI_i}f_t(z_t) - f_t(z)} \\
  &\leq{} \sum_{i=1}^{h+1}\sup_{z\in\cC}\crl*{\sum_{t\in\cI_i}f_t(z_t) - f_t(z)} \\
  &\leq{} \sum_{i=1}^{h+1}R(T/(h+1)) \\
  &= (h+1)R(T/(h+1)).
  \end{align*}
  \end{dproof}

  \expconcavequadratic*
    \begin{dproof}[\pfref{lem:exp_concave_quadratic}]
      Recall that the function $f$ is
    $\alpha$-exp-concave if and only if
  \[
  \grad^{2}f(z)\psdgeq{}\alpha\grad{}f(z)\grad{}f(z)^{\trn}.
  \]
    We have
    \[
  \grad{}f(z) = 2A^{\trn}\Sigma(Az -
  b),\quad\text{and}\quad\grad^{2}f(z) = 2A^{\trn}\Sigma{}A.
  \]
  Hence
  \[
    \grad{}f(z)\grad{}f(z)^{\trn}\psdleq{}4A^{\trn}\Sigma{}A\nrm*{b-Az}_{\Sigma}^{2}
    \leq{} 2R\cdot{}\grad^{2}f(z).
  \]
  \end{dproof}


%% file: appendix_vaw.tex

In this section we develop a variant of the Vovk-Azoury-Warmuth algorithm \citep{Vovk98,AzouryWarmuth01} for a vector-valued online regression setting. At each timestep $t=1,\ldots,T$, the learner receives a matrix $A_t\in{}\bbR^{d_1\times{}d_2}$, predicts $z_t\in\bbR^{d_2}$, then receives $b_t\in\bbR^{d_1}$ and experiences loss $f_t(z_t)$, where $f_t(z)=\nrm*{A_tz-b_t}_{\Sigma}^{2}$ and $\Sigma\psdgt{}0$ is a known matrix. The goal of the learner is to attain low regret
\[
\sum_{t=1}^{T}f_t(z_t)-\inf_{z\in\cC}\sum_{t=1}^{T}f_t(z),
\]
where $\cC$ is a convex constraint set. Recall from \pref{alg:vaw} that $\vvaw$ is the algorithm which, at time $t$, predicts with
\begin{equation}
z_t = \argmin_{z\in\cC}\crl*{\tri*{z,{\textstyle-2\sum_{i=1}^{t-1}A_i^{\trn}\Sigma{}b_i}} + \nrm*{z}_{E_t}^{2}},\label{eq:vvaw}
\end{equation}
where $E_t=\veps{}I+\sum_{i=1}^{t}A_i^{\trn}\Sigma{}A_i$.

\begin{theorem}
  \label{thm:vvaw}
  Let $\nrm*{\Sigma}_{\op}\leq{}S$. Suppose that we run the $\vvaw$ strategy (\pref{alg:vaw}) with parameter $\veps$, and that for all $t$ we have $\nrm*{b_t}\leq{}Y$ and $\nrm*{A_t}_{\op}\leq{}R$. Then we are guaranteed that for all $z\in\cC$,
  \begin{equation}
    \label{eq:vvaw_regret}
\sum_{t=1}^{T}f_t(z_t)-\sum_{t=1}^{T}f_t(z) \leq{} \veps\nrm*{z}^{2} + 4SY^{2}\cdot{}d_2\log\prn*{1+ SR^{2}T/(d_2\veps)}.
\end{equation}
In particular, if $\sup_{z\in\cC}\nrm*{z}\leq{}B$, then by setting $\veps=SY^{2}/B^{2}$ we are guaranteed that
  \begin{equation}
    \label{eq:vvaw_regret2}
\sum_{t=1}^{T}f_t(z_t)-\inf_{z\in\cC}\sum_{t=1}^{T}f_t(z) \leq{} 5SY^{2}\cdot{}d_2\log\prn*{1+ B^{2}R^{2}Y^{-2}T/d_2}.
\end{equation}
\end{theorem}

\begin{dproof}[\pfref{thm:vvaw}]
We assume $\Sigma=I$ without loss of generality by reparameterizing via $A_t'=\Sigma^{1/2}A_t$ and $b_t'=\Sigma^{1/2}b_t$, with $Y$ and $R$ scaled up by a factor of $S^{1/2}$.

Our proof follows the treatment of \vaw in \cite{orabona2015generalized}, which views the algorithm as an instance of online mirror descent with a sequence of time-varying regularizers. Consider the following algorithm parameterized by a sequence of convex regularizers $\cR_t:\cC\to\bbR$.
\begin{itemize}
\item Initialize $\theta_1=0$.
  \item For $t=1,\ldots,T$: 
    \begin{itemize}
    \item  Let $z_t=\argmin_{z\in\cC}\crl*{\tri*{z,\theta_{t}} +
        \cR_t(z)}$.
      \item Receive $g_t$ and set $\theta_{t+1}=\theta_t + g_t$.
  \end{itemize}
\end{itemize}
The following lemma bounds the regret of this strategy for online linear optimization.
\begin{lemma}[\cite{orabona2015generalized}, Lemma 1]
  \label{lem:omd}
  Suppose that each function $\cR_t$ is $\beta$-strongly convex with respect to a norm $\nrm*{\cdot}_t$, and let $\nrm*{\cdot}_{t,\star}$ denote the dual norm. Then the online mirror descent algorithm ensures that for every sequence $g_1,\ldots,g_T$, for all $z\in\cC$,
  \begin{equation}
    \label{eq:omd}
        \sum_{t=1}^{T}\tri*{g_t,z_t-z} \leq{} \cR_T(z) + \sum_{t=1}^{T}\prn*{\frac{\nrm*{g_t}^{2}_{t,\star}}{2\beta} + \cR_{t-1}(z_t)-\cR_{t}(z_t)}.
  \end{equation}
\end{lemma}
Observe that the \vvaw algorithm \pref{eq:vvaw} is equivalent to running online mirror descent with $g_t=-2A_t^{\trn}b_t$ and $\cR_{t}(z)=\nrm*{z}_{E_t}^{2}$. We use this observation to bound the regret through \pref{lem:omd}. In particular, letting $\nrm*{\cdot}_{t}=\nrm*{\cdot}_{E_t}^{2}$, we may take $\beta=1$, which gives
\begin{align*}
  \sum_{t=1}^{T}f_t(z_t)-f_t(z) &=
                                  \sum_{t=1}^{T}\nrm*{A_tz_t-b_t}^{2}-\nrm*{A_tz-b_t}^{2} \\
  &=
    \sum_{t=1}^{T}2\tri*{-A_t^{\trn}b_t,z_t-z} + \sum_{t=1}^{T}\nrm*{A_tz_t}^{2} - \sum_{t=1}^{T}\nrm*{A_tz}^{2}\\
                                &=
                                  \sum_{t=1}^{T}\tri*{g_t,z_t-z} + \sum_{t=1}^{T}\nrm*{A_tz_t}^{2} - \cR_T(z) + \veps\nrm*{z}^{2} \\
  &\leq{} \cR_T(z) + \sum_{t=1}^{T}\prn*{\frac{1}{2}\nrm*{g_t}^{2}_{E_t^{-1}} + \cR_{t-1}(z_t)-\cR_{t}(z_t)}
    + \sum_{t=1}^{T}\nrm*{A_tz_t}^{2} - \cR_T(z) + \veps\nrm*{z}^{2}\\
                                &= \sum_{t=1}^{T}\prn*{\frac{1}{2}\nrm*{g_t}^{2}_{E_t^{-1}} + \cR_{t-1}(z_t)-\cR_{t}(z_t)}
    + \sum_{t=1}^{T}\nrm*{A_tz_t}^{2} + \veps\nrm*{z}^{2},
\end{align*}
where the inequality uses \pref{lem:omd}, along with the fact that the dual norm for $\nrm*{\cdot}_t$ is $\nrm*{\cdot}_{E_t^{-1}}$. To simplify further, we observe that $\cR_{t-1}(z_t)-\cR_{t}(z_t)=-\nrm*{A_tz_t}^{2}$, so that
\begin{align*}
  \sum_{t=1}^{T}f_t(z_t)-f_t(z) &\leq{}
  \veps\nrm*{z}^{2} + \frac{1}{2}\sum_{t=1}^{T}\nrm*{g_t}^{2}_{E_t^{-1}}\\
  &=
  \veps\nrm*{z}^{2} + 2\sum_{t=1}^{T}\nrm*{A_{t}^{\trn}b_t}^{2}_{E_t^{-1}}
  \leq{}
  \veps\nrm*{z}^{2} + 2Y^{2}\sum_{t=1}^{T}\nrm*{E_t^{-1/2}A_{t}^{\trn}}_{\op}^{2}.
\end{align*}
To bound the right-hand side we use a generalization of the usual log-determinant potential argument. Throughout the argument we use that since $E_t\psdgt{}A_t^{\trn}A_t$, $0\leq{}\nrm*{E_t^{-1/2}A_t^{\trn}}_{\op}<1$. To begin, observe that for each $t$, we have
\begin{align*}
\det(E_{t-1}) =
\det(E_{t} - A_t^{\trn}A_t) =
  \det(E_{t})\cdot\det(I - E_t^{-1/2}A_t^{\trn}A_tE_t^{-1/2}).
\end{align*}
Consequently,
\begin{align*}
\frac{\det(E_t)}{\det(E_{t-1})} &= \frac{1}{\det(I - E_t^{-1/2}A_t^{\trn}A_tE_t^{-1/2})}\\
&= \frac{1}{\prod_{i=1}^{d_2}\prn*{1-\lambda_i\prn*{E_t^{-1/2}A_t^{\trn}A_tE_t^{-1/2}}}}
= \prod_{i=1}^{d_2}\frac{1}{1-\lambda_i\prn*{E_t^{-1/2}A_t^{\trn}A_tE_t^{-1/2}}}.
\end{align*}
Next we observe that since $0\leq{}\nrm*{E_t^{-1/2}A_t^{\trn}}_{\op}<1$, we are guaranteed that 
$\frac{1}{1-\lambda_i\prn*{E_t^{-1/2}A_t^{\trn}A_tE_t^{-1/2}}}\geq{}1$ for all $i$, and consequently
\[
\prod_{i=1}^{d_2}\frac{1}{1-\lambda_i\prn*{E_t^{-1/2}A_t^{\trn}A_tE_t^{-1/2}}}
\geq{} \frac{1}{1-\eigmax\prn*{E_t^{-1/2}A_t^{\trn}A_tE_t^{-1/2}}}\geq{} 1+\eigmax\prn*{E_t^{-1/2}A_t^{\trn}A_tE_t^{-1/2}},
\]
where the second inequality uses that $\frac{1}{1-x}\geq{}1+x$ for $x\in[0,1)$.
%
%
Since $\log{}x$ is increasing, this establishes that
\[
\log\prn*{1+\nrm*{E_t^{-1/2}A_t^{\trn}}_{\op}^{2}} = \log\prn*{1+\eigmax\prn*{E_t^{-1/2}A_t^{\trn}A_tE_t^{-1/2}}} \leq{} \log\prn*{\frac{\det(E_t)}{\det(E_{t-1})}}.
\]
Next we use that since $\nrm*{E_t^{-1/2}A_t^{\trn}}_{\op}\leq{}1$, we have
\[
  \nrm*{E_t^{-1/2}A_t^{\trn}}_{\op}^{2}
  \leq 2\cdot{}\log\prn*{1+\nrm*{E_t^{-1/2}A_t^{\trn}}_{\op}^{2}},
\]
using the elementary inequality $x\leq{}2\log(1+x)$ for all $x\in\brk*{0,1}$. Altogether, this gives
\begin{align*}
  \sum_{t=1}^{T}\nrm*{E_t^{-1/2}A_{t}^{\trn}}_{\op}^{2}
  \leq{} 2 \sum_{t=1}^{T}\log\prn*{\frac{\det(E_t)}{\det(E_{t-1})}} = 2\log\prn*{\frac{\det(E_T)}{\det(E_{0})}},
\end{align*}
where we recall $E_0=\veps{}I$. Finally, we have
\begin{align*}
  \log\prn*{\frac{\det(E_T)}{\det(E_{0})}}
  = \sum_{i=1}^{d_2}\log\prn*{1+ \lambda_i\prn*{\sum_{t=1}^{T}A_t^{\trn}A_t}/\veps}
  &\leq{}  d_2\log\prn*{1+ R^{2}T/(d_2\veps)}.
\end{align*}
\end{dproof}


%% file: appendix_analysis.tex
\newcommand{\Tnot}{T_0}
\newcommand{\Deltanot}{\Delta_0}

\subsection{Proof of Theorem \ref*{thm:main_reg_decomp}}
We restate \pref{thm:main_reg_decomp} here for reference.
\mainregretdecomp*

\begin{dproof}[\pfref{thm:main_reg_decomp}]
To begin, recall that by taking $\Dq$ as in \pref{eq:dq}, we have $\nrm*{q_t^{M}}\leq{}\Dq$ for all $M\in\Mst$, and we also have $\Dqst\leq\Dq$.  

For the first step, let $\pi$ be any policy of the form
$\pi_t(x;\matw)=-\Kinf{}x-q_t^{M_t}(\matw_{t-1})$, and let
$\pilearn_t(x;\matw)=-K_tx-q_t^{M_t}(\matw_{t-1})$ be the
corresponding controller that uses the finite-horizon state-feedback
matrices $\crl*{K_t}_{t=1}^{T}$. To begin, using the performance
difference lemma \pref{eq:cost_perf_diff} along with
\pref{lem:kinf_to_kt}, 
\begin{align*}
  \abs*{\sum_{t=1}^{T}\advstar_t(u^{\pi}_t;x^{\pi}_t,\matw)-
  \advstar_t(u^{\pilearn}_t;x^{\pilearn}_t,\matw)} \leq{} \Ckinf.
\end{align*}
Next, using \pref{lem:advstar}, we have
\begin{align*}
  \sum_{t=1}^{T}\advstar_t(u^{\pilearn}_t;x^{\pilearn}_t,\matw) = \sum_{t=1}^{T}\nrm*{q^{M_t}(\matw_{t-1}) - \qstar_t(\wr[t]) }_{\Sigma_t}^2.
\end{align*}
Using \pref{lem:truncate_regret}, the choice of $h$ in the theorem statement guarantees that
\begin{align*}
    \sum_{t=1}^{T}\abs*{\advstar_{t}(u_t^{\pilearn};x_{t}^{\pilearn},\matw)
  - \| q^{M_t}(\matw_{t-1}) -  \qstar_{t;t+h}(w_{t:t+h})\|_{\Sigma_t}^2}
  \leq \Ctrunc,
\end{align*}
and finally \pref{lem:qt_to_qinf} ensures that
\begin{align*}
  &\left|\sum_{t=1}^{T}\| q^{M_t}(\matw_{t-1}) -  \qstar_{t;t+h}(w_{t:t+h})\|_{\Sigma_t}^2
  - \| q^{M_t}(\matw_{t-1}) -
    \qstar_{\infty;h}(w_{t:t+h})\|_{\Sigma_{\infty}}^2\right|\\
  &=
  \left|\sum_{t=1}^{T}\| q^{M_t}(\matw_{t-1}) -  \qstar_{t;t+h}(w_{t:t+h})\|_{\Sigma_t}^2
        - 
        \advhat_{t;h}( M_t ;\matw_{t+h})
        \right|\\
  &\leq{} C_{q_{\infty},\Sigma_{\infty}}.
\end{align*}
Summing up all the error terms, the total error is proportional to
\begin{align*}
  &\Ckinf + \Ctrunc + C_{q_{\infty},\Sigma_{\infty}}\\
  &=\bigoht\prn*{
    \kappast^{4}\betast^{6}\Psist^{13}\Gammast^{6}(1-\gammast)^{-2}\Dq^{2}\cdot{}\log{}(D_qT)
    } + \bigoht(\Dq^{2}\betast\Psist^{4}\Gammast^{2}(1-\gammast)^{-1}\log{}T)\\
  &~~~~~~+ \bigoht\prn*{
    \Dq^{2}\cdot{}\betast^{4}\Psist^{7}\Gammast^{4}\kappast^{2}(1-\gammast)^{-1}h\log(\Dq{}T) }.
\end{align*}
Using the value for $\Dq$ from \pref{eq:dq} and that $h=\bigoht((1-\gammast)^{-1}\log{}T)$, we upper bound the total error as
\begin{align*}
  \bigoht\prn*{
  \betast^{11}\Psist^{19}\Gammast^{11}\kappapi^{8}(1-\gammapi)^{-4}\log^{2}T
  }.
\end{align*}

\end{dproof}

\subsection{Supporting lemmas}

\qstartruncate*

\begin{dproof}[\pfref{lem:qstar_truncate}]
Let $\tau=t+h$. Then we have
  \begin{align*}
    \qstar_{t:\tau}(w_{t:\tau})
    - \qstar_{t:T}(w_{t:T})
    = \sum_{i=\tau+1}^{T-1}\Sigma_t^{-1}B^{\trn}\prn*{\prod_{j=t+1}^{i}\Aclt[j]^{\trn}}P_{i+1}w_i,
  \end{align*}
  Hence we can bound the error as
    \begin{align*}
      \nrm*{\qstar_{t:\tau}(w_{t:\tau})
    - \qstar_{t:T}(w_{t:T})}
      = \betast\Psist\Gammast\sum_{i=\tau+1}^{T-1}\nrm*{\prod_{j=t+1}^{i}\Aclt[j]^{\trn}}_{\op}.
    \end{align*}
    We bound each term in the sum as
    \begin{align*}
      \nrm*{\prod_{j=t+1}^{i}\Aclt[j]^{\trn}}_{\op}
      & \leq{} \nrm*{\prod_{j=t+1}^{\tau+1}\Aclt[j]^{\trn}}_{\op}
        \nrm*{\prod_{j=\tau+1}^{i}\Aclt[j]^{\trn}}_{\op}.
        \intertext{Applying \pref{lem:closed_loop_refined} to the
        first term and \pref{lem:closed_loop_bound} to the second,
        this is at most}
        & \leq{} \kappast^{2}\betast\Gammast\gammab^{\tau-t}.
    \end{align*}
The result follows by summing.
\end{dproof}

\truncateqregret*
\begin{dproof}[\pfref{lem:truncate_regret}]
  First recall that we have
  $\nrm*{\Sigma_t}_{\op}\leq{}\nrm*{R}_{\op}+\nrm*{B}_{\op}^{2}\nrm*{\Pinf}_{\op}\leq{}2\Psist^{2}\Gammast\rdef\Dsig$. Let $h$ be fixed, and let $\Ttrunc\ldef{}\Tstab-h$, so that
  $t+h\leq{}\Tstab$ for all $t\leq{}\Ttrunc$. We begin by writing off
  all of the timesteps after $\Ttrunc$:
  \begin{align*}
    &\sum_{t=1}^{T}\abs*{\advstar_{t}(\pilearn_t(x_t^{\pilearn});x_{t}^{\pilearn},\matw)
      - \| q_t(\matw) -  \qstar_{t;t+h}(w_{t:t+h})\|_{\Sigma_t}^2}\\
    &=\sum_{t=1}^{T}\abs*{\| q_t(\matw) -  \qstar_{t}(w_{t:T})\|_{\Sigma_t}^2
      - \| q_t(\matw) -  \qstar_{t;t+h}(w_{t:t+h})\|_{\Sigma_t}^2}\\
    &\leq{}\sum_{t=1}^{\Ttrunc}\abs*{\| q_t(\matw) -  \qstar_{t}(w_{t:T})\|_{\Sigma_t}^2
      - \| q_t(\matw) -  \qstar_{t;t+h}(w_{t:t+h})\|_{\Sigma_t}^2} + 4\Dq^{2}\Dsig(\Deltast+h)\\
    &\leq{}4\Dq\Dsig\sum_{t=1}^{\Ttrunc}\nrm*{\qstar_{t}(w_{t:T}) -  \qstar_{t;t+h}(w_{t:t+h})}_{\Sigma_t} + 4\Dq^{2}\Dsig(\Deltast+h).
  \end{align*}
    Since $t+h\leq{}\Tstab$ for all $t$ in the last summation,
  \pref{lem:qstar_truncate} implies that
  \begin{align*}
    \sum_{t=1}^{\Ttrunc}\nrm*{\qstar_{t}(w_{t:T}) -  \qstar_{t;t+h}(w_{t:t+h})}_{\Sigma_t} \leq{} \kappast^{2}\betast^{2}\Psist\Gammast^{2}T^{2}\gammab^{h}
    \leq{} \kappast^{2}\betast^{2}\Psist\Gammast^{2}T^{2}\exp\prn*{-(1-\gammab)h}.
  \end{align*}
  By choosing $h=(1-\gammabar)^{-1}\log(\kappast^{2}\betast^{2}\Psist\Gammast^{2}T^{2})$, the total error from
  this term is $\bigoh(1)$. Combining
  this with the previous bound, we see that the total error is at most
  \begin{align*}
    \bigoh(\Dq\Dsig + \Dq^{2}\Dsig(\Deltast+h)).
  \end{align*}
  Lastly, we simplify by using that
  $\Deltast=\bigoht(\betast\Psist^{2}\Gammast)$ and expanding $h$ and $\Dsig$, so
  that the final error term is at most
  \[
    \bigoht(\Dq^{2}\betast\Psist^{4}\Gammast^{2}(1-\gammast)^{-1}\log{}T).
  \]
\end{dproof}
\kinftokt*
\begin{dproof}[\pfref{lem:kinf_to_kt}]
To begin, suppose that that the states under both controllers satisfy
$\nrm*{x}\leq{}\Dx$ and the actions satisfy $\nrm*{u}\leq{}\Du$, where
$\Dx,\Du\geq{}1$. Then,
we immediately have
\begin{align*}
  \abs*{\cost(\pilearn,\matw)-\cost(\pi,\matw)} \leq{}
  2\Psist\sum_{t=1}^{T}\Dx\nrm*{x_{t}^{\pilearn}(\matw)-x_{t}^{\pi}(\matw)}
  + \Du\nrm*{u_{t}^{\pilearn}(\matw)-u_{t}^{\pi}(\matw)},
\end{align*}
which follows becase the function $x\mapsto\nrm*{x}^{2}$ is
$2C$-Lipschitz whenever $\nrm*{x}\leq{}C$. We will first bound the
state and action errors on the right hand side, then give appropriate
bounds on $\Dx$ and $\Du$ at the end of the proof.

Let $\Delta_0$ be fixed, and let $T_0=T-\Delta_0$. Then we can bound
the error further as
\begin{align*}
        &\abs*{\cost(\pilearn;\matw)-\cost(\pi;\matw)}\\
        &\leq{}
        2\Psist\sum_{t=1}^{T}\Dx\nrm*{x_{t}^{\pilearn}(\matw)-x_{t}^{\pi}(\matw)}
          + \Du\nrm*{u_{t}^{\pilearn}(\matw)-u_{t}^{\pi}(\matw)}\\
        &\leq{}
          4\Psist(\Dx^2+\Du^2)\Deltanot
          + 
          2\Psist\sum_{t=1}^{\Tnot}\Dx\nrm*{x_{t}^{\pilearn}(\matw)-x_{t}^{\pi}(\matw)}
        + \Du\nrm*{u_{t}^{\pilearn}(\matw)-u_{t}^{\pi}(\matw)},
      \end{align*}
      For the control error term, we further have
      \begin{align*}
        \sum_{t=1}^{\Tnot}
        \nrm*{u_t^{\pilearn}(\matw)-u_t^{\pi}(\matw)}
        &=\sum_{t=1}^{\Tnot}
          \nrm*{\Kt{}x_t^{\pilearn}(\matw)-\Kinf{}x_t^{\pi}(\matw)}\\
        &\leq{}\sum_{t=1}^{\Tnot}
          \nrm*{\Kinf}\nrm*{x_t^{\pilearn}(\matw)-x_t^{\pi}(\matw)}
          + \Dx\sum_{t=1}^{\Tnot}\nrm*{\Kt{}-\Kinf}_{\op}\\
        &\leq{}\sum_{t=1}^{\Tnot}
          \betast\Psist^{2}\Gammast\nrm*{x_t^{\pilearn}(\matw)-x_t^{\pi}(\matw)}
          + \Dx\sum_{t=1}^{\Tnot}\nrm*{\Kt{}-\Kinf}_{\op}.
      \end{align*}
      In total, this gives us
      \begin{align*}
        &\abs*{\cost(\pilearn;\matw)-\cost(\pi;\matw)}\\
        &\leq{}
          4\Psist(\Dx^2+\Du^2)\Deltanot
          + 
          2\Psist\sum_{t=1}^{\Tnot}(\Dx+\Du\betast\Psist^{2}\Gammast)\nrm*{x_{t}^{\pilearn}(\matw)-x_{t}^{\pi}(\matw)}
          + \Dx\Du\nrm*{\Kt-\Kinf}_{\op}.
      \end{align*}
To bound the state error, we recall that from
\pref{lem:state_expression}, we have
\[
  x_{t+1}^{\pilearn}(\matw_t)
  - x_{t+1}^{\pi}(\matw_t)
  = \sum_{i=1}^{t}(\Aclt[i\to{}t]-\Aclinf^{t-i})(w_i-B{}q_i(\matw_{i-1})),
\]
and so
\[
  \nrm*{x_{t+1}^{\pilearn}(\matw_t)
  - x_{t+1}^{\pi}(\matw_t)}
\leq{} 2\Psist\Dq\sum_{i=1}^{t}\nrm*{\Aclt[i\to{}t]-\Aclinf^{t-i}}_{\op}.
\]
To bound the error, we recall
\pref{lem:at_bound}, restated here.
\atbound*
We set $\alpha=4\Dx\Du\betast\Psist^{4}\Gammast\Dq$, which ensures
that
  \[
  \abs*{\cost(\pilearn;\matw)-\cost(\pi;\matw)}
  \leq{}
  4\Psist(\Dx^2+\Du^2)\Deltanot + C',
  \]
where $C'$ is a numerical constant. To conclude, we recall from \pref{lem:state_action_bound} that we can
take $D_x \leq{} \bigoht\prn*{
        \kappast^{2}\betast^{3/2}\Psist^{3}\Gammast^{3/2}(1-\gammast)^{-1}\cdot\Dq
      }$ and $D_u\leq{}\bigoht\prn*{
        \kappast^{2}\betast^{5/2}\Psist^{5}\Gammast^{5/2}(1-\gammast)^{-1}\cdot\Dq
      }$.

\end{dproof}

  \qttoqinf*
  \begin{dproof}[\pfref{lem:qt_to_qinf}]
    Before diving into the proof, we recall that, since
    $P_t\psdleq\Pinf$, we have
    \[
      \nrm*{\Sigma_t}_{\op}\leq{}\nrm*{\Sigma_{\infty}}_{\op}=\nrm*{\Rx
        + B^{\trn}\Pinf{}B}_{\op}\leq{}2\Psist^{2}\Gammast \rdef \Dsig.
    \]
  We also recall that $\Dq\geq{}\Dqst$. Now let $\Delta_0\in\bbN$ be a fixed constant to be chosen later, and let $T_0=T-\Delta_0$. We immediately upper
bound the error as
\begin{align*}
&\abs*{\sum_{t=1}^{T}\| q_t -
                 \qstar_{t;t+h}(w_{t:t+h})\|_{\Sigma_t}^2 - \| q_t -
                 \qstar_{\infty;h}(w_{t:t+h})\|_{\Sigma_{\infty}}^2}\\
  &\leq{}\abs*{\sum_{t=1}^{T_0}\| q_t -
    \qstar_{t;t+h}(w_{t:t+h})\|_{\Sigma_t}^2 - \| q_t -
    \qstar_{\infty;h}(w_{t:t+h})\|_{\Sigma_{\infty}}^2}
    + 4\Dsig{}\Dq{}^{2}\Delta_0.
\end{align*}
Now, let $t\leq{}T_0$ be fixed. We upper bound the error for each time
as
\begin{align*}
&  \abs*{\nrm*{ q_t -
  \qstar_{t;t+h}(w_{t:t+h})}_{\Sigma_t}^2 -
                 \nrm*{ q_t -\qstar_{\infty;h}(w_{t:t+h})}_{\Sigma_{\infty}}^2}\\
  &\leq{}  \abs*{\nrm*{ q_t -
    \qstar_{t;t+h}(w_{t:t+h})}_{\Sigma_t}^2 -\nrm*{ q_t -
    \qstar_{t;t+h}(w_{t:t+h})}_{\Sigmainf}^2}\\
  &~~~~+
    \abs*{\nrm*{ q_t -
    \qstar_{t;t+h}(w_{t:t+h})}_{\Sigmainf}^2-
    \nrm*{ q_t -\qstar_{\infty;h}(w_{t:t+h})}_{\Sigma_{\infty}}^2}\\
  &\leq{}\Dq{}^{2}\underbrace{\nrm*{\Sigma_t-\Sigmainf}_{\op}}_{\cE_1}
    + 4\Dq{}\Dsig\underbrace{\nrm*{\qstar_{t;t+h}(w_{t:t+h})-\qstar_{\infty;h}(w_{t:t+h})}}_{\cE_2}.
\end{align*}
\paragraph{Bounding $\cE_1$.}
Expanding the definition of $\Sigma_t$ and $\Siginf$, we immediately
see that
$\nrm*{\Sigma_t-\Sigmainf}_{\op}\leq{}\Psist^{2}\nrm*{P_{t+1}-\Pinf}_{\op}$. Using
\pref{lem:value_iteration}, we have
\begin{align*}
  \nrm*{P_{t+1}-\Pinf}_{\op}^{2}
  \leq{} \betast\Gammast(1+\nust^{-1})^{-(T-t)},
\end{align*}
where $\nust= 2\betast\Psist^{2}\Gammast$. Hence, summing across all
rounds, we have
\begin{align}
  \sum_{t=1}^{T}\nrm*{\Pt-\Pinf}_{\op}
  &\leq\sum_{t=1}^{T}\betast^{1/2}\Gammast^{1/2}(1+\nust^{-1})^{-(T-t)/2}.\notag
    \intertext{Since $\nust^{-1}\leq{}1$ and $1+x\geq{}e^{x/2}$ for
    $x\in[0,1]$, we can upper bound by}
  &\leq\betast^{1/2}\Gammast^{1/2}\sum_{t=1}^{T}e^{-\nust^{-1}(T-t)/4}\notag\\
  &\leq{}O(\betast^{1/2}\Gammast^{1/2}\nust)\notag\\&=O(\betast^{3/2}\Psist^{2}\Gammast^{3/2}),\label{eq:pt_err_bound}
\end{align}
and $\sum_{t=1}^{T}\nrm*{\Sigt-\Siginf}_{\op}\leq{}O(\betast^{3/2}\Psist^{4}\Gammast^{3/2})$.
\paragraph{Bounding $\cE_2$.}
Let $t\leq{}T_0\leq{}T-h$ be fixed, then we have
\begin{align*}
  &\nrm*{\qstar_{t;t+h}(w_{t:t+h})-\qstar_{\infty;h}(w_{t:t+h})}\\
  &  = \nrm*{
  \sum_{i=t}^{t+h}\Sigma_t^{-1}B^{\trn}\prn*{\prod_{j=t+1}^{i}\Aclt[j]^{\trn}}P_{i+1}w_i
  -
    \sum_{i=t}^{t+h}\Sigma_{\infty}^{-1}B^{\trn}(\Aclinf^{\trn})^{i-t}P_{\infty}w_{i}}\\
  &= \nrm*{
  \sum_{i=t}^{t+h}\Sigma_t^{-1}B^{\trn}\Aclt[t\to{}i]^{\trn}P_{i+1}w_i
  -
    \sum_{i=t}^{t+h}\Sigma_{\infty}^{-1}B^{\trn}(\Aclinf^{\trn})^{i-t}P_{\infty}w_{i}}\\
  &\leq{}
    \sum_{i=t}^{t+h}\nrm*{\Sigma_t^{-1}B^{\trn}\Aclt[t\to{}i]^{\trn}P_{i+1}-\Sigma_{\infty}^{-1}B^{\trn}(\Aclinf^{\trn})^{i-t}P_{\infty}}_{\op}.
\end{align*}
Note that for each timestep we have
\begin{align*}
  &
    \nrm*{\Sigma_t^{-1}B^{\trn}\Aclt[t\to{}i]^{\trn}P_{i+1}-\Sigma_{\infty}^{-1}B^{\trn}(\Aclinf^{\trn})^{i-t}P_{\infty}}_{\op}\\
  &\leq{}
    \nrm*{(\Sigma_t^{-1}-\Siginf^{-1})B^{\trn}\Aclt[t\to{}i]^{\trn}P_{i+1}}_{\op}
    +
    \nrm*{\Siginf^{-1}B^{\trn}(\Aclt[t\to{}i]^{\trn}-(\Aclinf^{\trn})^{i-t})P_{i+1}}_{\op}\\
  &~~~~+ \nrm*{\Sigma_{\infty}^{-1}B^{\trn}(\Aclinf^{\trn})^{i-t}(P_{i+1}-P_{\infty})}_{\op}.
\end{align*}
If we select $T_0\leq\Tstab-h$, then we are guaranteed by
\pref{lem:closed_loop_refined} that
$\nrm*{\Aclt[t\to{}i]}_{\op}\leq{}\betast^{1/2}\Gammast^{1/2}\kappast^{2}\gammab^{i-t}$,
and we also know that
$\nrm*{\Aclinf^{i-t}}_{\op}\leq{}\kappast\gammast^{i-t}$. Hence, we
can upper bound the errors above by
\begin{align*}
   & \betast^{1/2}\Psist\Gammast^{3/2}\kappast^{2}\gammab^{i-t}\nrm*{\Sigma_t^{-1}-\Siginf^{-1}}_{\op}
    +
     \betast\Psist\Gammast\nrm*{\Aclt[t\to{}i]^{\trn}-(\Aclinf^{\trn})^{i-t}}_{\op}\\
  &~~~~+ \betast\Psist\kappast\gammast^{i-t}\nrm*{P_{i+1}-P_{\infty}}_{\op}.
\end{align*}
Furthermore, recall that $\Sigt=\Rx+B^{\trn}P_{t+1}B\psdgeq\Rx$ and
$\Siginf=\Rx+B^{\trn}\Pinf{}B\psdgeq\Rx$, and so we have
\[
\nrm*{\Sigma_t^{-1}-\Siginf^{-1}}_{\op}\leq\betast^{2}\Psist^{2}\nrm*{P_{t+1}-\Pinf}_{\op}.
\]
Putting everything together this gives
\begin{align*}
  &\sum_{t=1}^{T_0}\nrm*{\qstar_{t;t+h}(w_{t:t+h})-\qstar_{\infty;h}(w_{t:t+h})}\\
  &\leq
  2\betast^{5/2}\Psist^{3}\Gammast^{3/2}\kappast^{2}(1-\gammast)^{-1}(h+1) \sum_{t=1}^{T_0}\nrm*{P_{t+1}-\Pinf}_{\op}
  +
  \betast\Psist\Gammast \sum_{t=1}^{T_0}\sum_{i=t}^{t+h}\nrm*{\Aclt[t\to{}i]^{\trn}-(\Aclinf^{\trn})^{i-t}}_{\op}\\
  &~~~~+ \betast\Psist\kappast
    \sum_{t=1}^{T_0}\sum_{i=t}^{t+h}\gammast^{i-t}\nrm*{P_{i+1}-P_{\infty}}_{\op}\\
    &\leq
  4\betast^{5/2}\Psist^{3}\Gammast^{3/2}\kappast^{2}(1-\gammast)^{-1}(h+1) \sum_{t=1}^{T_0}\nrm*{P_{t+1}-\Pinf}_{\op}
  +
  \betast\Psist\Gammast
      \sum_{t=1}^{T_0}\sum_{i=t}^{t+h}\nrm*{\Aclt[t\to{}i]^{\trn}-(\Aclinf^{\trn})^{i-t}}_{\op}.
      \intertext{Recalling \pref{eq:pt_err_bound}, we can further
      upper bound the first erm:}
  &\leq
    O\prn*{\betast^{4}\Psist^{5}\Gammast^{3}\kappast^{2}(1-\gammast)^{-1}h}
  +
    \betast\Psist\Gammast \sum_{t=1}^{T_0}\sum_{i=t}^{t+h}\nrm*{\Aclt[t\to{}i]^{\trn}-(\Aclinf^{\trn})^{i-t}}_{\op}.
\end{align*}
To bound the last term, we recall \pref{lem:at_bound}.
\atbound*
We choose $\alpha=\betast\Psist\Gammast{}\Dq{}\Dsig$, and set
$\Delta_0=C\cdot{}\betast\Psist^{2}\Gammast\log(\kappast^{2}\Psist\Gammast(1-\gammast)^{-1}\cdot\alpha{}T^{3})\vee\Deltast+h$,
so we are ensured that
\[
\betast\Psist\Gammast
\sum_{t=1}^{T_0}\sum_{i=t}^{t+h}\nrm*{\Aclt[t\to{}i]^{\trn}-(\Aclinf^{\trn})^{i-t}}_{\op}
\leq{} C\cdot\frac{1}{\Dq{}\Dsig}.
\]
Putting everything together leads to a final bound of
\begin{align*}
  &O\prn*{\Dq{}^{2}\betast^{3/2}\Psist^{4}\Gammast^{3/2}}
  +
  O\prn*{\Dq{}\Dsig\betast^{4}\Psist^{5}\Gammast^{3}\kappast^{2}(1-\gammast)^{-1}h}
  + O(\Dsig{}\Dq{}^{2}\Delta_0)\\
  &=
  \bigoht\prn*{\Dq{}^{2}\betast^{3/2}\Psist^{4}\Gammast^{3/2}
    +\Dq{}\betast^{4}\Psist^{7}\Gammast^{4}\kappast^{2}(1-\gammast)^{-1}h
  + \Dq{}^{2}\cdot{}\betast\Psist^{4}\Gammast^{2}\log(\Dq{}T)}.
\end{align*}

\end{dproof}
